\documentclass[twoside,11pt]{article}

\usepackage[preprint]{jmlr2e}

\usepackage[utf8]{inputenc}

\usepackage{algorithm}
\usepackage[noend]{algorithmic}
\usepackage{amsfonts, amsmath, amssymb}
\usepackage{mathtools}
\usepackage{booktabs}
\usepackage{enumerate}
\usepackage{caption}
\usepackage{subcaption}
\usepackage{url}
\usepackage{wrapfig}

\usepackage[T1]{fontenc}
\usepackage{enumitem}
\usepackage{nicefrac}

\usepackage{xspace}

\newtheorem{assumption}{Assumption}

\newcommand{\R}{\mathbb{R}}

\newcommand{\mD}{\mathcal{D}}
\newcommand{\mI}{\mathcal{I}}
\newcommand{\mP}{\mathcal{P}}
\newcommand{\mZ}{\mathcal{Z}}

\newcommand{\qigt}{Q_i(\gamma, \Theta)}

\newcommand{\qigtok}{Q_i(\gamma, \Theta_{1:K})}

\newcommand{\InnerLoop}{\textup{InnerLoop}}
\DeclareMathOperator*{\E}{\mathbb{E}}
\DeclareMathOperator*{\supp}{supp}
\DeclareMathOperator*{\argmin}{argmin}

\newcommand{\localupdate}{\textsc{LocalUpdate}\xspace}
\newcommand{\localupdatedecay}{\textsc{LocalUpdateDecay}\xspace}
\newcommand{\yogi}{\textsc{Yogi}\xspace}
\newcommand{\adam}{\textsc{Adam}\xspace}

\DeclarePairedDelimiterX{\inp}[2]{\langle}{\rangle}{#1, #2}
\DeclarePairedDelimiterX{\abs}[1]{\lvert}{\rvert}{#1}
\DeclarePairedDelimiterX{\norm}[1]{\lVert}{\rVert}{#1}
\DeclarePairedDelimiterX{\cbr}[1]{\{}{\}}{#1} 
\DeclarePairedDelimiterX{\rbr}[1]{(}{)}{#1} 
\DeclarePairedDelimiterX{\sbr}[1]{[}{]}{#1} 

\ShortHeadings{On The Outsized Importance of Learning Rates in Local Update Methods}{Charles and Kone\v{c}n\'{y}}
\firstpageno{1}

\begin{document}

\title{On the Outsized Importance of Learning Rates \\ in Local Update Methods}

\author{\name Zachary Charles \email zachcharles@google.com \\
       \name Jakub Kone\v{c}n\'{y} \email konkey@google.com \\
       \addr Google Research}

\editor{}
\maketitle

\begin{abstract}
We study a family of algorithms, which we refer to as \emph{local update} methods, that generalize many federated learning and meta-learning algorithms.
We prove that for quadratic objectives, local update methods perform stochastic gradient descent on a \textit{surrogate} loss function which we exactly characterize. We show that the choice of client learning rate controls the condition number of that surrogate loss, as well as the distance between the minimizers of the surrogate and true loss functions.
We use this theory to derive novel convergence rates for federated averaging that showcase this trade-off between the condition number of the surrogate loss and its alignment with the true loss function.
We validate our results empirically, showing that in communication-limited settings, proper learning rate tuning is often sufficient to reach near-optimal behavior. We also present a practical method for automatic learning rate decay in local update methods that helps reduce the need for learning rate tuning, and highlight its empirical performance on a variety of tasks and datasets.
\end{abstract}


\section{Introduction}

Historically, machine learning was analyzed from a ``centralized'' perspective, in which a model is trained on a single central source of data. In recent years, there has been a shift away from centralized machine learning, due in part to the increase of user data and the increasing awareness of the risks to privacy that can accompany centralized data collection.

Federated learning (FL) \citep{kairouz2019advances} is a distributed framework for learning models without directly sharing user data. In this framework, heterogeneous clients all use their own data to perform local training. In the popular FedAvg algorithm \citep{mcmahan17fedavg}, the client models are then averaged at a central server, broadcast to a (possibly different) sample of clients, and the process is repeated. The core tenet is that instead of having clients share data, we instead share the results of \emph{local updates} the clients perform on their own datasets using an optimization algorithm.

While there has been growing interest in FL in research communities (see \citep{kairouz2019advances} and \citep{li2019federated} for surveys of many recent works and open problems), this general paradigm of performing local updates on heterogeneous datasets has a storied history in machine learning. In particular, much of the work on meta-learning has focused on trying to learn models that perform well (or can quickly learn how to perform well) on a large number of heterogeneous tasks. This similarity with FL is even more clear in work on model-agnostic meta-learning (MAML) \citep{finn2017model}, in which local client gradient updates are used to learn a global model. Connections between these two areas were noted by \citet{jiang2019improving} and have since been explored in many other works \citep{khodak2019adaptive, fallah2020personalized}.

While there is a wide variety of theoretical and empirical analyses of the aforementioned methods, it is generally difficult to understand their behavior in \emph{heterogeneous settings}. There is enough evidence that these methods are useful in practice in complex scenarios \citep{hard2018federated, yang2018applied, hard2020training}, yet on a theoretical level, many works derive results comparable to, or worse than, that of mini-batch SGD in heterogeneous or even homogeneous settings; See \citep{kairouz2019advances} for a discussion of homogeneity and heterogeneity, and see \citep{woodworth2020local} for a detailed discussion of comparisons to mini-batch SGD. Unfortunately, these results shed little light onto how methods such as FedAvg improve (or degrade) convergence.

In this work, we analyze a generalized local update paradigm that encompasses many FL and MAML methods, as well as other popular optimization methods such as mini-batch SGD. In order to better understand the structure of these methods in heterogeneous settings without an abundance of assumptions, we focus on the special case of quadratic loss functions. We are generally concerned with understanding the following questions that bridge both theory and practice.
\begin{itemize}
    \item How do local update methods improve or hinder convergence?
    \item Why, despite a relative paucity of theoretical evidence, do these methods often perform better in practice than theoretically established methods such as mini-batch SGD?
    \item What obstacles are there to the performance of local update methods, and how do we mitigate these issues?
\end{itemize}

As a partial answer to these questions, we highlight the main findings of our work.

\begin{enumerate}
    \item We show that in the quadratic case, local update methods are equivalent to the stochastic gradient method on a surrogate loss function which we exactly characterize. Thus, we can view local update methods that use multiple heterogeneous datasets as instead performing SGD on a single ``central'' loss function.
    \item We show that methods such as FedAvg and many incarnations of MAML implicitly regularize the condition number of this surrogate loss function, allowing for improved convergence of the surrogate loss. On the other hand, we show that this condition number reduction comes at the cost of increasing the discrepancy between minimizers of the surrogate and the true loss function. Notably, this trade-off is controlled by fundamental algorithmic choices, especially the choice of learning rate.
    \item We give explicit convergence rates for FedAvg that exhibit the trade-off between the condition number and the discrepancy between the surrogate and true loss functions above. Our results are similar in scope to work by \citet{woodworth2020local} (showing that local SGD can outperform mini-batch SGD), but work under heterogeneous data settings.
    \item We use our theoretical insights to design practical improvements to federated learning methods. First, we show that \emph{decoupling client and server learning rates} has significant implications for improving convergence to better models. We show that despite the non-optimality of critical points of FedAvg, combining this learning rate decoupling with proper tuning can result in near-optimal performance in settings with limited communication. Finally, we detail a simple, practical method for automatic learning rate decay in federated learning that helps reduce the burden of learning rate tuning. We show empirically that this method improves the convergence of FedAvg, without requiring manually crafted learning rate schedules, across a suite of realistic and challenging non-convex tasks.
    \end{enumerate}
    
\subsection{Related work}

\paragraph{Federated learning}
Federated learning is a distributed machine learning paradigm in which training is done locally on clients, without any centralized data aggregation.  Federated learning has enabled privacy-aware learning in a variety of applications~\citep{hard2018federated, chen2019federated, brisimi2018federated, samarakoon2018federated, hard2020training}, and has seen a large volume of work on the intersection of federated learning with topics including differential privacy~\citep{mcmahan2017learning, augenstein2020generative}, fairness~\citep{mohri2019agnostic, li2019fair}, robustness~\citep{ghosh2019robust, bagdasaryan2018backdoor, sun2019can}, and communication-efficiency~\citep{konevcny2016federated, sattler2019robust, basu2019qsparse,  reisizadeh2019fedpaq}. For a more detailed discussion of federated learning, we defer to surveys by \citet{kairouz2019advances} and \citet{li2019federated}.

\paragraph{Meta-learning}
In meta-learning (aka \emph{learning to learn}), the objective is to use a collection of tasks to learn how to learn a new task efficiently \citep{vanschoren2019meta}. A particularly influential recent approach is model-agnostic meta-learning (MAML) proposed by \citet{finn2017model}. The core idea has inspired a number of extensions \citep{antoniou2018how, nichol2018first, rusu2018metalearning, grant2018recasting, rajeswaran2019meta, raghu2019rapid}, which broadly use a two-level optimization structure to perform meta-learning. Convergence properties of some of these optimization algorithms were recently studied by \citet{fallah2019convergence}, who also highlight differences in convergence of MAML and first-order approximations to MAML.

\paragraph{Federated optimization}
One of the most common approaches to optimization in the setting of federated learning is the FedAvg method \citep{mcmahan17fedavg}. While designed for heterogeneous sources of data, the study of FedAvg has roots in that of Local SGD~\citep{zinkevich2010parallelized, stich2018local, wang2018cooperative, stich2019error, yu2019parallel, khaled2020tighter}, a communication-efficient optimization method for homogeneous clients. As interest in federated learning has grown, so too has the number of proposed federated optimization methods. These can often be seen as variants of FedAvg, that incorporate techniques such as momentum~\citep{hsu2019measuring}, adaptive optimization~\citep{reddi2020adaptive, xie2019local}, proximal updates~\citep{li2018federated, pathak2020fedsplit} and control variates~\citep{karimireddy2019scaffold}. We again defer to \citet{kairouz2019advances} and \citet{li2019federated} for more detailed references.
    
\paragraph{Convergence (and non-convergence) of FedAvg}
While we defer to \citet[Section 3.2]{kairouz2019advances} for a complete discussion of federated optimization, we discuss a few important connections. First, while there has been huge progress in theoretical understandings of FedAvg, existing works generally have not been able to show that these methods consistently improve upon mini-batch SGD~\citep{woodworth2020local}. Even theoretically and empirically successful techniques such as SCAFFOLD \citep{karimireddy2019scaffold} have only been shown to converge faster than mini-batch SGD on quadratic objectives.
    
This failure of convergence was noted by \citet{li2019convergence}, who showed that without learning rate decay, FedAvg is not guaranteed to converge. Later, \citet{karimireddy2019scaffold} and \citet{woodworth2020local} showed that there are settings where FedAvg converges provably slower than mini-batch SGD. Similarly, \citet{malinovsky2020local} and \citet{pathak2020fedsplit} showed that in heterogeneous settings, FedAvg can converge to sub-optimal points, even in non-stochastic, strongly convex settings. \citet{pathak2020fedsplit} further give a proximal version of federated gradient descent that converges to the empirical risk minimizer in convex settings. 
    
\paragraph{Comparisons to our work}
Our work is most closely related to that of \citet{malinovsky2020local}, \citet{pathak2020fedsplit}. We also evince the non-convergence of FedAvg. However, we extend the analysis to stochastic settings, and to a more general class of algorithms that encompasses many meta-learning algorithms. As such, our work is also closely related to that of \citet{fallah2019convergence}, who demonstrated differences (and non-convergence issues) of various MAML algorithms. Our work takes this a step further, where we give a unified view of both MAML and federated learning methods, and give a broader characterization of the sub-optimal convergence of these methods in the case of quadratic losses. Our work is also novel in its focus on the interplay between convergence, suboptimality, and algorithmic choices, especially learning rates.

\paragraph{Notation}

For a vector $v \in \R^d$, we let $\norm{v}$ denote its $\ell_2$ norm. For a matrix $A \in \R^{n\times m}$, we let $\norm{A}$ denote its operator norm with respect to the $\ell_2$ vector norm. For a symmetric positive semi-definite matrix $A$, we will let $A^{1/2}$ denote its matrix square root. For any real symmetric matrix $A$ (therefore with real eigenvalues), we will let $\lambda_{\max}(A)$ and $\lambda_{\min}(A)$ denote its largest and smallest eigenvalues, respectively.

\section{Preliminaries}

Suppose we wish to learn a model $x \in \R^d$. Let $\mI$ denote some collection of \emph{clients}, and let $\mathcal{P}$ be a distribution on $\mI$. For each $i \in \mI$, we assume that there is an associated \emph{data distribution} $\mD_i$ on some example space $\mZ$. For any $z \in \mZ$, we assume there is a unique corresponding symmetric matrix $A_z \in \R^{d \times d}$ and vector $c_z \in \R^d$, and define a quadratic loss function
\begin{equation}
\label{eq:quadratic_loss}
    f(x; z) := \frac{1}{2}\|A_z^{1/2}(x-c_z)\|^2.
\end{equation}
We let $\nabla f(x; z)$ denote the gradient of the function $x \mapsto f(x; z)$ with respect to $x$. For $i \in \mI$, we define the client loss function $f_i$ and the overall loss function $f$ as follows:
\begin{equation}\label{eq:objective}
f_i(x) := \E_{z \sim \mD_i} [f(x ; z)],~~~f(x) := \E_{i \sim \mP}[f_i(x)].
\end{equation}
One common objective in our setup is to minimize $f(x)$, though this is often not the direct goal of MAML methods. Note that the joint distribution over $(\mI, \mZ)$ implicitly defines a (marginal) distribution over $\mZ$, recovering standard risk minimization frameworks. This framework also encompasses distributed risk minimization in which $\mP$ is a uniform distribution over a finite set of nodes $i \in \mI$ and $\mD_i$ is the uniform distribution over the (finite) dataset stored at node $i$. However, we take a more general approach and do not assume $\mI$ or $\mZ$ to be finite throughout.
We also focus on the \emph{heterogeneous} setting, where the client distributions $\mD_i$ are not all identical, as opposed to the \emph{homogeneous} setting, where all $\mD_i$ are identical.

\paragraph{Modelling assumptions and relevance}
As FL has matured, it has become more evident that there are two varieties, with distinct system-imposed constraints, recently termed by \citet{kairouz2019advances} as \emph{cross-device federated learning} and \emph{cross-silo federated learning}.\footnote{A different categorization, \emph{vertical} and \emph{horizontal}, was proposed by \citet{yang2019federated}, which is based on modelling constraints, rather than on system constraints. The setup in this work applies primarily to horizontal FL, though we expect that much of our framework carries over to the vertical setting.} The primary distinction between these two frameworks that is relevant to our work is that in cross-silo FL, there are relatively few participating clients. Moreover, these clients are typically reliable and almost always available. By contrast, in cross-device FL there are potentially very large numbers of clients, only a small fraction of which are available at any given point in time. Furthermore, the clients cannot be addressed directly or re-identified if participating multiple times. For a more detailed summary, see~\citep[Table~1]{kairouz2019advances}.

In cross-device FL, a client $i$ sampled from $\mI$ corresponds to a single device, and $\mD_i$ corresponds to the data available on that device.
In many practical cross-device FL systems (see~\citet{bonawitz2019towards, hard2018federated}), the server does not control the selection of clients from the global population $\mI$. Instead, participation is initiated by the clients, based on pre-defined eligibility criteria, such as whether the device is charging and on unmetered wifi. Thus, the client distribution $\mP$ can be considered as fixed, with only minor possibilities for it to be shaped by the server (e.g. whether to enforce sampling without replacement).

On the other hand, in many examples of cross-silo FL, participating clients correspond to various medical or financial organizations, or different geographical regions of the same organization \citep{nvidia_clara, yang2019federated}. The participating clients are typically fixed in advance, and often all of them participate in every communication round. Thus, while cross-silo FL may be accurately described by a finite-sum optimization problem, this framework is less useful for cross-device FL.

To see this, consider the task of next word prediction on mobile devices. The FL training described by \citet{hard2018federated} runs for $3000$ communication rounds, with up to $500$ clients participating in each round. That is at most $1.5$ million distinct clients, a small fraction of the total number of possible clients\footnote{As of May 26, 2020, the Google Play Store reports ``1,000,000,000+ installs'' for the GBoard application.}. This also implies that it is nearly impossible to compute exact values of the loss $f(x)$. Instead, evaluation of a model's quality is done using the same mechanism as the training -- by using a subset of the clients eligible at a given time -- which has significant implications for algorithm design (as we discuss in Sections \ref{sec:exp_valid} and \ref{sec:lr_decay}). These issues are exacerbated by heterogeneity; Under extreme heterogeneity, finite-sum modelling approaches may lead to theory that does not accurately represent practical FL systems. Thus, our modelling assumptions are designed to encompass both cross-silo and cross-device setting.

Our setup is also relevant to that of model-agnostic meta-learning (MAML), first proposed by \citet{finn2017model}. In MAML, the main objective is to find a gradient-based mechanism, which given a task $i$ sampled from $\mP$, adapts to have good performance on the distribution $\mD_i$.
Unlike cross-device FL, where we generally cannot quantify $\mP$ directly because of data restrictions, the distribution $\mP$ is the primary object of interest in MAML. However, much like cross-device FL this distribution is generally not known a priori, but instead is problem-dependent.

\subsection{\localupdate algorithms}\label{sec:local_update_algs}

In the following, we will consider a broad class of algorithms that attempt to minimize $f(x)$ (such as in FL methods) or attempt to learn a model that personalizes well with respect to $\mP$ (such as in meta learning algorithms). We refer to these as \localupdate algorithms. In such methods, at each round, a central coordinator (which we will refer to as a \emph{server}) works with $M$ \emph{clients} (or in the language of MAML, tasks) sampled from $\mP$ and broadcasts its global model to the clients. Each client $i$ optimizes its loss function $f_i$ (initializing at the broadcast model) by iteratively applying mini-batch SGD with batch size $B$ and \emph{client} (inner) learning rate $\gamma$. The mini-batch gradients are computed by taking samples from the client's local dataset $\mD_i$. The client then sends a linear combination (parameterized by $\Theta = (\theta_1, \theta_2, \dots)$ where $\theta_i \in \R_{\geq 0}$) of its gradients to the server. We will only consider $\Theta$ with finite support. For such $\Theta$, we define
\[
K(\Theta) = \max\{i~|~\theta_i > 0\}.
\]
Throughout our work, we will omit the trailing zeros in any $\Theta$ with finite support. The server averages the available updates, and, treating this average as a stochastic gradient of the loss function $f(x)$, performs a gradient step with a \emph{server} (outer) learning rate $\eta$. Algorithms \ref{alg:outerloop} and \ref{alg:innerloop} give pseudo-code for \localupdate.

\newcommand{\T}{\rule{0pt}{2.2ex}}
\newcommand{\SUB}[1]{\ENSURE \hspace{-0.15in} \textbf{#1}}
\newcommand{\algfont}[1]{\texttt{#1}}
\renewcommand{\algorithmicensure}{}

\begin{figure}[ht]

\begin{minipage}[t]{.48\textwidth}
\begin{algorithm}[H]
\caption{\localupdate: Outer Loop}
\label{alg:outerloop}
\begin{algorithmic}
\SUB{OuterLoop$(x, \{\eta_t\}_{t\geq 1}, \{\gamma_t\}_{t \geq 1}, \Theta)$:}
\STATE $x_1 = x$
\FOR{each round $t = 1, 2, \dots$, T}
    \STATE $I_t \leftarrow$ (random set of $M$ clients)
    \FOR{each client $i \in I_t$ \textbf{in parallel}}
        \STATE $q_t^i \leftarrow \text{InnerLoop}(i, x_t, \gamma_t, \Theta)$
    \ENDFOR
    \STATE $q_t \leftarrow (\nicefrac{1}{M})\sum_{i \in I_t} q_{t}^i$
    \STATE $x_{t+1} = x_t-\eta_tq_t$
\ENDFOR
\STATE return $x_{T+1}$
\end{algorithmic}
\end{algorithm}
\end{minipage}
\begin{minipage}[t]{.48\textwidth}
\begin{algorithm}[H]
\caption{\localupdate: Inner loop}
\label{alg:innerloop}
\begin{algorithmic}
\renewcommand{\arraystretch}{1.6}
\SUB{$\InnerLoop(i, x, \gamma, \Theta$):}
\STATE $x_1 = x$
\FOR{$k = 1, 2, \dots, K(\Theta)$}
    \STATE sample a set $S_k$ of size $B$ from $\mD_i$
    \STATE $g_k = (\nicefrac{1}{B}) \sum_{z \in S_k} \nabla f(x_k ; z)$
    \STATE $x_{k+1} \leftarrow x_k - \gamma g_k$ 
\ENDFOR
\STATE return $\sum_{k = 1}^{K(\Theta)} \theta_k g_k$
\end{algorithmic}
\end{algorithm}
\end{minipage}

\end{figure}

This method recovers some well-known algorithms for specific choice of $\gamma, \eta$, and $\Theta$. For convenience of notation, we define
\begin{equation}\label{eq:theta_maml}
    \Theta_K = (\underbrace{0, \dots, 0}_{\text{K-1 times}}, 1),
\end{equation}
so in particular $\Theta_1 = (1)$ and $K(\Theta_1) = 1$, and similarly
\begin{equation}\label{eq:theta_fedavg}
    \Theta_{1:K} = (\underbrace{1, \dots, 1}_{\text{K times}}).
\end{equation}

Many existing training algorithms can be expressed as special cases of \localupdate. We give a non-exhaustive list below.
\begin{itemize}
    \item The simplest setting is mini-batch SGD. This can be recovered in multiple ways. For example, suppose each client $i$ corresponds to a single example $z_i$. Then, \localupdate with $\Theta = \Theta_1$ is equivalent to mini-batch SGD with batch size $M$ and learning rate $\eta$.
    \item Alternatively, if there is only a single client ($|\mI|=1$), then \localupdate with $\Theta = \Theta_1$ becomes mini-batch SGD with batch size $B$ and learning rate $\eta$.  As expected, the choice of $\gamma$ has no impact in either instance of mini-batch SGD.
    \item More generally, setting $\Theta = \Theta_1$ recovers distributed mini-batch SGD, with total batch size $MB$ and learning rate $\eta$.  Again, $\gamma$ has no impact on the global model.
    \item When there is a single client and $\Theta = \Theta_{1:K}$, then \localupdate recovers the Lookahead optimizer~\citep{zhang2019lookahead} with $K$ ``fast weights''.
    \item In the homogeneous setting, if $\Theta = \Theta_{1:K}$, and $\gamma = \eta$, then \localupdate is equivalent to Local SGD with $K$ local steps. For further details, see Appendix \ref{sec:special_cases}.
    \item In the heterogeneous setting, if we set $\Theta = \Theta_{1:K}$ and $\gamma = \eta$, then \localupdate is equivalent to FedAvg with $K$ local steps (see Appendix \ref{sec:special_cases} for details). When $\gamma$ and $\eta$ are not necessarily equal, we actually recover Reptile~\citep{nichol2018first}, as well as the Generalized FedAvg algorithm in \citep{reddi2020adaptive}. For convenience of notation, we will refer to this algorithm as FedAvg/Reptile throughout. This equivalence between FedAvg and Reptile was first noted by \citet{jiang2019improving}. In fact, we show in Section \ref{sec:exp_valid} that this decoupling of client and server learning rates is critical to understanding and improving the convergence of FedAvg.
    \item When $\Theta = \Theta_K$, we recover the first-order MAML (FOMAML) algorithm of \citet{finn2017model}. A similar functional relation between FOMAML and Reptile was previously described by \citet{nichol2018first}.
    \item In the MAML algorithm, \citet{finn2017model} use $K$ local update steps for each ``task'' (in our vocabulary, client). We will refer to this as $\mathcal{K}$-MAML throughout. As we show in Section~\ref{sec:maml}, when the underlying loss functions are quadratic and the clients perform gradient descent updates, $\mathcal{K}$-MAML is recovered by setting $\Theta = \Theta_{2K+1}$. This gives a previously unknown connection between FL and MAML algorithms. As we discuss in Section \ref{sec:maml}, this does not hold when the clients use SGD due to potential biases in estimating Hessian-gradient products via stochastic gradients.
\end{itemize}

As written, both clients and server use SGD as their optimizer in \localupdate. However, one could use techniques such as momentum or adaptive learning rates on either the server (as explored by \citet{reddi2020adaptive}) or the client (as explored by \citet{xie2019local}). While our results can be extended to these settings, we leave this to future work. Our goal is not to derive convergence results for as broad a class of algorithms as possible. Rather, we wish to understand how the choice of $\gamma$ and $\Theta$ impact the dynamics of optimization, especially in heterogeneous settings.

We note that in Algorithm \ref{alg:innerloop}, each clients performs a designated number of steps of mini-batch SGD, with samples taken from some underlying client distribution $\mathcal{D}_i$. When $\mathcal{D}_i$ is the uniform distribution over some finite set $\mathcal{S}_i$, we could instead write Algorithm \ref{alg:innerloop} in terms of performing some number of epochs $E$ of mini-batch SGD over $\mathcal{S}_i$, as is done in \citep{mcmahan17fedavg} and many other works on federated learning. In this case, the batch size $B$ dictates the number of client gradient steps (as the client roughly take $E|\mathcal{S}_i|/B$ steps). Thus, in such settings, the choice of $B$ has an analogous impact as the choice of the number of local steps $K(\Theta)$ in Algorithm \ref{alg:innerloop}. For simplicity of analysis, we will analyze the latter throughout, but our results can be easily extended to the former.

\subsection{Outline}
The rest of this paper is organized as follows. In Section~\ref{sec:local_update_as_sgd}, we show that a round of \localupdate method is equivalent to performing a \emph{single} (stochastic) gradient step with respect to  a surrogate objective, which we exactly characterize.

In Section~\ref{sec:example}, we use simple examples to show that the surrogate loss and the original loss can vary substantially. Moreover, we show how choices of $\gamma$ and $\Theta$ affect the discrepancy between the two losses. In particular, we highlight how the choice of $\gamma$ is crucial to the performance of \localupdate. In Section~\ref{sec:properties}, we analyze spectral properties of the surrogate loss, and show that \localupdate can be viewed as implicit regularization on the condition number, where the amount of regularization is controlled by $\gamma$ and $\Theta$.

The next sections present to the best of our knowledge a novel proof technique, characterizing the convergence of FedAvg/Reptile in heterogeneous settings.\footnote{While we focus on FedAvg and Reptile, we note that a similar analysis can be performed for any of the special cases listed above, using a similar proof strategy.} In Section \ref{sec:distance}, we bound the distance between the minimizers of the surrogate and the true loss function in terms of the client learning rate $\gamma$. We use these results in Section \ref{sec:convergence} to derive convergence rates for FedAvg/Reptile that highlight how the choice of client learning rate $\gamma$ gives rise to a trade-off between local and global optimization. In particular, we show that learning rate decay is both sufficient and necessary for convergence to the true risk minimizer.

While our theoretical results are valid only for quadratic loss functions, in Section~\ref{sec:exp_valid} we show empirically that our conclusions carry over to more general settings, including non-convex objectives. Our empirical results highlight the importance of learning rate tuning in federated learning. In Section \ref{sec:lr_decay}, we combine our theoretical insights with important systems-level constraints to design a method for automatic learning rate decay methods for local update methods. In particular, we present a simple, easy to implement method for automatic learning rate decay, and show its efficacy in improving accuracy and reducing the need for client learning rate tuning.

\section{\localupdate as SGD}
\label{sec:local_update_as_sgd}

When $K(\Theta) > 1$ and $\gamma > 0$, the dynamics of \localupdate may be very different than those of mini-batch SGD. We will show that for quadratic functions, these dynamics are related but distinct. In particular, we will show that any local update method on a quadratic function can be viewed as SGD on some appropriately defined surrogate loss function. Moreover, the discrepancy between the true loss function and the surrogate loss function is dictated by the choice of client learning rate $\gamma$ and $\Theta$.

For $i \in \mI$, define:
\[A_i := \E_{z \sim \mD_i}[A_z].\]
We assume throughout that $A_i$ is finite and invertible. We also define
\[c_i := A_i^{-1}\E_{z \sim \mD_i}[A_zc_z].\]
Again, we assume this is finite. We then have the following lemma.
\begin{lemma}\label{lem:f_i_loss}
For all $i \in \mI$, there is some constant $\tau_i$ such that
\[
f_i(x) = \dfrac{1}{2}\norm{A_i^{1/2}(x-c_i)}^2 + \tau_i.
\]
\end{lemma}

In the sequel, we will omit the constant term $\tau_i$, and let
\[
f_i(x) = \dfrac{1}{2}\norm{A_i^{1/2}(x-c_i)}^2
\]
as this does not change the gradients of the loss $f_i$. Since each $A_z$ is symmetric and positive definite, so is $A_i$. We define the following:
\[
    A := \E_{i\sim\mP}[A_i],~~~c:=\E_{i\sim\mP}[c_i].
\]
We will assume that these expectations exist and are finite throughout. We will also utilize the following mild assumptions at different times.

\begin{assumption}\label{assm0}
$K(\Theta) > 0$.
\end{assumption}

\begin{assumption}\label{assm1}
There are $\mu, L > 0$ such that for all $i$, 
\[\mu I \preceq A_i \preceq L I.\]
\end{assumption}

\begin{assumption}\label{assm2}There are finite $\sigma_A$ and $\sigma_c$ such that
\[
\E_{i\sim\mP}\sbr{\norm{A_i-A}^2} \leq \sigma_A^2.
\]
\[
\E_{i\sim\mP}\sbr{\norm{c_i-c}^2} \leq \sigma_c^2.
\]
\end{assumption}

Assumption \ref{assm0} prevents pathologically bad choices of $\Theta$ in which clients simply send 0 to the server at every round. Assumption \ref{assm1} amounts to assuming upper and lower bounds on the Lipschitz and strong convexity parameters of each loss function $f_i$. This is satisfied if there are a finite number of clients, and for each, $A_i$ is positive definite. However, it is often true in more generality if the underlying matrices $A_z$ satisfy some kind of bounded eigenvalue condition. Moreover, when the number of clients is finite, we can always ensure that $\mu I \preceq A_i$ for all $i$ by adding $\ell_2$ regularization to our objective function.

Assumption \ref{assm2} assumes that the matrices $A_i$ and optimal points $c_i$ for each loss function have bounded variance. We do not assume that the gradients computed by the clients have bounded norm. Intuitively, as $\sigma_c \to 0$, local update methods should provide more benefit, as the clients are taking more steps towards a shared optimum. While $\sigma_c = 0$ in the case of homogeneous data distributions (i.e. $\mD_i$ are the same for all $i\in\mI$), these two conditions are not equivalent. There are heterogeneous data distributions which still yield $\sigma_c = 0$. Also, note that $c$ is in general not the minimizer of the objective $f$.

Fix $i \in \mI$, and consider Algorithm \ref{alg:innerloop}. We initialize $x_1 = x$, and then at each iteration $k$ we sample a set $S_k$ uniformly at random (with replacement) from $\mD_i$, then update via
\begin{equation}\label{eq:inner_g_update}
    g_{k} = \frac{1}{B}\sum_{z \in S_k} \nabla f(x_k ; z),
\end{equation}
\begin{equation}\label{eq:inner_x_update}
    x_{k+1} = x_k- \gamma g_k.
\end{equation}
We first prove a basic recurrence relation concerning the local gradients $g_k$ for task $i$.
\begin{lemma}\label{lem:grad_recurrence}
For $i \in \mI$, suppose that $A_i$ is invertible and $g_k$ as in Algorithm \ref{alg:innerloop}, for all $k \geq 1$,
\begin{equation}\label{eq:grad_recurrence}
    \E[g_{k+1}] = (I-\gamma A_i)\E[g_k].
\end{equation}
\end{lemma}

\paragraph{Defining the surrogate loss}

Using Lemma \ref{lem:grad_recurrence}, we will show that Algorithm \ref{alg:innerloop} can be viewed as performing SGD on a \textit{surrogate loss}. This surrogate loss will be parameterized by the inputs $\gamma$ and $\Theta$ to Algorithm \ref{alg:innerloop}. To define the surrogate loss, we first define, for each client $i \in \mI$, a \textit{distortion matrix} $Q_i(\gamma, \Theta)$ as follows:
\begin{equation}\label{eq:Q_matrix}
Q_i(\gamma, \Theta) := \sum_{k=1}^{K(\theta)} \theta_k (I-\gamma A_i)^{k-1}.
\end{equation}
We can then define, for each $i \in \mI$, the client's surrogate loss function:
\begin{equation}\label{eq:surrogate_i}
\tilde{f}_i(x, \gamma, \Theta) := \frac{1}{2}\|(Q_i(\gamma, \Theta)A_i)^{1/2}(x-c_i)\|.
\end{equation}
The overall surrogate loss function is then given by
\begin{equation}\label{eq:surrogate_loss}
    \tilde{f}(x, \gamma, \Theta) := \E_{i \sim \mP} [\tilde{f}_i(x, \gamma, \Theta)].
\end{equation}
Informally, the matrix $Q_i(\gamma, \Theta)$ can be viewed as causing a distortion to the matrix $A_i$. When $\Theta = \Theta_1$, one can see that $Q_i(\gamma, \Theta) = I$, in which case there is no distortion. For other $\Theta$, $Q_i(\gamma, \Theta)$ may significantly distort $A_i$, and can amplify heterogeneity of the $A_i$. Using Lemma \ref{lem:grad_recurrence}, we derive the following property of the output of the Algorithm~\ref{alg:innerloop}.
\begin{theorem}\label{thm:sgd_objective}
Suppose that $A_i$ is invertible. Then
\begin{equation}\label{eq:thm_1}
\E[\InnerLoop(i, x, \gamma, \Theta)] = \nabla_x \tilde{f}_i(x, \gamma, \Theta) = Q_i(\gamma, \Theta)\nabla f_i(x).
\end{equation}
\end{theorem}

\begin{proof}
By direct computation, 
\[\nabla \tilde{f}_i(x, \gamma, \Theta) = Q_i(\gamma, \Theta)A_i(x-c_i).\]
On the other hand, by Lemma \ref{lem:grad_recurrence},
\begin{align*}
\E[\InnerLoop(i, w, \gamma, \Theta)] &= \sum_{k=1}^{K(\theta)} \theta_k(1-\gamma A_i)^{k-1}\E[g_1]\\
&= \sum_{k=1}^{K(\theta)} \theta_k(1-\gamma A_i)^{k-1}A_i(x-c_i)\\
&= Q_i(\gamma, \Theta)A_i(x-c_i).
\end{align*}
This proves the first equality. The second follows from noting that $\nabla f_i(x) = A_i(x-c_i)$.
\end{proof}

Let $q_t$ be as in Algorithm \ref{alg:outerloop}. Then Theorem \ref{thm:sgd_objective} implies $\E[q_t] = \nabla \tilde{f}(x, \gamma, \Theta)$.
In particular, one round of \localupdate with a given $\gamma, \Theta$ is equivalent to performing one step of SGD on the surrogate loss function $\tilde{f}(x, \gamma, \Theta)$.

We note that a version of Theorem was first shown for the case $\Theta = \Theta_{2}$ by \citet{fallah2020personalized}, and was used to compare the behavior of FOMAML and MAML. We will take this comparison a step further, by showing in Section \ref{sec:maml} that in the non-stochastic client setting, MAML can also be viewed as performing SGD on a similarly-defined surrogate loss.

Theorem \ref{thm:sgd_objective} has important consequences regarding the impact of other ``outer optimizers'' in Algorithm \ref{alg:outerloop}, such as the adaptive server optimization \citep{reddi2020adaptive}. If we treat the the output of Algorithm~\ref{alg:innerloop} simply as a stochastic gradient oracle of $\tilde{f}(x, \gamma, \Theta)$, we can apply existing convergence guarantees of any gradient based methods to understand the behavior of \localupdate method with different outer optimizers. In particular, this implies that the choice of outer optimization method primarily impacts the \emph{speed} of convergence to $\argmin_x \tilde{f}(x, \gamma, \Theta)$, but not the point \localupdate actually converges to. We empirically analyze the use of adaptive server methods in \localupdate in Sections \ref{sec:exp_valid} and \ref{sec:lr_decay}.

\subsection{MAML}\label{sec:maml}

As previously discussed, in the setting above, one can actually view MAML as a special case of \localupdate. In this section we elaborate on the claim, using a similar presentation of MAML as in \citep{nichol2018first}. MAML with $K$ local steps (which we refer to as $\mathcal{K}$-MAML) can be viewed as a simple modification of \localupdate. Algorithm \ref{alg:outerloop} proceeds in the same manner. In Algorithm \ref{alg:innerloop}, each client still executes $K$ mini-batch SGD steps. However, what each client sends to the server differs from Algorithm \ref{alg:outerloop}.

For simplicity, we define $X_{K}^i(x)$ as the function that runs $K$ steps of mini-batch SGD, starting from $x$, for some fixed mini-batches $S_1, \dots, S_K$ of size $B$ drawn independently from $\mathcal{D}_i$. For convenience, we let $X_0^i(x) = x$. We then define
\begin{equation}
m^i_K(x ; z) := f(X_{K}^i(x); z),
\end{equation}
\begin{equation}
m_{K}^i(x) := \E_{z \sim \mD_i} m^i_K(x ; z).
\end{equation}
Note that these are implicitly functions of the mini-batches $S_1, \ldots, S_K$ sampled from $\mD_i$. The output $q^i(x)$ of client $i$ (as a function of its initial model $x$) is a stochastic estimate of $\nabla m_K^i(x)$, so that
\begin{equation}\label{eq:maml_output}
\E[q^i(x)] = \nabla m_K(x ; z),
\end{equation}
The remainder of the MAML algorithm proceeds in the same way as \localupdate. Namely, the server averages the client outputs, and uses this as a gradient estimate with learning rate $\eta$. That is,
\begin{equation}\label{eq:maml_server_update}
q(x) = \sum_{i \in I} q^i(x),
\end{equation}
\begin{equation}\label{eq:maml_server_update_2}
x' = x - \eta q(x).
\end{equation}

We now show that when the clients use gradient descent to perform their local update, $\mathcal{K}$-MAML is in expectation equivalent to performing \localupdate with $\Theta = \Theta_{2K+1}$.
\begin{theorem}\label{thm:maml}If $X_K^i(x)$ is the function that runs $K$ steps of gradient descent, starting from $x$, on the client dataset $\mD_i$, then
\[
\nabla m^i_K(x) = \E[\InnerLoop(x, \gamma, \Theta_{2K+1})] = \nabla_{x} \tilde{f}_i(x, \gamma, \Theta_{2K+1}).
\]
\end{theorem}
It is fruitful to reflect on what this means. Informally, this result shows that for quadratic functions, the gradient of the loss after $K$ steps of gradient descent, taken with respect to the initial point $x$, is in expectation the gradient of the loss function after $K+1$ additional SGD steps. In particular, given $q(x)$ as in \eqref{eq:maml_output}, we have
\[
\E[q(x)] = \nabla_x \tilde{f}(x, \gamma, \Theta_{2K+1}).
\]
Thus, the MAML update in \eqref{eq:maml_server_update_2} amounts to a single stochastic gradient update on the surrogate loss function $\tilde{f}(x, \gamma, \Theta_{2K+1}).$

We note that this result relied on the clients using gradient descent to compute $X_k^i(x)$. For computational efficiency, this is often instead done using mini-batch SGD. However, computing $\nabla m_K^i(x)$ then involves computing unbiased estimates of the Hessian and gradient using the same batches of data. By the chain rule, estimating $\nabla m_K^i(x)$ involves multiplying these Hessian and gradient estimates. However, the product of these unbiased estimators need not be unbiased since they were computed with respect to the same batch of data. Thus, this correspondence between $\mathcal{K}$-MAML and \localupdate may break down in computationally-efficient (but biased) MAML implementations. For more detailed discussion on this bias, see \citep{fallah2019convergence}. We also note that \citet{fallah2019convergence} analyze the convergence properties of MAML and FOMAML, and independently observe that MAML and FOMAML need not share stationary points for quadratic objectives.

\section{Local update methods tend towards different global minima}\label{sec:example}
As we will show in Section \ref{sec:properties}, $\tilde{f}_i(x, \gamma, \Theta)$ shares many properties with $f_i(x)$, including having the same global minima. However, we first highlight that a crucial correspondence breaks down when considering the population-level global minima in heterogeneous settings. That is, in general
\[\argmin_x f(x) \neq \argmin_x \tilde{f}(x, \gamma, \Theta).\]
In order to enhance our understanding of the surrogate loss function, we first give both analytic and empirical examples.

\paragraph{Analytic examples} Let $\mathcal{P}$ have support $\{1, 2\}$, with each option equally likely. Suppose that $\mD_1, \mD_2$ have support only on the points $z = 1$ and $z = 2$ respectively, and suppose $A_z = z, c_z = 1/z$, so that $f(x; z) = \frac{1}{2}z(x - 1/z)^2$.
We see that $f(x)$ is minimized at $x = 2/3$. On the other hand, let $\Theta = \Theta_{1:2}$. For $\gamma \geq 0$, we can compute the $i$-th distortion matrix by
\begin{align*}
    Q_i(\gamma, \Theta) = 1 + (1-\gamma i) = 2-\gamma i.
\end{align*}
Note that this is positive definite for $i = 1, 2$ as long as $\gamma \in [0, 1/2)$. By \eqref{eq:surrogate_loss},
\begin{align*}
\tilde{f}(x, \gamma, \Theta) = \dfrac{(2-\gamma)(x-1)^2}{4} + (1-\gamma)(x-1/2)^2.
\end{align*}
For $\gamma \in [0, 6/5)$, this is a positive definite quadratic function, with minimum given by
\[
\argmin_{x} \tilde{f}(x, \gamma, \Theta) = \dfrac{4-3\gamma}{6-5\gamma},
\]
implying
\[
\norm{\argmin_{x} \tilde{f}(x, \gamma, \Theta) - \argmin_{x} f(x)} = \dfrac{\gamma}{3(6-5\gamma)}.
\]
Therefore, even if we run \localupdate until convergence, it would not converge to the true risk minimizer. This holds even though there are only two clients, each with a single data point. In other words, some form of learning rate decay is \emph{necessary} for convergence to the risk minimizer. The necessity of learning rate decay was first shown by \citet{li2019convergence}, and later shown in \citep{malinovsky2020local} and \citep{pathak2020fedsplit}. We take this analysis further, by showing how this sub-optimal behavior is explicitly governed by algorithmic choices, especially learning rate and the number of local steps taken.

For instance, in the example above, as $\gamma \to 0$, the surrogate risk minimizer converges to the true risk minimizer. Conversely, as $\gamma$ increases, the two minima become further apart. In fact, as $\gamma \to 6/5$, the distance between the two optima grows arbitrarily large, even though $\tilde{f}(x, \gamma, \Theta)$ is a positive definite quadratic function. The critical issue here is that while $\gamma < 6/5$ ensures that $f(x, \gamma, \Theta_{1:2})$ is well-behaved (ie. that it is a positive definite quadratic function), it is not small enough to ensure that each client's surrogate loss function $\tilde{f}_i(x, \gamma, \Theta)$ is well-behaved. Note $f_2(x, \gamma, \Theta_{1:2})$ is a negative definite quadratic function for $\gamma > 1/2$, which causes this divergence. Even though the iterates of \localupdate are converging, there may be some client which is not converging in any meaningful sense. Therefore:

\begin{center}
    \textbf{The client learning rate should be set sufficiently small so that all client loss functions are well-behaved, even if the overall loss function is well-behaved.}
\end{center}

A similar analysis shows that if we instead fix $\gamma$, the distance between the surrogate risk and true risk minimizers depends on $K$. Let $\gamma = 1/2$. By \eqref{eq:Q_matrix},
\[
Q_1(1/2, \Theta_{1:K}) = \sum_{k=0}^{K-1} 2^{-k} = 2-2^{-(K-1)},~~Q_2(1/2, \Theta_{1:K}) = 1,
\]
implying
\begin{align*}
    f(x, 1/2, \Theta_{1:K}) =\dfrac{1}{2}\left((1-2^{-K})(x-1)^2 + (x-1/2)^2\right).
\end{align*}
This is a positive definite quadratic with minima given by
\[
\argmin_x f(x, 1/2, \Theta_{1:K}) = \dfrac{3\cdot 2^{K} - 2}{2^{K+2}-2}.
\]
This then implies
\[
\norm{\argmin_{x} f(x, 1/2, \Theta_{1:K}) - \argmin_{x} f(x)} = \dfrac{2^K-2}{6(2^{K+1}-1)}.
\]

When $K = 1$, this gap is 0, while as $K \to \infty$, the distance increases monotonically to $1/12$. In fact, as $K \to \infty$, the minimizer of the surrogate loss function converges to the expected value of the client loss minimizers. In Lemma \ref{lem:asymptotic} we prove an even stronger statement, and show that it holds for all positive definite quadratic loss functions.

\paragraph{Empirical examples}
Next, we give an empirical generalization of the above example for further illustration. Let $f(x; z) = (1/2)z x^2 - x$ for $z > 0$ (ie. $A_z = z$ and $c_z = 1/z$). We let $\mP$ have support $[0.5, 2]$ and density function $q(i) = 8i/15$. For each $i \in \supp(\mP)$, we let the client distribution $\mD_i$ be supported on a single point $z = i$, so that $f_i(x) = (1/2)ix^2 - x$. Again, deterministic $f_i$ will still be sufficient observe discrepancies between the true and surrogate loss functions.

In Figure~\ref{fig:1d-example}, we plot the behavior of FedAvg/Reptile with $K$ local steps (\localupdate with $\Theta = \Theta_{1:K}$) on this problem. We fix $\eta = 10^{-4}, M = B = 1$, and illustrate the change in behavior as either $\gamma$ varies and $K$ is fixed, or $K$ varies and $\gamma$ is fixed. We also plot the true minimizer and the average minimizer, $\E_{i\sim\mathcal{P}} [\argmin_x f_i(x)]$.

\begin{figure}[ht]
\centering
\includegraphics[width=0.9\linewidth]{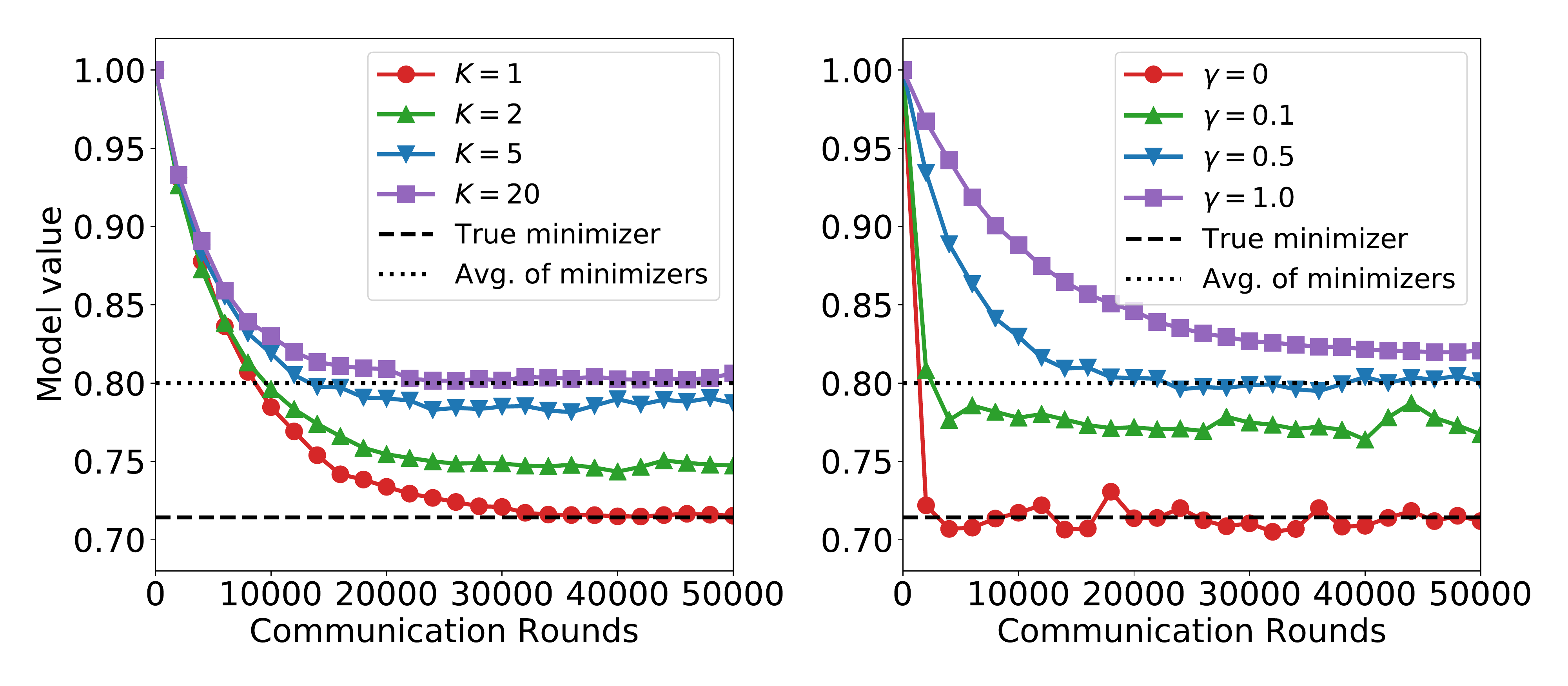}
\caption{\localupdate convergence behavior with $\Theta = \Theta_{1:K}$ on an example $1$-dimensional problem. (Left) We fix $\gamma = 0.5$ and vary $K$. (Right) We fix $K = 20$ and vary $\gamma$.}
\label{fig:1d-example}
\end{figure}

If $\gamma = 0$ or $K = 1$, we converge to $x^*$. As $\gamma$ or $K$ increases, we converge to a point further from $x^*$. As $K \to \infty$ we converge to the average minimizer. We also see that \emph{decreasing} the local stepsize \emph{increases} the variance. This is to be expected: When $\gamma=0$, \localupdate with $M=1$ reduces to mini-batch SGD with batch size $K$, but the gradients in the batch are \emph{summed} rather than averaged.  For $\gamma > 0$, the magnitude of the gradients being summed decreases as the client converges to its minimizer. If we set $\gamma$ to be larger than $1.0$, we see an even greater gap between the surrogate minimizer and the true minimizer (due to the presence of negative-definite clients).

In Section \ref{sec:distance}, we derive general bounds on the distance between surrogate and true minimizers. To do so, we will use spectral properties of the matrix $Q_i(\gamma, \Theta)$, which we derive in the next section.

\section{Surrogate loss properties}\label{sec:properties}

In order to understand the surrogate loss function $\tilde{f}(x, \gamma, \Theta)$, we will analyze properties of the distortion matrix $Q_i(\gamma, \Theta)$. We first note that in certain special cases, $Q_i(\gamma, \Theta)$ is some scaled version of the identity matrix. The following result is a simple consequence of the definition in \eqref{eq:Q_matrix}.

\begin{lemma}\label{lem:Q_special_case}
    Let $a = \sum_{k=1}^{K(\theta)} \theta_k$. If $\gamma = 0$ or $K(\theta) = 1$, then $Q_i(\gamma, \Theta) = aI$ and $\tilde{f}_i(x, \gamma, \Theta) = af_i(x)$.
\end{lemma}

Under these settings, we therefore have $\tilde{f}(x, \gamma, \Theta) = af(x)$, and so the surrogate and true loss functions only differ by a constant. More generally, the functions $\tilde{f}_i(x, \gamma, \Theta)$ inherit many properties from $f_i(x)$, as captured in the following lemma.

\begin{lemma}\label{lem:Q_structure}
    Let $L_i = \lambda_{\max}(A_i), \mu_i = \lambda_{\min}(A_i) > 0$ and suppose that $\gamma < L_i^{-1}$. Then
    \begin{enumerate}
        \item $Q_i(\gamma, \Theta)$ is symmetric and positive definite.
        \item $\tilde{f}_i(x, \gamma, \Theta)$ and $f_i(x)$ have the same unique global minimizer $x^* = c_i$.
        \item For each eigenvalue $\lambda$ of $A_i$, $Q_i$ has an eigenvalue
        \[
        \sum_{k = 1}^{K(\Theta)} \theta_k(1-\gamma \lambda)^{k-1}.
        \]
        with the same multiplicity.
        \item $Q_i(\gamma, \Theta)$ satisfies
        \[
        \lambda_{\max}(Q_i(\gamma, \Theta)) = \sum_{k = 1}^{K(\Theta)} \theta_k(1-\gamma \mu_i)^{k-1}.
        \]
        \[
        \lambda_{\min}(Q_i(\gamma, \Theta)) = \sum_{k = 1}^{K(\Theta)} \theta_k(1-\gamma L_i)^{k-1}.
        \]
    \end{enumerate}
\end{lemma}

For convenience, we note some special cases of Lemma \ref{lem:Q_structure} for $\Theta = \Theta_{K}, \Theta_{1:K}$.

\begin{lemma}\label{lem:Q_structure_fomaml}
    Let $L_i = \lambda_{\max}(A_i), \mu_i = \lambda_{\min}(A_i)$ and suppose that $\gamma < L_i^{-1}$. Then
         \[
        \lambda_{\max}(Q_i(\gamma, \Theta_K)) = (1-\gamma\mu_i)^K,~~~
        \lambda_{\min}(Q_i(\gamma, \Theta_K)) = (1-\gamma L_i)^K.
        \] 
        
    If Assumption \ref{assm1} holds and $\gamma < L^{-1}$, then
         \[
        \lambda_{\max}(Q_i(\gamma, \Theta_K)) \leq (1-\gamma\mu)^K,~~~
        \lambda_{\min}(Q_i(\gamma, \Theta_K)) \geq (1-\gamma L)^K.
        \] 
\end{lemma}

\begin{lemma}\label{lem:Q_structure_fedavg}
    Let $L_i = \lambda_{\max}(A_i), \mu_i = \lambda_{\min}(A_i)$ and suppose that $\gamma < L_i^{-1}$. Then
         \[
        \lambda_{\max}(Q_i(\gamma, \Theta_{1:K})) = \dfrac{(1-(1-\gamma \mu_i)^K)}{\gamma \mu_i},~~~
        \lambda_{\min}(Q_i(\gamma, \Theta_{1:K})) = \dfrac{(1-(1-\gamma L_i)^K)}{\gamma L_i}.
        \] 
        
    If Assumption \ref{assm1} holds and $\gamma < L^{-1}$, then
         \[
        \lambda_{\max}(Q_i(\gamma, \Theta_{1:K})) \leq \dfrac{(1-(1-\gamma \mu)^K)}{\gamma \mu},~~~
        \lambda_{\min}(Q_i(\gamma, \Theta_{1:K})) \geq \dfrac{(1-(1-\gamma L)^K)}{\gamma L}.
        \] 
\end{lemma}

As the eigenvalue bounds above suggest, $\qigtok$ actually has a relatively simple form, as we show in the next lemma.
\begin{lemma}\label{lem:QA_fedavg}
Suppose $0 < \gamma < L_i^{-1}$. Then
\[
\qigtok = (I-(I-\gamma A_i))(\gamma A_i)^{-1}.\]
\end{lemma}
Note that when $\gamma = 0$, \eqref{eq:Q_matrix} implies $\qigtok = KI$.

We are also interested in the matrix $Q_i(\gamma, \Theta)A_i$, as its eigenvalues govern the Lipschitz and strong convexity parameters of the function $\tilde{f}_i(x, \gamma, \Theta)$. We have the following result.
\begin{lemma}\label{lem:QA_structure}
    Let $L_i = \lambda_{\max}(A_i), \mu_i = \lambda_{\min}(A_i)$.
    \begin{enumerate}
        \item For each eigenvector and eigenvalue pair $(v, \lambda)$ of $A_i$, $v$ is an eigenvector of $Q_i(\gamma, \Theta)A_i$ with eigenvalue
        \[
        \sum_{k =1}^{K(\Theta)} \theta_k(1-\gamma\lambda)^{k-1}\lambda.
        \]
        \item If $\gamma < L_i^{-1}$, $Q_i(\gamma, \Theta)A_i$ is symmetric and positive definite with eigenvalues satisfying
        \[
        \lambda_{\max}(Q_i(\gamma, \Theta)A_i) \leq \sum_{k = 1}^{K(\Theta)} \theta_k(1-\gamma \mu_i)^{k-1}L_i.
        \]
        \[
        \lambda_{\min}(Q_i(\gamma, \Theta)A_i) \geq \sum_{k = 1}^{K(\Theta)} \theta_k(1-\gamma L_i)^{k-1}\mu_i.
        \]
    \end{enumerate}
\end{lemma}

The bounds in Lemma \ref{lem:QA_structure} can be refined for specific $\Theta$. We first consider FedAvg/Reptile, when $\Theta = \Theta_{1:K}$. By Lemma \ref{lem:QA_fedavg}, when $0 < \gamma < L_i^{-1}$, we have
\[
\qigtok A_i = \dfrac{I-(I-\gamma A_i)^K}{\gamma}.
\]
We will therefore be able to compute the eigenvalues of $\qigtok A_i$ in terms of the function
\begin{equation}\label{eq:phi}
\phi_{K, \lambda}(\gamma) := \gamma^{-1}(1-(1-\gamma\lambda)^K).
\end{equation}

In fact, $\phi_{K, \lambda}(\gamma)$ is actually continuous at $0$, with its value being given by $\phi_{K, \lambda}(0) = K\lambda$. One way to see this is by noting that by basic properties of geometric sums, for $0 \leq \gamma \leq \lambda^{-1}$,
\begin{equation}\label{eq:phi_alt}
\phi_{K, \lambda}(\gamma) = \sum_{k=1}^K (1-\gamma\lambda)^{k-1}\lambda.
\end{equation}

We can now give strong bounds on the spectrum of $\qigtok A_i$. We get the following:

\begin{lemma}\label{lem:QA_structure_fedavg}
    Let $L_i = \lambda_{\max}(A_i), \mu_i = \lambda_{\min}(A_i)$.
    \begin{enumerate}
        \item For each eigenvector, eigenvalue pair $(v, \lambda)$ of $A_i$, $v$ is an eigenvector of $\qigtok A_i$ with eigenvalue $\phi_{K, \lambda}(\gamma)$. 
    \item If $\gamma < L_i^{-1}$, the maximum and minimum eigenvalues of $\qigtok A_i$ are given by
        \[
        \lambda_{\max}(Q_i(\gamma, \Theta_{1:K})A_i) = \phi_{K, L_i}(\gamma),
        \]
        \[
        \lambda_{\min}(Q_i(\gamma, \Theta_{1:K})A_i) = \phi_{K, \mu_i}(\gamma).
        \]
    \item If Assumption \ref{assm1} holds and $\gamma < L^{-1}$, then
        \[
        \lambda_{\max}(Q_i(\gamma, \Theta_{1:K})A_i) \leq \phi_{K, L}(\gamma),
        \]
        \[
        \lambda_{\min}(Q_i(\gamma, \Theta_{1:K})A_i) \geq \phi_{K, \mu}(\gamma).
        \]
    \end{enumerate}
\end{lemma}

We can also tighten the bounds in Lemma \ref{lem:QA_structure} for the MAML-style algorithms, where $\Theta = \Theta_K$, as long as the learning rate is set appropriately.

\begin{lemma}\label{lem:QA_structure_maml}
    Let $L_i = \lambda_{\max}(A_i), \mu_i = \lambda_{\min}(A_i)$.
    \begin{enumerate}
        \item For each eigenvector, eigenvalue pair $(v, \lambda)$ of $A_i$, $v$ is an eigenvector of $Q_i(\gamma, \Theta_K) A_i$ with eigenvalue $(1-\gamma\lambda)^{K-1}\lambda$. 
    \item If $\gamma < (KL_i)^{-1}$, the maximum and minimum eigenvalues of $Q_i(\gamma, \Theta_K) A_i$ are given by
        \[
        \lambda_{\max}(Q_i(\gamma, \Theta_{K})A_i) = (1-\gamma L)^{K-1}L,
        \]
        \[
        \lambda_{\min}(Q_i(\gamma, \Theta_{K})A_i) = (1-\gamma\mu)^{K-1}\mu.
        \]
    \item If Assumption \ref{assm1} holds and $\gamma < (KL)^{-1}$, then
        \[
        \lambda_{\max}(Q_i(\gamma, \Theta_{K})A_i) \leq (1-\gamma L)^{K-1}L,
        \]
        \[
        \lambda_{\min}(Q_i(\gamma, \Theta_{K})A_i) \geq (1-\gamma\mu)^{K-1}\mu.
        \]
    \end{enumerate}
\end{lemma}

Lemmas \ref{lem:QA_structure_fedavg} and \ref{lem:QA_structure_maml} imply the following results regarding the condition number of the surrogate loss.

\begin{corollary}\label{cor:cond_fedavg}
    If Assumption \ref{assm1} holds and $\gamma < L^{-1}$, then $\tilde{f}(x, \gamma, \Theta_{1:K})$ is $\tilde{L}$-smooth and $\tilde{\mu}$-strongly convex where
    \[
    \tilde{L} \leq \phi_{K, L}(\gamma),
    \]
    \[
    \tilde{\mu} \geq \phi_{K, \mu}(\gamma).
    \]
    Therefore, $\tilde{f}(x, \gamma, \Theta_{1:K})$ has condition number $\tilde{\kappa}$ satisfying
    \[
    \tilde{\kappa} \leq \dfrac{\phi_{K, L}(\gamma)}{\phi_{K, \mu}(\gamma)} = \dfrac{1-(1-\gamma L)^K}{1-(1-\gamma\mu)^K}.
    \]
\end{corollary}

\begin{corollary}\label{cor:cond_maml}
    If Assumption \ref{assm1} holds and $\gamma < (KL)^{-1}$, then $\tilde{f}(x, \gamma, \Theta_{K})$ is $\tilde{L}$-smooth and $\tilde{\mu}$-strongly convex where
    \[
    \tilde{L} \leq (1-\gamma L)^{K-1}L,
    \]
    \[
    \tilde{\mu} \geq (1-\gamma \mu)^{K-1}\mu.
    \]
    Therefore, $\tilde{f}(x, \gamma, \Theta_{1:K})$ has condition number $\tilde{\kappa}$ satisfying
    \[
    \tilde{\kappa} \leq \left(\dfrac{1-\gamma L}{1-\gamma\mu}\right)^{K-1}\dfrac{L}{\mu}.
    \]
\end{corollary}

\paragraph{Local computation as implicit regularization}

For $\Theta = \Theta_K$, we clearly see that as $K \to 1$ or $\gamma \to 0$, then $\tilde{\kappa} \to L/\mu$, the condition number of the true loss function. However, if $\gamma$ is not close to 0, we see an exponential reduction (in terms of $K$) of the condition number. While the analysis is not quite as clear for $\Theta = \Theta_{1:K}$, one can show that for all $\gamma$, $\tilde{\kappa} \leq L/\mu$, with equality if and only if $\gamma = 0$ or $K = 1$. Moreover, the condition number decreases as $K \to \infty$ or $\gamma \to L^{-1}$.

It is well known that the condition number measures how quickly methods such as gradient descent can find a minimizer of a strongly convex function~(see Chapter 3 of \citep{bubeck2014convex} for reference). By performing more local computations on the clients, and with larger learning rate, we actually reduce the condition number of the surrogate loss. We see that intuitive notions about methods such as FedAvg (e.g. that more local computation improves convergence) can be made formal by analyzing properties of the surrogate loss. Thus, we have the following important takeaway:
\begin{center}
    \textbf{Methods such as MAML, FedAvg, and Reptile perform implicit regularization on the condition number of the surrogate loss function they are actually optimizing.}
\end{center}

When $\Theta = \Theta_K$ or $\Theta_{1:K}$, we see that \localupdate may be able to optimize the surrogate loss more quickly (due to the condition number reduction). However, as shown in Section \ref{sec:example}, the surrogate loss may differ drastically from the true loss. In the next section, we use the spectral properties of the surrogate loss derived above to quantify the distance between the minimizers of these two functions.

\section{Bounding the distance between global minima}\label{sec:distance}

As discussed above, \localupdate is not optimizing the desired loss function $f(x)$, but a surrogate loss function $\tilde{f}(x, \gamma, \Theta)$. In this section, we will bound the distance between the minima of these two functions. We will assume that Assumptions \ref{assm0}, \ref{assm1}, and \ref{assm2} hold throughout. Recall that by Lemma \ref{cor:cond_fedavg}, as long as $\gamma < L^{-1}$, $\tilde{f}(x, \gamma, \Theta)$ is strongly convex and therefore has a unique minimizer.

For simplicity of notation, we will fix $\gamma, \Theta$ and let $Q_i$ refer to $Q_i(\gamma, \Theta)$ throughout this section. We will also define matrices $\tilde{A}_i$ and $\tilde{A}$, depending on $\Theta$, as follows:
\begin{equation}\label{eq:distance_quantities}
\tau := \rbr*{\sum_{k=1}^{K(\Theta)} \theta_k}^{-1},~~~\tilde{A}_i := \tau Q_iA_i,~~~\tilde{A} := \E_i\sbr*{\tilde{A}_i}.
\end{equation}

We also define the following quantities:
\[
x^*(\gamma, \Theta) := \argmin_{x} \tilde{f}(x, \gamma, \Theta),~~x^* := \argmin_x f(x).
\]

By Lemma \ref{lem:Q_special_case}, if $\gamma = 0$ or $\theta_k = 0$ for $k \geq 2$, then $Q_i = \tau^{-1}I$. Therefore, $\tilde{f}(x, \gamma, \Theta) = \tau^{-1}f(x)$, and so $x^*(0, \Theta) = x^*$. However, for general $\gamma$ and $\Theta$, $x^*(\gamma, \Theta)\neq x^*$.

We are first interested in how far apart the two minimizers can possibly be. We first consider the asymptotic affect of $K$ when setting $\Theta = \Theta_{1:K}$. In fact, varying $K$ can only change the distance between the two by a fixed amount, as shown in the following.

\begin{lemma}\label{lem:asymptotic}
    Suppose $0 < \gamma < L^{-1}$. Then for all $K \geq 1$, 
        \[
        \tilde{f}(x, \gamma, \Theta_{1:K}) \leq \frac{1}{2\gamma}\E_{i\sim\mP}\norm{x-c_i}^2.
        \]
    Moreover, $\tilde{f}(x, \gamma, \Theta_{1:K})$ converges pointwise to this function, ie.
    \[
    \lim_{K \to \infty} \tilde{f}(x, \gamma, \Theta_{1:K}) = \dfrac{1}{2\gamma}\E_{i \sim \mP}[\norm{x-c_i}^2]
    \]
    which has a unique minimizer at $x = \E_{i \sim \mP}\sbr{c_i} = c$.
\end{lemma}

Intuitively, we see that as long as the client learning rate is not too high, the worst possible surrogate loss is the one defined by the average distance to the client optimizers. Intuitively, as $K \to \infty$, \localupdate will take steps oriented more and more towards the average of the client minimizers (one-shot averaging), which is reflected in the experiment in Figure~\ref{fig:1d-example}.

We now wish to understand the non-asymptotic regime, especially the distance between the surrogate risk minimizer and the true risk minimizer, as this will help inform us how to set $\gamma$ in \localupdate. We have the following result.

\begin{theorem}\label{thm:optima_distance}
If $\gamma < L^{-1}$, then
\begin{equation}\label{eq:optima_distance}
\norm{x^*(\gamma, \Theta) - x^*} \leq \dfrac{L\sigma_c}{\mu}\rbr*{1+\dfrac{\sigma_A}{\mu}}\dfrac{1- \chi(\gamma, \Theta)}{ \chi(\gamma, \Theta)}
\end{equation}
where
\begin{equation}\label{eq:chi}
\chi(\gamma, \Theta) := \dfrac{\sum_{k = 1}^{K(\Theta)} \theta_k (1-\gamma L)^{k-1}}{\sum_{k = 1}^{K(\Theta)} \theta_k}.
\end{equation}

\end{theorem}

When $\Theta = \Theta_{1:K}$ (as in FedAvg/Reptile), we can derive an even tighter bound that omits the direct dependency on $L/\mu$. Recall that by Corollary \ref{cor:cond_fedavg}, the condition number of $\tilde{f}(x, \gamma, \Theta_{1:K})$ is bounded above by
\[
\tilde{\kappa} = \dfrac{\phi_{K, L}(\gamma)}{\phi_{K, \mu}(\gamma)}
\]
where  $\phi_{K, \lambda}(\gamma)$ is as in \eqref{eq:phi}. We will see in the following theorem that the distance between minimizers is controlled by $\phi_{K, L}(\gamma)$ and $\phi_{K, \mu}(\gamma)$.

\begin{theorem}\label{thm:optima_distance_fedavg}
    If $\gamma < L^{-1}$, then
    \[
    \norm{x^*(\gamma, \Theta_{1:K}) - x^*} \leq \sigma_c\left(1+\frac{\sigma_A}{\mu}\right)\dfrac{LK - \phi_{K, L}(\gamma)}{\phi_{K, \mu}(\gamma)}.
    \]
\end{theorem}

Informally, the term $KL-\phi_{K, L}(\gamma)$ measures the discrepancy between the surrogate loss function and the true loss function; as $\gamma \to 0$, $\phi_{K, L}(\gamma) \to KL$.

In order to get better control on Theorem \ref{thm:optima_distance_fedavg} for $\gamma > 0$, we will show that when $\gamma$ is sufficiently small, $\phi_{K, \lambda}(\gamma)$ is close to $K\lambda$.

\begin{lemma}\label{lem:gamma_cond}
    Let $0 \leq \epsilon \leq 1-e^{-K}$ and suppose
    \[
    \gamma \leq  \dfrac{\ln(1/(1-\epsilon))}{K\lambda}.
    \]
    Then
    \[
    \phi_{K, \lambda}(\gamma) \geq (1-\epsilon)K\lambda.
    \]
\end{lemma}
This results in the immediate corollary.
\begin{corollary}
    Suppose $\epsilon \leq 1-e^{-K}$ and
    \[
    \gamma \leq \dfrac{\ln(1/(1-\epsilon))}{KL}.
    \]
    Then
    \[
    \phi_{K, L}(\gamma) \geq (1-\epsilon)KL~~\text{and}~~
    \phi_{K, \mu}(\gamma) \geq (1-\epsilon)K\mu.
    \]
\end{corollary}
Plugging these into Theorem \ref{thm:optima_distance_fedavg}, we get the following.
\begin{corollary}\label{cor:optima_distance_fedavg}
For $\epsilon \leq 1-e^{-K}$, suppose
\[
\gamma \leq \dfrac{\ln(1/(1-\epsilon))}{KL}.
\]
Then
\[
\norm{x^*(\gamma, \Theta_{1:K}) - x^*} \leq \sigma_c\left(1+\frac{\sigma_A}{\mu}\right)\dfrac{L}{\mu}\dfrac{\epsilon}{1-\epsilon}.
\]
\end{corollary}

\newcommand{\qigti}{Q_i(\gamma_1, \Theta_{1:K})}
\newcommand{\qigtii}{Q_i(\gamma_2, \Theta_{1:K})}

We can also use a similar analysis to bound the distance between $x^*(\gamma, \Theta)$ for different values of $\gamma$. We will focus on the FedAvg/Reptile case. We will show that this distance depends on the discrepancy between the eigenvalues of the matrices $\qigtok A_i$.

\begin{theorem}\label{thm:distance_optima_gamma}
Let $\gamma_1 \leq \gamma_2 < L^{-1}$. Then
    \[
    \norm{x^*(\gamma_1, \Theta_{1:K}) - x^*(\gamma_2, \Theta_{1:K})} \leq \sigma_c\left(1 + \dfrac{\phi_{K, L}(\gamma_2)}{\phi_{K, \mu}(\gamma_2)}\right)\left(\dfrac{\phi_{K, L}(\gamma_1) - \phi_{K, L}(\gamma_2)}{\phi_{K, \mu}(\gamma_1)}\right).
    \]
\end{theorem}

The presence of the $\phi$ terms makes the dependence on $|\gamma_1-\gamma_2|$ somewhat opaque. In fact, we have the following simpler (though looser) bound.

\begin{corollary}\label{cor:distance_optima_gamma}
Let $\gamma_1 \leq \gamma_2 < L^{-1}$. Then
    \[
    \norm{x^*(\gamma_1, \Theta_{1:K}) - x^*(\gamma_2, \Theta_{1:K})} \leq 2\sigma_c\dfrac{L^3}{\mu^2}(\gamma_2-\gamma_1).
    \]
\end{corollary}

One particularly useful consequence is that if the client learning rate satisfies $\gamma_t = c/t$ for some constant $c$, then
\[
\norm{x^*(\gamma_{t+1}, \Theta_{1:K}) - x^*(\gamma_{t}, \Theta_{1:K})} \leq O\left(\dfrac{1}{t^2}\right).
\]
We will use this later to show that by decaying the client learning rate in this manner, successive model updates in \localupdate will be closely aligned.

\section{Convergence of FedAvg/Reptile}\label{sec:convergence}

We now wish to use Theorem \ref{thm:optima_distance_fedavg} to understand how quickly FedAvg/Reptile with $K$ local steps converges to $x^*$. However, by Lemma \ref{thm:sgd_objective}, we know that performing FedAvg/Reptile with a fixed client learning rate of $\gamma$ will only result in convergence to $x^*(\gamma, \Theta_{1:K})$, the minima of the surrogate loss $\tilde{f}(x, \gamma, \Theta_{1:K})$, not to the true risk minimizer $x^*$. We will therefore analyze the convergence behavior of these algorithms with and without learning rate decay.

Our goal in this section is two-fold. First, we wish to understand the trade-offs incurred by performing local computation instead of mini-batch SGD. Second, we wish to show that by Theorem \ref{thm:sgd_objective}, we can analyze federated learning and meta-learning algorithms using classical optimization techniques. There are a large number of important works on federated optimization, that consider more general cases than ours. Unfortunately, the proof techniques behind many of these are relatively opaque, and require careful accounting of the bias incurred by performing local computation. This sometimes leads to either proof errors, or else omitted critical assumptions~\citep[Appendix A]{woodworth2020local}. Both of these can hinder understanding or make comparisons between convergence rates difficult. By contrast, while limited to a much narrower range of loss functions, our analysis uses essentially standard convex optimization analyses (such as by \citet{rakhlin2011making} and \citet{bottou2018optimization}), combined with the results from Section \ref{sec:distance}. We also emphasize that while we focus on FedAvg/Reptile, our results can be easily extended to more general instances of \localupdate.

\subsection{Fixed client learning rate}

We first wish to understand the setting where the client learning rate $\gamma$ is fixed. Fix a client step-size $\gamma$ and $K$, and for notational convenience, define
    \[
    \tilde{f}_\gamma(x) := \tilde{f}(x, \gamma, \Theta_{1:K}),
    \]
    \[
    \tilde{f}_\gamma^* := \min_x \tilde{f}_\gamma(x),
    \]
    \[
    x^*_\gamma :=\argmin_x \tilde{f}_\gamma(x).
    \]
As $\gamma$ tends towards $0$, $\tilde{f}_\gamma(x)$ converges to the true loss function $f(x)$ (under different notions of convergence depending on the set of assumptions made). For example, under Assumption \ref{assm1} this convergence will occur uniformly on $\R^d$.

Fix $\gamma_t = \gamma$ for all $t$  in \localupdate. Then, at each iteration $t$ the server starts at a point $x_t$ which it broadcasts to some number of clients. The clients compute local updates $q_t^i$ via $\InnerLoop(\gamma, \Theta)$ (Algorithm \ref{alg:innerloop}), and send these values to the server. The server then computes the average $q_t$ of the $q_t^i$ and updates its model via $x_{t+1} = x_t-\eta_tq_t$.

Given a starting point $x$, we let $q_\gamma(x)$ denote the random vector computed by averaging $M$ vectors of the form $q^i = \InnerLoop(i, x, \gamma, \Theta)$ where $i \sim \mP$. Thus, in Algorithm \ref{alg:outerloop} with a constant client learning rate $\gamma$, $x_{t+1} = x_t - \eta_tq_\gamma(x_t)$. Recall that by Theorem \ref{thm:sgd_objective},
\[
    \E[q_\gamma(x)] = \nabla \tilde{f}_\gamma(x).
\]

Throughout this section, we will assume Assumptions \ref{assm1} and \ref{assm2}, as well as the following ``bounded variance'' condition.

\begin{assumption}\label{assm3}
For all $x$ and $i$, $\E_{z \sim \mD_i}\norm{\nabla_x f(x ; z) - \nabla f_i(x)}^2 \leq G^2$.
\end{assumption}
We can then translate this into a bound on the variance of $q(x, \gamma, \Theta_{1:K})$.

\begin{lemma}\label{lem:var_bound}
    Suppose Assumption \ref{assm3} holds. Then for all $\gamma \geq 0$,
    \[
    \E\norm{q_\gamma(x) - \nabla \tilde{f}_\gamma(x)}^2 \leq \dfrac{KG^2}{MB}.
    \]
\end{lemma}

We will also use the following bound on the strong convexity parameter of our surrogate loss functions.

\begin{lemma}\label{lem:ft_params}
    Suppose that Assumption \ref{assm1} holds. Then for all $t$, $\tilde{f}_\gamma$ is $\mu_\gamma$-strongly convex where
    \[
    \mu_\gamma := \phi_{K, \mu}(\gamma) = \dfrac{1-(1-\gamma \mu)^K}{\gamma}.
    \]
    and $L_\gamma$-smooth where
    \[
    L_\gamma := \phi_{K, L}(\gamma) = \dfrac{1-(1-\gamma L)^K}{\gamma}.
    \]    
\end{lemma}

Note that these results follow directly from Lemma \ref{lem:QA_structure_fedavg} and the fact that the strong convexity and smoothness parameters are governed by the maximum and minimum eigenvalues of $\qigtok A_i$.

Using techniques similar to those in \citep{rakhlin2011making}, we arrive at the following descent lemma.

\begin{lemma}\label{lem:descent_step}
Suppose that Assumptions \ref{assm1} and \ref{assm3} hold, and that we have step sizes $\{\eta_t\}_{t \geq 1}$ satisfying
\begin{equation}\label{eq:eta_t_bound}
\eta_t \leq \dfrac{\mu_\gamma}{L_\gamma^2}.
\end{equation}
Then
\[
    \E[\norm{x_{t+1} - x_\gamma^*}^2] \leq (1-\eta_t\mu_\gamma)\E[\norm{x_t - x_\gamma^*}^2] + \eta_t^2\dfrac{KG^2}{MB}.
\]
\end{lemma}

\paragraph{On upper bounds for server learning rates}
Note that in Lemma \ref{lem:descent_step}, we assume a slightly stronger condition than is often assumed in optimization literature, namely that $\eta \leq 1/L_\gamma\kappa_\gamma$ where $\kappa_\gamma$ is the condition number. Typically, works on optimization would only require $\eta$ to be at most the inverse of the Lipschitz constant.
While it is an open question as to whether this condition is necessary, there are a few relevant factors. First, we note that we can relax \eqref{eq:eta_t_bound} to $\eta_t \leq L_\gamma^{-1}$ if we strengthen Assumption \ref{assm3} to a bounded gradient assumption instead of a bounded variance assumption. Second, when $\gamma$ is moderately large with respect to $\mu$, Lemma \ref{cor:cond_fedavg} implies that $\mu_\gamma \approx L_\gamma$, so \eqref{eq:eta_t_bound} gives a similar condition to assuming $\eta_t \leq L_\gamma^{-1}$. Finally, we note that a similar bound on the learning rate was used by \citet{reisizadeh2019fedpaq} in conjunction with a bounded variance assumption. While we conjecture that this condition can be relaxed, we leave this for future work.

Applying Lemma \ref{lem:descent_step} repeatedly, we derive at the following.

\begin{theorem}[Fixed client LR, fixed server LR]\label{thm:fix_client_fix_server}
Suppose that $\eta \leq \mu_\gamma/L_\gamma^2$. Then the outputs of $\localupdate(\eta, \gamma, \Theta_{1:K})$ satisfy
\[
\E[\norm{x_{t+1}-x_\gamma^*}^2] \leq (1-\eta\mu_\gamma)^t\norm{x_1-x_\gamma^*}^2 + \dfrac{\eta KG^2}{\mu_\gamma MB}.
\]
\end{theorem}

Many prior results for fixed client learning rate provide a bound of the same general form as Theorem~\ref{thm:fix_client_fix_server} (ie. a sum of a decaying term and a constant error term), but bound the distance from $x^*$, rather than from $x_\gamma^*$. This makes the error term's significance more opaque. While non-federated optimization results often have constant error terms due to stochasticity, in federated convergence results  (eg. \cite[Theorem 5]{khaled2020tighter}), the constant term often does not disappear in deterministic setting ($G=0$). To the reader, it may not be immediately clear why this is the case. This could be due to actual convergence properties, or due to the analysis not being tight. By contrast, our result shows that this error term is an inherent property of the algorithm, as in general, $x_\gamma^* \neq x^*$.

As is the case in general stochastic optimization, a constant learning rate is only sufficient to arrive in a neighborhood of the critical point of the underlying loss. However, this critical point is not the true risk minimizer $x^*$. As in Section \ref{sec:example}, we see that in heterogeneous settings, client learning rate decay is necessary for convergence to the true risk minimizer.

To make the suboptimality gap tend towards zero, we must decay the server learning rate $\eta$ over time, as in the following theorem.

\begin{theorem}[Fixed client LR, decaying server LR]\label{thm:fix_client_decay_server}
Suppose that for all $t \geq 1$,
\[
\eta_t = \dfrac{a_\gamma}{b_\gamma+t}~~\text{for}~~a_\gamma = \frac{2}{\mu_\gamma}~~\text{and}~~b_\gamma \geq \frac{2L_\gamma^2}{\mu_\gamma^2}.
\]
Then the outputs of $\localupdate(\{\eta_t\}, \gamma, \Theta_{1:K})$ satisfy
\[
\E[\norm{x_{t}-x_\gamma^*}^2] \leq \dfrac{\nu_\gamma}{b_\gamma + t}
\]
where
\[
\nu_\gamma := \max\left\{\dfrac{4KG^2}{\mu_\gamma^2MB}, (b_\gamma + 1)\norm{x_1-x_\gamma^*}^2 \right\}.
\]
\end{theorem}

We can now derive a convergence rate towards the true risk minimizer, $x^*$.

\begin{corollary}\label{cor:fix_client_decay_server}
Fix $\epsilon$ satisfying $\epsilon \leq 1-e^{-K}$. Suppose
\[
\gamma \leq \dfrac{\ln(1/(1-\epsilon))}{KL}
\]
and that $\eta_t$ is as in Theorem \ref{thm:fix_client_decay_server}. Then the outputs of $\localupdate(\{\eta_t\}, \gamma, \Theta_{1:K})$ satisfy
\[
\E[\norm{x_t-x^*}^2] \leq \dfrac{2\nu_\gamma}{b_\gamma + t} + 2\sigma_c^2\left(1 + \dfrac{\sigma_a}{\mu}\right)^2\dfrac{L^2}{\mu^2}\dfrac{\epsilon^2}{(1-\epsilon)^2}
\]
where
\[
\nu_\gamma \leq \max\left\{\dfrac{4G^2}{(1-\epsilon)^2\mu^2 KMB} , (b_\gamma+1)\norm{x_1-x_\gamma^*}^2\right\}.
\]
\end{corollary}

While notationally complex, this result has a few important facets. Define
\[
\beta_1 = \dfrac{G^2}{(1-\epsilon)^2\mu^2 KMB},~~\beta_2 = (b_\gamma+1)\norm{x_1-x_\gamma^*},~~\beta_3 = \sigma_c^2\left(1 + \dfrac{\sigma_a}{\mu}\right)^2\dfrac{L^2}{\mu^2}\dfrac{\epsilon^2}{(1-\epsilon)^2}.
\]
Thus, this result shows that,
\[
\E[\norm{x_t-x^*}^2] \leq \mathcal{O}\rbr*{ \frac{\max\{\beta_1, \beta_2\}}{t} + \beta_3}.
\]

As $\epsilon \to 0$, $\gamma \to 0$, \localupdate becomes roughly equivalent to mini-batch SGD with batches of size $KB$ on $M$ clients. In fact, $\beta_1$ is the convergence term we would derive from performing mini-batch SGD with batches of size $KB$ on $M$ clients per round. In particular, we get a variance reduction of $(KMB)^{-1}$. This variance reduction is similar in nature to work by \cite{woodworth2020local} in the homogeneous setting ($\sigma_A = \sigma_c = 0$), which shows an analogous improvement in convergence rates. Note that $\beta_1$ gets larger as $\epsilon$ gets larger. Thus, the variance reduction is reduced in heterogeneous settings.

The term $\beta_2$ measures the initial suboptimality to the surrogate loss function. It is here that having a small $\gamma$ may incur a price. Note that $b_\gamma$ is twice the condition number of $\tilde{f}_\gamma$. As shown in Corollary \ref{cor:cond_fedavg} and in the discussion thereafter, larger $\gamma$ lead to smaller $b_\gamma$ due to the implicit regularization of local computation. Conversely, at $\gamma = 0$, $b_\gamma = 2 L^2/\mu^2$, which may dominate $\beta_1$. Thus, smaller $\gamma$ lead to a larger effective initial suboptimality gap.

The last term $\beta_3$ measures the discrepancy between the surrogate loss function and the true loss function. If $\sigma_c = 0$, for instance, when the data is completely homogeneous, this discrepancy is 0. We then recover similar results to that of \cite{woodworth2020local}, which shows a $1/K$ improvement in convergence for Local SGD with $K$ steps, but in the heterogeneous setting. If $\sigma_c \neq 0$, we can still remove the effect of heterogeneity by setting $\epsilon = 0$ (which means setting $\gamma = 0$). In this case, FedAvg reduces to mini-batch SGD with batches of size $KB$ on each client.

\subsection{Decaying client and server learning rates}

In this section, we will show that by decaying the client and server learning rates appropriately, we can derive a bound on $\norm{x_t-x^*}^2$ that does not require setting the learning rates in terms of the desired optimality gap $\epsilon$.

Throughout this section, we again focus on the FedAvg/Reptile setting. At every iteration $t$, we will use client learning rate $\gamma_t$ and server learning rate $\eta_t$ that decay at a $1/t$ rate. Let $\kappa = L/\mu$ denote the condition number of $f(x)$. We have the following theorem.

\begin{theorem}[Decaying client LR, decaying server LR]\label{thm:convergence_overall}
Let
\[
\eta_t = \dfrac{a_t}{b+t}~~\text{for}~~a_t = \frac{3}{\mu_t}~~\text{and}~~b = 3\kappa^2
\]
and let
\[
\gamma_t = \min\left\{\dfrac{1}{L(b + t)}, \dfrac{\ln(2)}{K\mu}\right\}.
\]
Suppose we run \localupdate with $\gamma_t, \eta_t$ as above and $\Theta = \Theta_{1:K}$ to produce iterates $x_t$. Then for all $t \geq 1$,
\[
\E[\norm{x_t-x^*}^2] \leq \dfrac{2\nu + 16\sigma_c^2\kappa^2}{b + t}.
\]
where
\[
\nu := \max\left\{\dfrac{18G^2}{\mu^2 KMB}, (b+1)\norm{x_1-x_1^*}^2\right\}.
\]
\end{theorem}

No attempt was made to optimize constants. Rather, the point was to show that we can derive bounds on the distance to the true risk minimizer $x^*$ that hold for all $t$ and that help illustrate the implicit trade-offs in \localupdate. In particular, we do not require that $\gamma$, $\eta$, or $t$ be defined in terms of the desired suboptimality gap $\epsilon$. Rather, all that was required was learning rate decay on the order of $O(1/t)$ at both the server and the client.

Recall that by Theorem \ref{thm:fix_client_fix_server}, some form of learning rate decay is necessary for convergence to the true risk minimizer. We therefore see that in some settings, learning rate decay is sufficient for \localupdate to converge to the true risk minimizer. While this was previously shown in the finite-sum setting for strongly convex functions by \citet{li2019convergence}, this result required bounded gradients, and did not illustrate the trade-offs in using FedAvg/Reptile over mini-batch SGD. 

\paragraph{Comparison to mini-batch SGD} Similar to Theorem 1 of \citep{woodworth2020local}, we see that performing \localupdate incurs a kind of variance reduction of $1/MB$ when compared to performing vanilla SGD over a shuffled version of the entire dataset. We also see that using \localupdate with $\gamma > 0$ does incur a potential benefit: Rather than having the convergence depend on the suboptimality gap $\norm{x_1-x^*}^2$ (as is the case for mini-batch SGD), it depends on $\norm{x_1-x_1^*}^2$. This may be much smaller depending on the initialization. For example, recall that by Lemma \ref{lem:asymptotic}, if $\gamma_1 > 0$, then as $K \to \infty$, $x_1^*$ tends to the ``one-shot average'' of the client minimizers, which typically requires many fewer communication rounds to estimate than the true risk minimizer. However, this reduction only benefits the convergence up to a point, in which case it becomes beneficial to use smaller $\gamma$.

As suggested by lower bounds on FedAvg in \citep{karimireddy2019scaffold} on FedAvg, our bounds do not show that FedAvg always converges faster than mini-batch SGD, and in fact, doing so may not be possible without further assumptions (such as a bound on the heterogeneity among clients) or more sophisticated optimization techniques (such as the use of control variates in SCAFFOLD \citep{karimireddy2019scaffold}).

While our theoretical results are only valid for the case of quadratic loss functions, we conjecture that even in much broader settings, the choice of learning rate $\gamma$ still dictates a trade-off between accuracy and initial convergence. we will show in the next section that this holds empirically, even in non-convex settings.

\section{Experimental results}\label{sec:exp_valid}

In this section, we analyze \localupdate empirically, in order to understand how the choice of client and server learning rates impact convergence in more realistic machine learning tasks. In particular, we focus on (not necessarily convex) tasks and datasets that reflect federated learning in practice. We will show that both the choice of client learning rate and tuning of the corresponding server learning rate can be vital to attain the best performance of \localupdate, especially in limited communication settings.

\paragraph{Datasets and models} We use four different datasets: the federated extended MNIST dataset (FEMNIST)~\citep{caldas2018leaf}, the federated version of CIFAR-100 created by \citet{reddi2020adaptive}, the Shakespeare dataset~\citep{caldas2018leaf} and the Stack Overflow dataset~\citep{stackoverflow}. The first two are image datasets, the second two are text datasets. All datasets are publicly available. We specifically use the versions available in TensorFlow Federated~\citep{ingerman2019tff}. All four datasets contain training and test clients.  For the purposes of our experiments, we only use the training clients, as our work only concerns the loss of clients in the training population. Notably, our work does not broach the subject of generalization, which we leave to future work. The number of clients and examples in each dataset is presented in Table \ref{table:datasets}.

\begin{table}[ht]
    \caption{Dataset statistics.}
    \label{table:datasets}
    \begin{center}
    \begin{sc}
    \begin{tabular}[t]{@{}lrr@{}}    
        \toprule
        Dataset & \# of Clients & Total \# of Examples \\
        \midrule
        CIFAR-100 & 500 & 50,000 \\
        FEMNIST & 3,400 & 671,585\\
        Shakespeare & 715 & 16,068\\
        Stack Overflow & 342,477 & 135,818,730\\
        \bottomrule
    \end{tabular}
    \end{sc}
    \end{center}
\end{table}

For FEMNIST, we train a moderately-sized CNN (the same as used in \citep{mcmahan17fedavg}) to perform character recognition. For CIFAR-100, we train a ResNet-18 (where we replace the batch norm layers with group norm, as suggested by \citep{hsieh2019non} and used by \citep{reddi2020adaptive}). For Shakespeare, we train an RNN with 2 LSTM layers to perform next-character-prediction. For Stack Overflow, we perform two distinct task: tag prediction (TP) and next-word-prediction (NWP). For Stack Overflow TP, we use a logistic regression classifier with one-versus-all classification. Note that this implies that Stack Overflow TP is a \emph{convex} task. For Stack Overflow NWP, we train an RNN with 1 LSTM layer to perform next-word-prediction. These five tasks were previously analyzed by \citet{reddi2020adaptive} for the purposes of comparing adaptive and non-adaptive federated optimization methods. For further details on datasets and models, see Appendix \ref{sec:models}.

\paragraph{Implementation and hyperparameters} We implement \localupdate in TensorFlow Federated~\citep{ingerman2019tff}. In all experiments, the set of clients $\mI$ is finite and we let $\mP$ be the uniform distribution over these clients. Each $\mD_i$ is the uniform distribution over some finite set of client examples. We analyze the performance of \localupdate with $\Theta = \Theta_{1:10}$ (ie. FedAvg/Reptile with $K = 10$ local steps) across the tasks discussed above. For FEMNIST, CIFAR-100, and Shakespeare, we sample $M = 10$ clients per round, while for Stack Overflow, we sample $M = 50$ clients (due to its much larger number of clients). In order to derive fair comparisons for different hyperparameter settings, we use a random seed to determine which clients are sampled at each round from $\mP$. All plots are made using the same seed. We sample clients without replacement within each round, but with replacement across rounds. We use a batch size of $B = 20$ for FEMNIST and CIFAR-100, $B = 4$ for Shakespeare, and $B = 16$ for Stack Overflow.

\subsection{Fixed server learning rates}

We first perform a comparable analysis to that in Section \ref{sec:example} above for FEMNIST, CIFAR-100, and Shakespeare. Namely, we fix the server learning rate $\eta = 0.01$, and see how the training loss varies as a function of the client learning rate $\gamma$. We vary $\gamma$ over
\[
\gamma \in \{0, 10^{-3}, \dots, 10^{-1}, 1, 10, 10^2\}
\]
and omit results that diverged due to the client learning rate being set too large. We plot the true loss function $f(x)$ (defined in \eqref{eq:objective}) in Figure \ref{fig:constant_server_lr}.

\begin{figure}[ht]
\centering
    \begin{subfigure}{.31\linewidth}
    \centering
    \includegraphics[width=\linewidth]{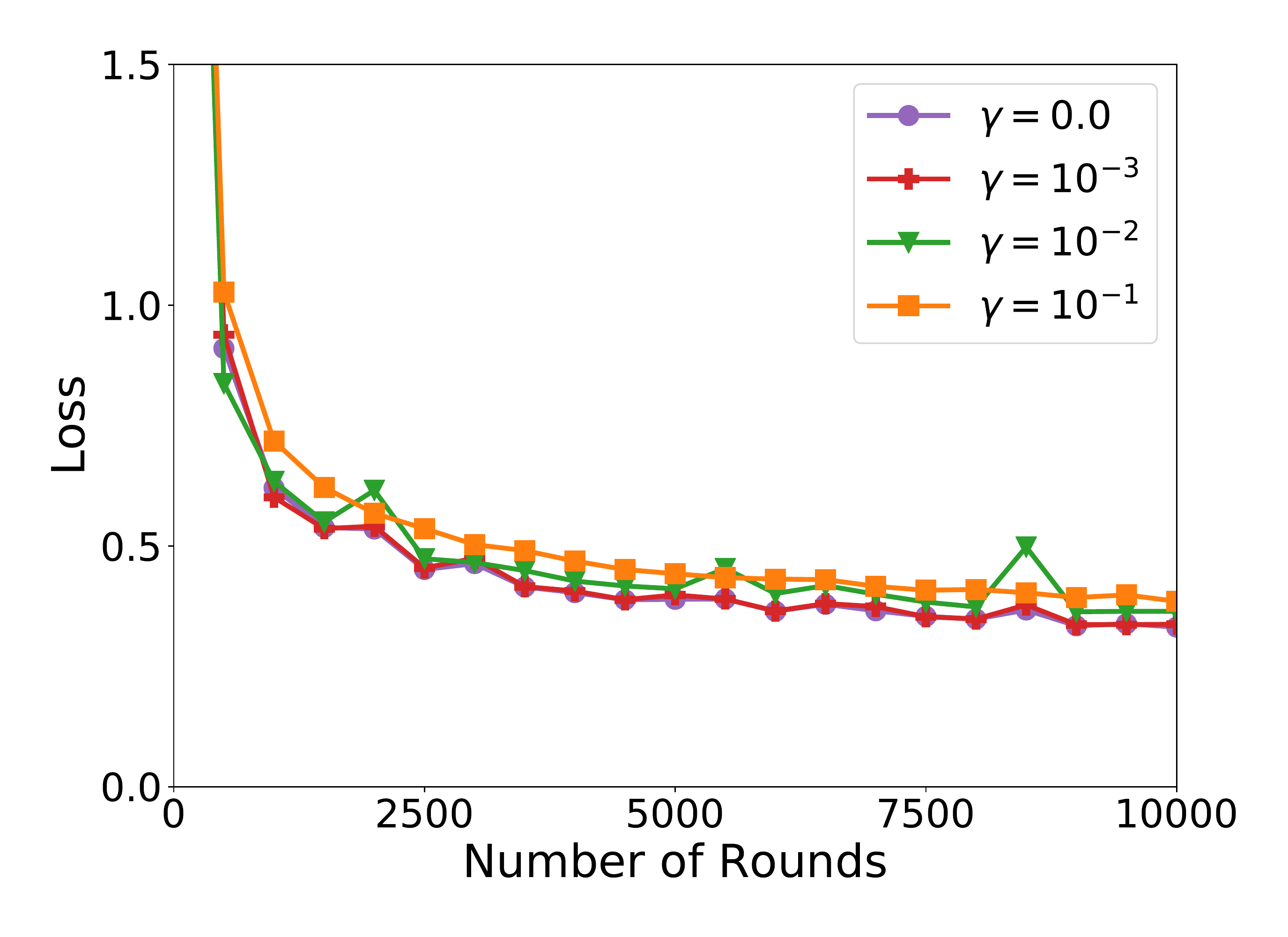}
    \caption{FEMNIST}
    \end{subfigure}
    \begin{subfigure}{.31\linewidth}
    \centering
    \includegraphics[width=\linewidth]{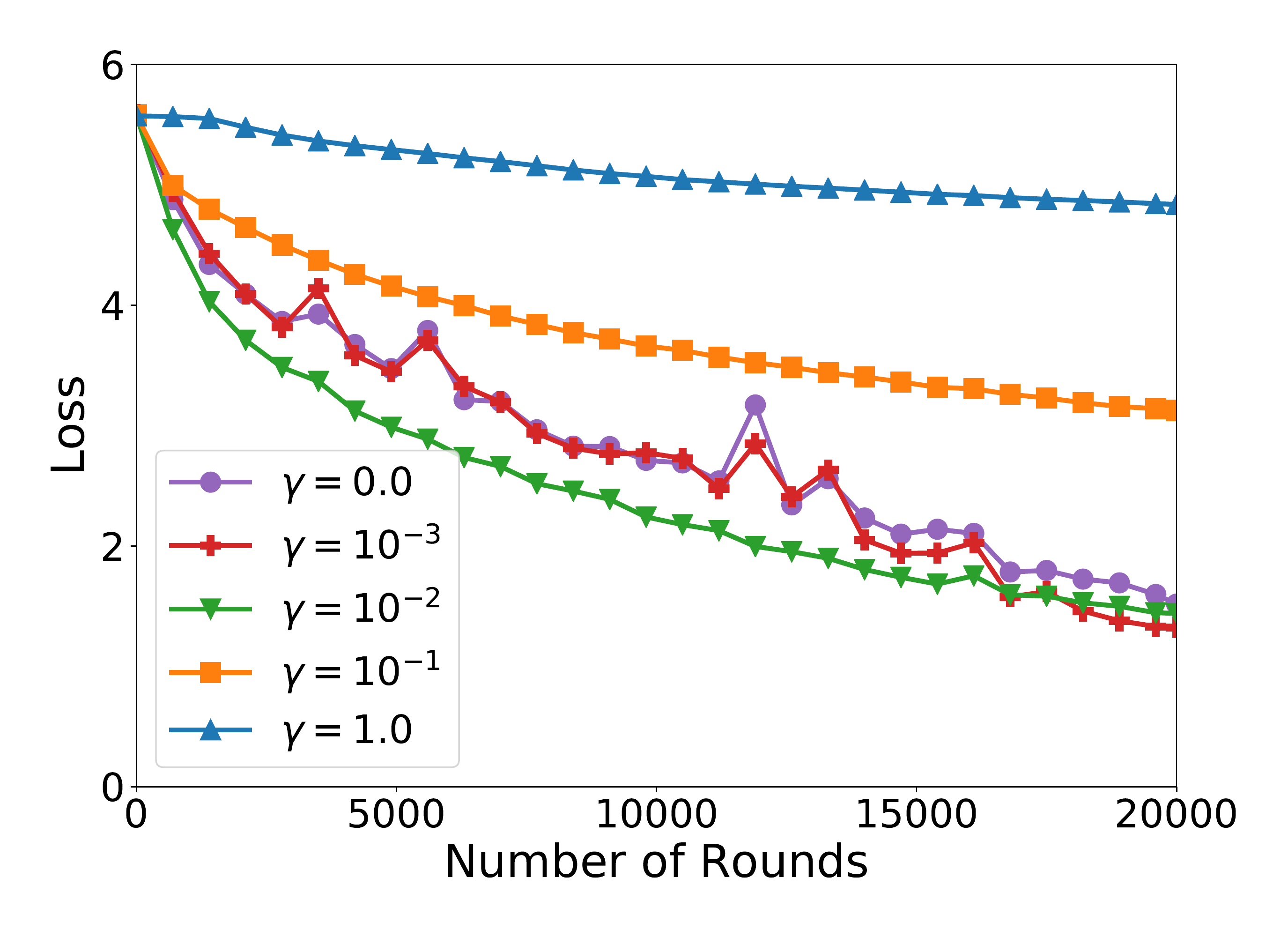}
    \caption{CIFAR-100}
    \end{subfigure}
    \begin{subfigure}{.31\linewidth}
    \centering
    \includegraphics[width=\linewidth]{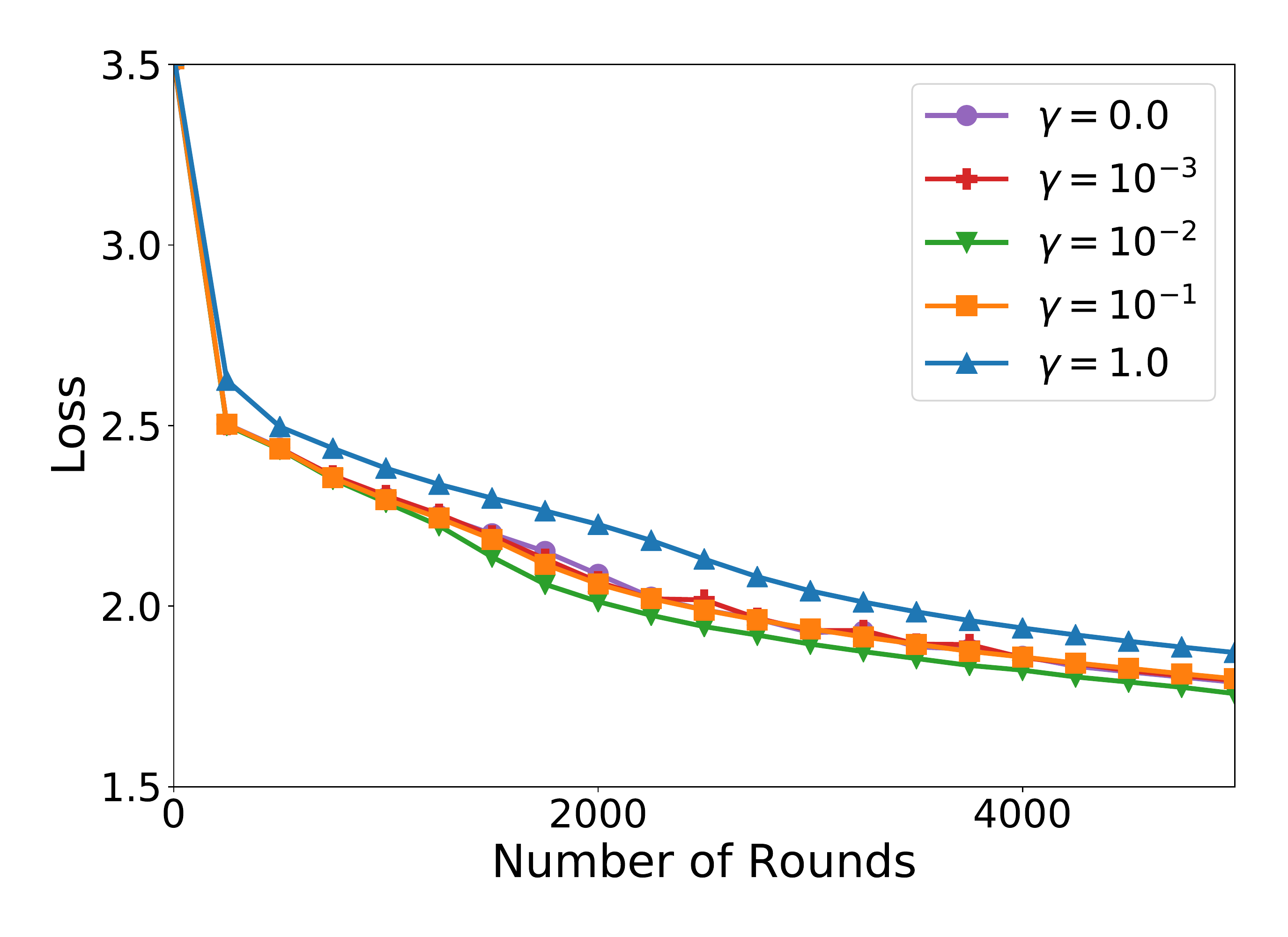}
    \caption{Shakespeare}
    \end{subfigure}
\caption{Cross-entropy training loss of \localupdate with fixed server learning rate $\eta = 0.01$ and $\Theta = \Theta_{1:10}$, and varying client learning rate $\gamma$.}
\label{fig:constant_server_lr}
\end{figure}

We see that on all three tasks, especially CIFAR-100, the choice of client learning rate can impact not just the speed of convergence, but what point the algorithm converges to.  In general, we see very similar behavior to that in Section \ref{sec:example}, despite the non-convex loss functions involved in all three tasks. For both FEMNIST and CIFAR-100, smaller client learning rates eventually reach lower training losses than higher learning rates. This is particularly evident in the results for CIFAR-100. While $\gamma = 10^{-2}$ initially performs better than all other methods, it is eventually surpassed by $\gamma = 10^{-3}$, and $\gamma = 0$ ends up obtaining a comparable accuracy. Conversely, setting $\gamma = 1.0$ results in a sub-optimal training loss.

We also see that while suboptimality gaps exist for FEMNIST and Shakespeare, they are much smaller than for CIFAR-100. Thus, our results also suggest that the theoretical suboptimality of larger $\gamma$ may not be as important a facet in practice. We see here that despite the asymptotic suboptimality of FedAvg with $\gamma > 0$, in realistic settings where the number of communication rounds is limited, there may be little to no disadvantage to using $\gamma > 0$. In fact, as we see for CIFAR-100, there may be advantages to using larger $\gamma$ in settings with relatively few communication rounds.

\subsection{Deriving fair comparisons between different client learning rates}

While the results in the previous section indicate the importance of $\gamma$, these comparisons are in some sense unfair. In particular, there is no reason that we need to fix the server learning rate $\eta$ across different values of $\gamma$. In fact, this ignores differences in the Lipschitz constant of the associated surrogate loss functions. As shown in Corollary \ref{cor:cond_fedavg}, choosing smaller $\gamma$ increases the Lipschitz constant $L_\gamma$ of the surrogate loss function, which in turn generally necessitates a smaller server learning rate. Thus, choosing larger $\gamma$ enable higher values of $\eta$. To justify this further, we plot the $\ell_2$ norm of the update applied to the server model at each round. That is, recall that in Algorithm \ref{alg:outerloop}, we update the model via
\[
x_{t+1} = x_t - \eta q_t.
\]
In Figure \ref{fig:constant_server_lr_update_norm}, we plot $\norm{q_t}$ for different choices of $\gamma$ in the CIFAR-100 task. Specifically, we plot the mean value of $\norm{q_t}$ for each consecutive 1000 rounds, as well as the standard deviation within those rounds.

\begin{figure}[ht]
\centering
    \begin{subfigure}{.6\linewidth}
    \centering
    \includegraphics[width=\linewidth]{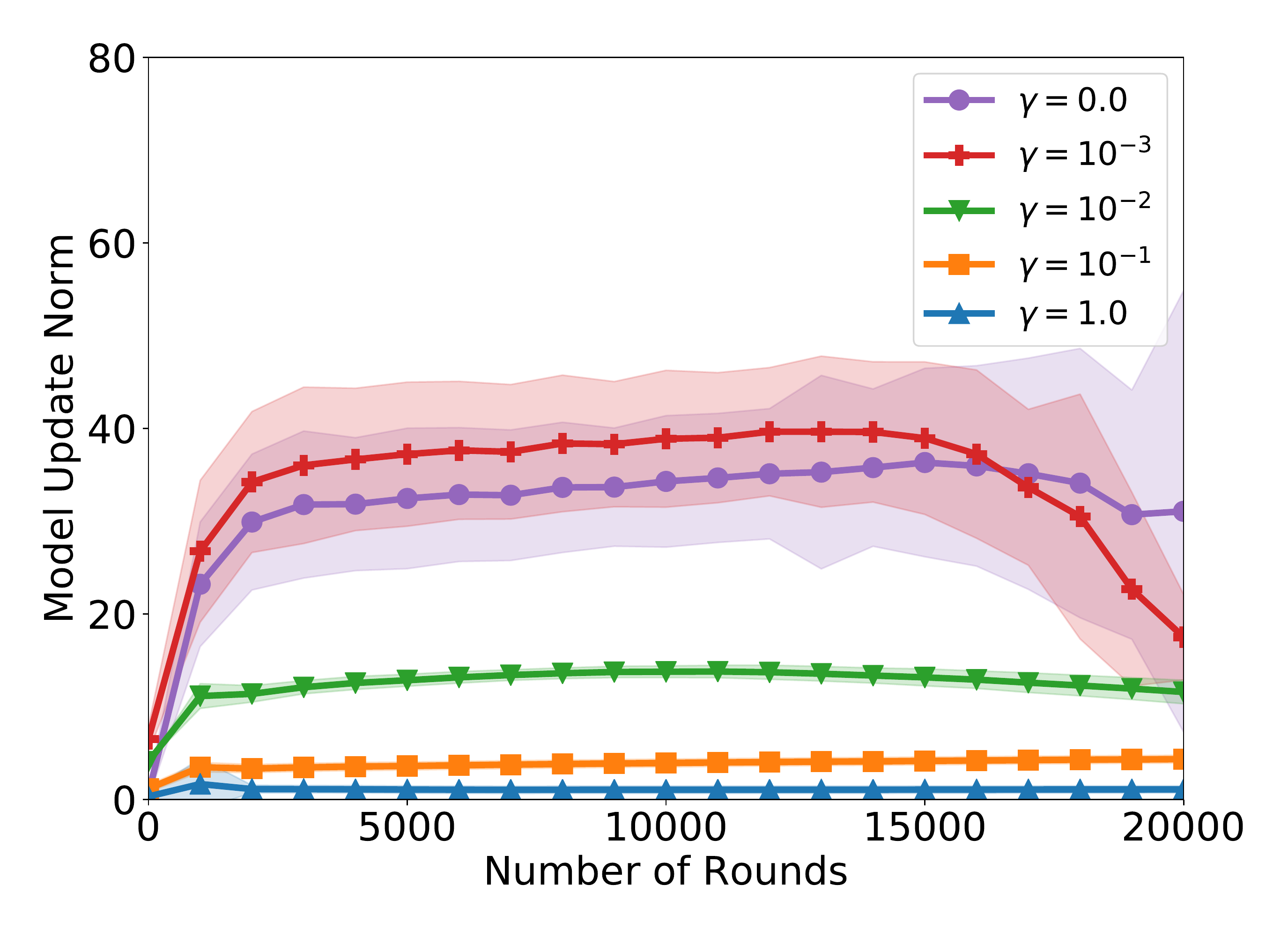}
    \end{subfigure}
\caption{The $\ell_2$ norm of the model update of \localupdate on CIFAR-100 at each round with fixed server learning rate $\eta = 0.01$ and varying client learning rate $\gamma$. We plot the mean $\ell_2$ norm across every consecutive 1000 rounds (bold lines), as well as the standard deviation across these rounds (pale regions).}
\label{fig:constant_server_lr_update_norm}
\end{figure}

As expected, the $\ell_2$ norm of the updates is highly dependent on the choice of $\gamma$. In particular, smaller learning rates lead to larger model updates. We also see that the absolute variance of the model update norm increases as well. As discussed after Corollary \ref{cor:fix_client_decay_server}, the variance reduction offered by setting $\gamma$ small SGD is potentially at odds with the increase in the condition number increase, a phenomenon reflected in Figure \ref{fig:constant_server_lr_update_norm}. Thus, we see that to derive a fair comparison between client learning rates, we must also allow the server learning rate to increase as the client learning rate increases.

One important caveat to this observation is that it is not enough to simply directly tie the client and server learning rate together. This is in fact what is done in the original incarnation of FedAvg in \citep{mcmahan17fedavg}. As discussed in Section \ref{sec:local_update_algs}, the original ``vanilla'' FedAvg algorithm corresponds to \localupdate with $\Theta = \Theta_{1:K}$ and $\gamma = \eta$. This does allow for the use of larger server learning rates with smaller client learning rates. However, this is not enough to necessarily derive optimal performance of \localupdate.

To demonstrate this, we plot the performance of \localupdate with $\eta = \gamma$ and $\Theta = \Theta_{1:10}$ for the CIFAR-100 task in Figure \ref{fig:cifar100_vanilla_fedavg}. We see results in stark contrast to Figure \ref{fig:constant_server_lr}(b). In particular, $\gamma = 1.0$ actually results in an increased training loss, and $\gamma = 0.1$ outperforms $\gamma = 0.01$ for most rounds. Moreover, by setting $\gamma = \eta$, we do not allow $\gamma = 0$, even though $\gamma > 0$ necessarily results in a surrogate loss that does not match the true training loss. While this version of FedAvg is convenient from an implementation and hyperparameter tuning point of view, it does not result in the best performance of \localupdate. As we show in the next section, we can greatly improve performance by tuning client and server learning rates separately.

\begin{figure}[ht]
\centering
    \begin{subfigure}{.6\linewidth}
    \centering
    \includegraphics[width=\linewidth]{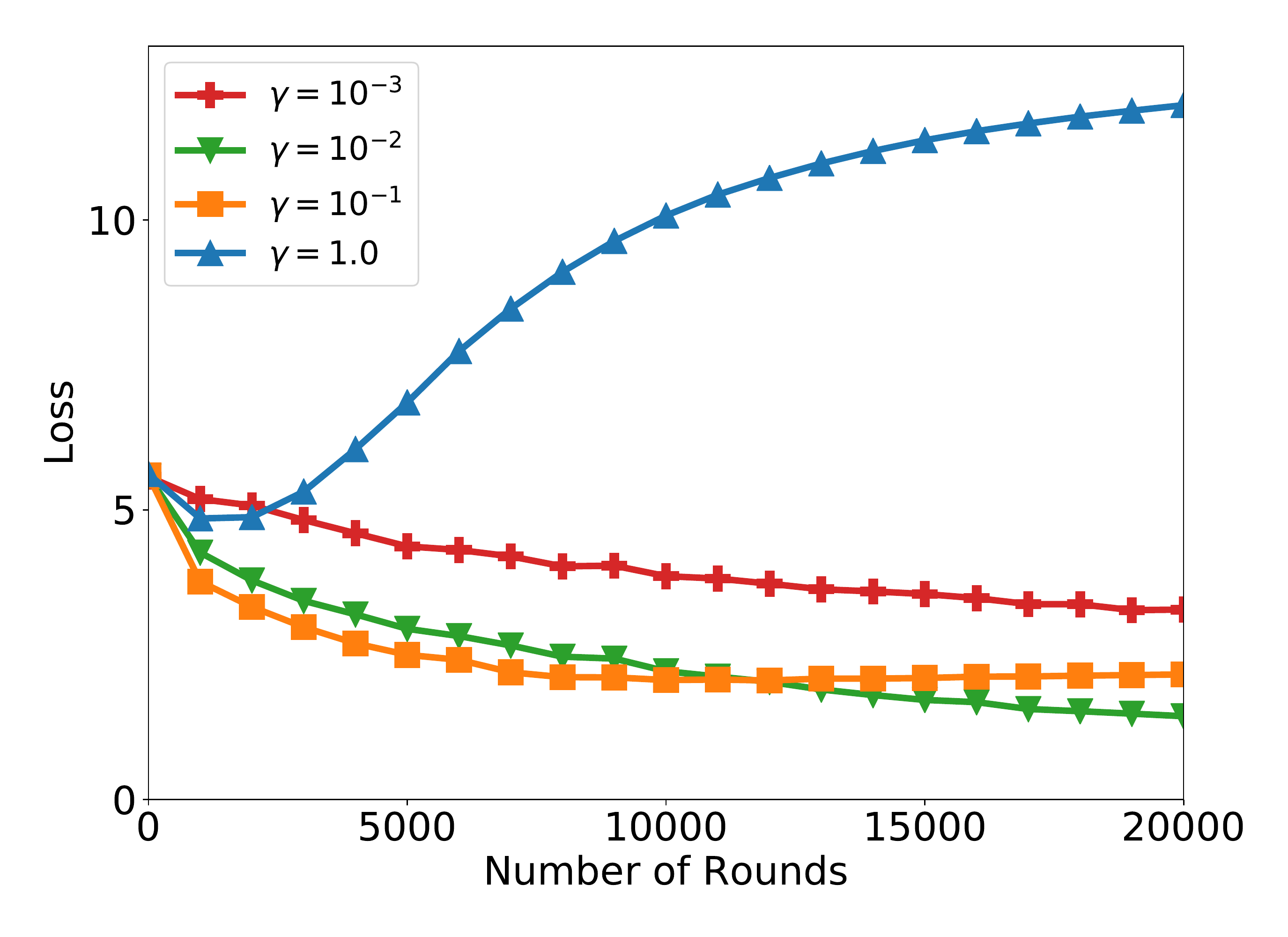}
    \end{subfigure}
\caption{Cross-entropy training loss of \localupdate with $\Theta = \Theta_{1:10}$ and $\gamma = \eta$ for varying values of $\gamma$ on the CIFAR-100 task.}
\label{fig:cifar100_vanilla_fedavg}
\end{figure}

\subsection{On the importance of tuned server learning rates}\label{sec:tuned_server}

Based on the discussion in the section above, to give the most fair comparisons between different values of $\gamma$ in \localupdate, we must also tune the server learning rate $\eta$. We perform the same experiments as in Figure \ref{fig:constant_server_lr}, but where we tune the server learning rate $\eta$ for each choice of $\gamma$. We vary $\eta$ over
\begin{equation}\label{eq:server_grid}
\eta \in \{10^{-3}, 10^{-2.5}, \dots, 10^{1}\}
\end{equation}
and select $\eta$ that results in the smallest average training loss over the last 100 rounds. We note that we use this averaging method as a single round of federated learning only samples a small number of clients. We plot the loss of \localupdate under these settings in Figures \ref{fig:tuned_server_lr}. For a list of all client learning rates and the corresponding best server learning rate for each task, see Appendix \ref{sec:tuned_lrs}.

Notably, we still see the same general trend discussed in our convergence rates in Section \ref{sec:convergence}. While large values of $\gamma$ may obtain a smaller training loss initially, eventually smaller values of $\gamma$ perform comparably, if not better. However, it is instructive to note that this may take many thousands of communication rounds. In particular, for all three tasks, we see that $\gamma = 0$ does not perform comparably larger $\gamma$ until near the end of our training procedure. We thus see the following:

\begin{center}
    \textbf{By performing appropriate server and client learning rate tuning, we can mitigate the suboptimality of FedAvg/Reptile for $\gamma > 0$, especially when the number of communication rounds is limited.}
\end{center}

In particular, our empirical results suggest that in realistic federated learning training tasks, the suboptimality of FedAvg/Reptile may only be an asymptotic concern. If we can only perform a limited number of training rounds, it often does benefit us to use larger values of $\gamma$. We note that this aligns with our discussion of how $\gamma > 0$ leads to condition number regularization (see Corollary \ref{cor:cond_fedavg}).

\begin{figure}[ht]
\centering
    \begin{subfigure}{.31\linewidth}
    \centering
    \includegraphics[width=\linewidth]{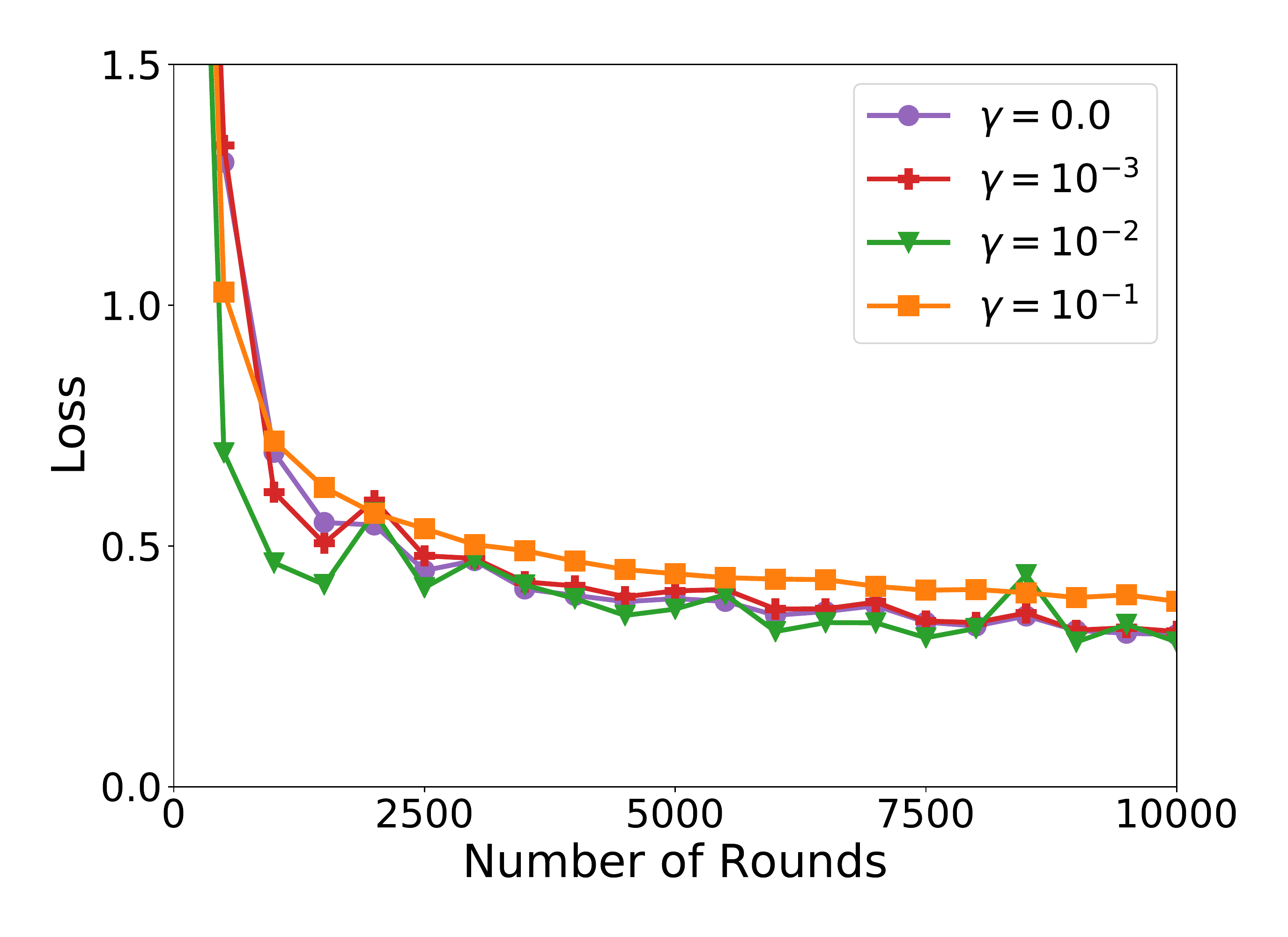}
    \caption{FEMNIST}
    \end{subfigure}
    \begin{subfigure}{.31\linewidth}
    \centering
    \includegraphics[width=\linewidth]{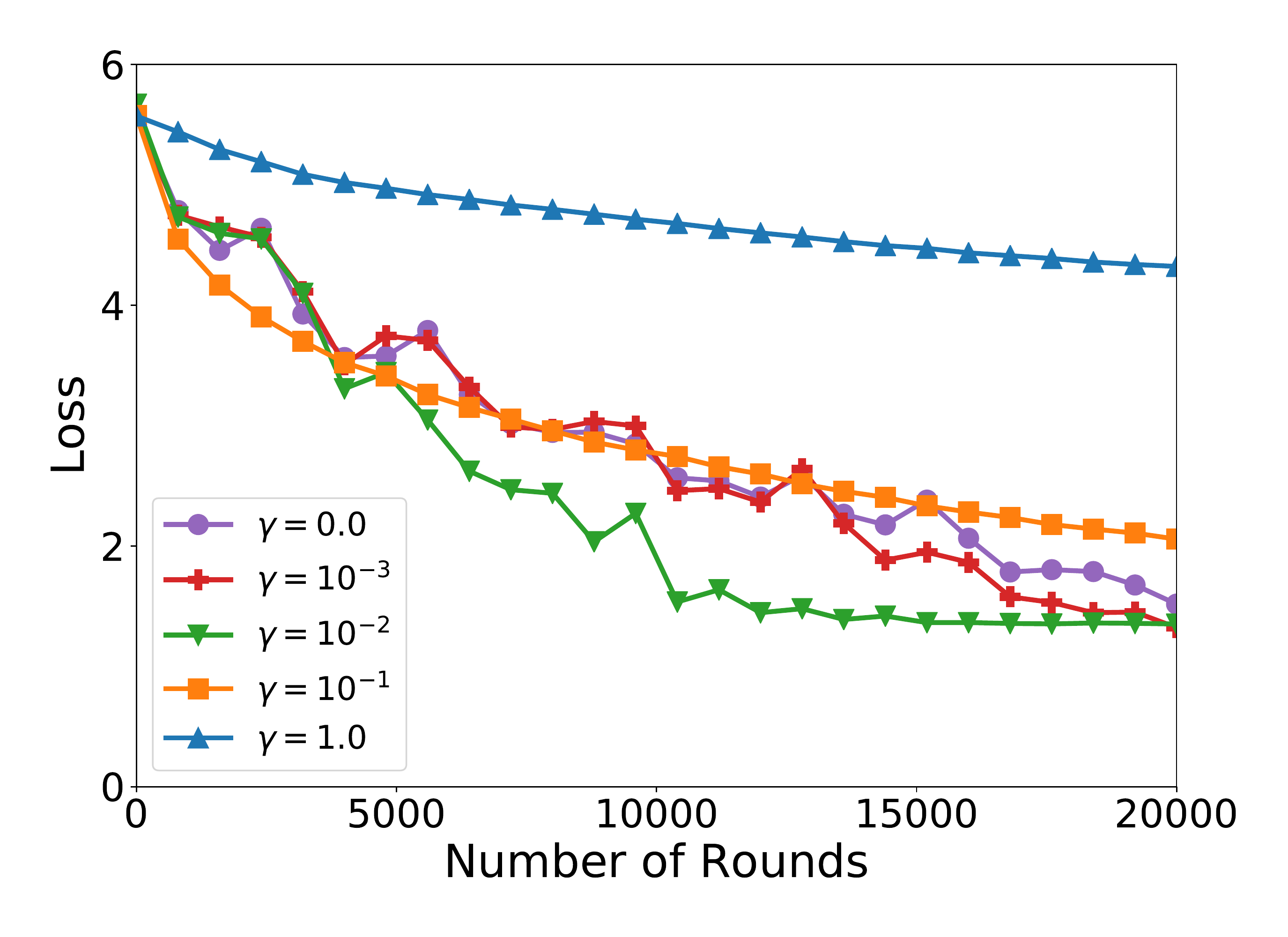}
    \caption{CIFAR-100}
    \end{subfigure}
    \begin{subfigure}{.31\linewidth}
    \centering
    \includegraphics[width=\linewidth]{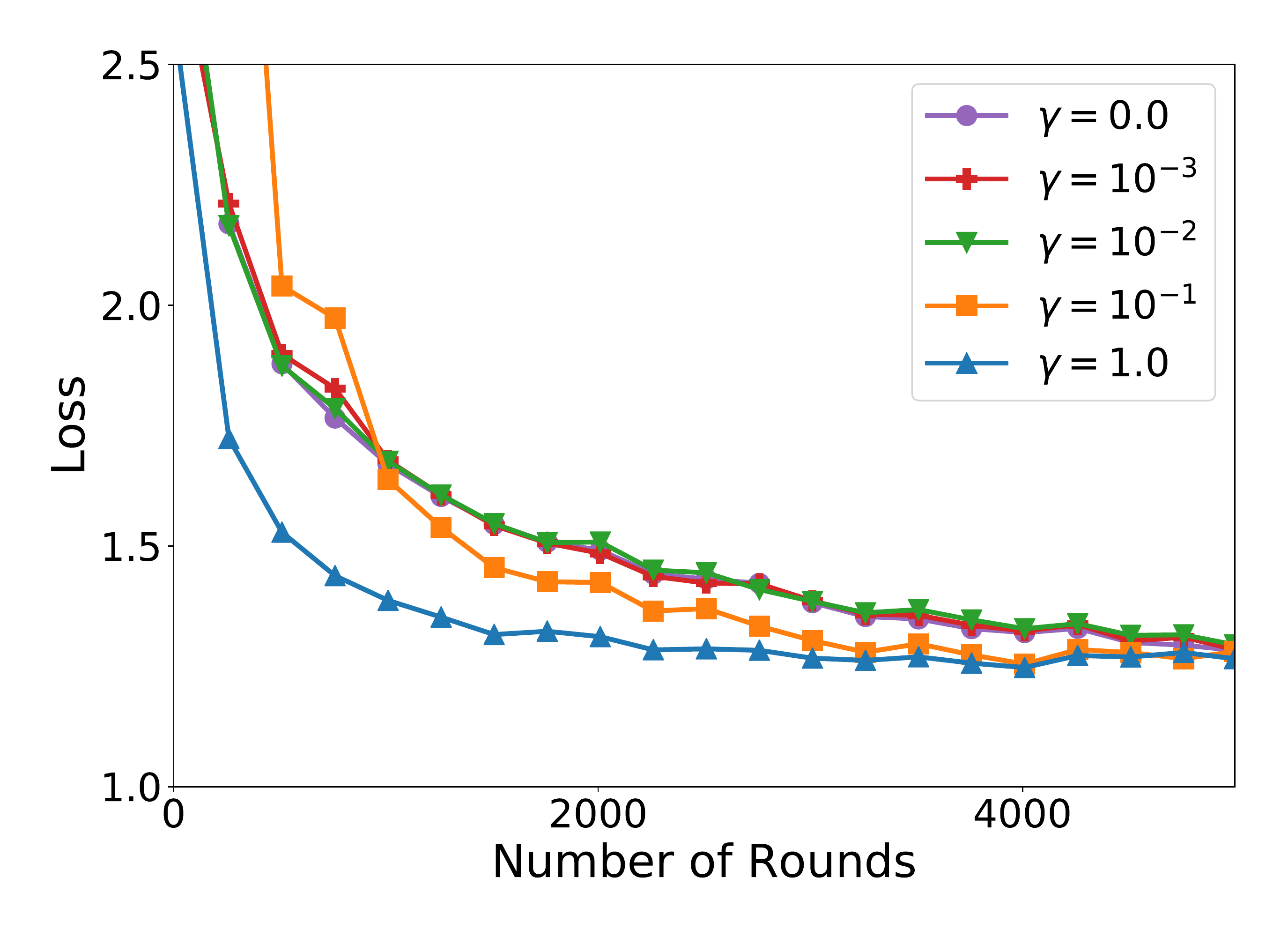}
    \caption{Shakespeare}
    \end{subfigure}
\caption{Cross-entropy training loss of \localupdate with $\Theta = \Theta_{1:10}$, varying client learning rate $\gamma$, and tuned server learning rate $\eta$.}
\label{fig:tuned_server_lr}
\end{figure}

\subsection{Adaptive optimization and \localupdate}\label{sec:adaptive_localupdate}

In order to understand the behavior of \localupdate more generally, we perform similar experiments on the Stack Overflow dataset. However, as shown by \citet{reddi2020adaptive}, performance of FedAvg on next-word-prediction tasks can be greatly improved by the use of adaptive optimization. In particular, \citet{reddi2020adaptive} found that the use of the \yogi optimizer~\citep{zaheer2018adaptive} improved performance in a wide variety of settings, including on the same tag-prediction and next-word-prediction tasks. Note that the \yogi optimizer is similar to the \adam optimizer~\citep{kingma2014adam}, except that it uses a kind of \emph{additive} adaptive update that can improve convergence by making more controlled progress. For more details, see \citep{zaheer2018adaptive}.

Thus, for these tasks, we use a modified version of \localupdate in which the server uses the client update $q_t$ as an estimate of the gradient of the loss function, and applies \yogi to this gradient. That is, we use Algorithm \ref{alg:outerloop}, but update the model via
\[
x_{t+1} = \yogi(x_t, q_t, \eta_t).
\]
We use the version of \yogi proposed by \citet{zaheer2018adaptive}, with first momentum term of $\beta_1 = 0.9$, second momentum parameter of $\beta_2 = 0.99$, an initial accumulator value of $0$, and an $\epsilon$ value of $10^{-5}$. We then perform analogous experiments to those above, where we first fix $\eta = 0.01$ and vary $\gamma$. However, due to the size of the Stack Oveflow dataset (see Table \ref{table:datasets}), we did not compute the total loss $f(x)$ over all clients. Instead, we plot the average loss of the clients that participated in a given round, \emph{before} local training occurs. This ``loss at current round'' can be viewed as a stochastic estimate of the true loss $f(x)$. This loss is plotted in Figure \ref{fig:stack_overflow_constant}. We also perform experiments where we vary $\gamma$ and tune $\eta$. We select the value of $\eta$ with the smallest average training loss over the last 100 communication rounds. The result is given in Figure \ref{fig:stack_overflow_tuned}. For a list of all client learning rates and the corresponding best server learning rate for each task, see Appendix \ref{sec:tuned_lrs}.

\begin{figure}[ht]
\centering
    \begin{subfigure}{.45\linewidth}
    \centering
    \includegraphics[width=\linewidth]{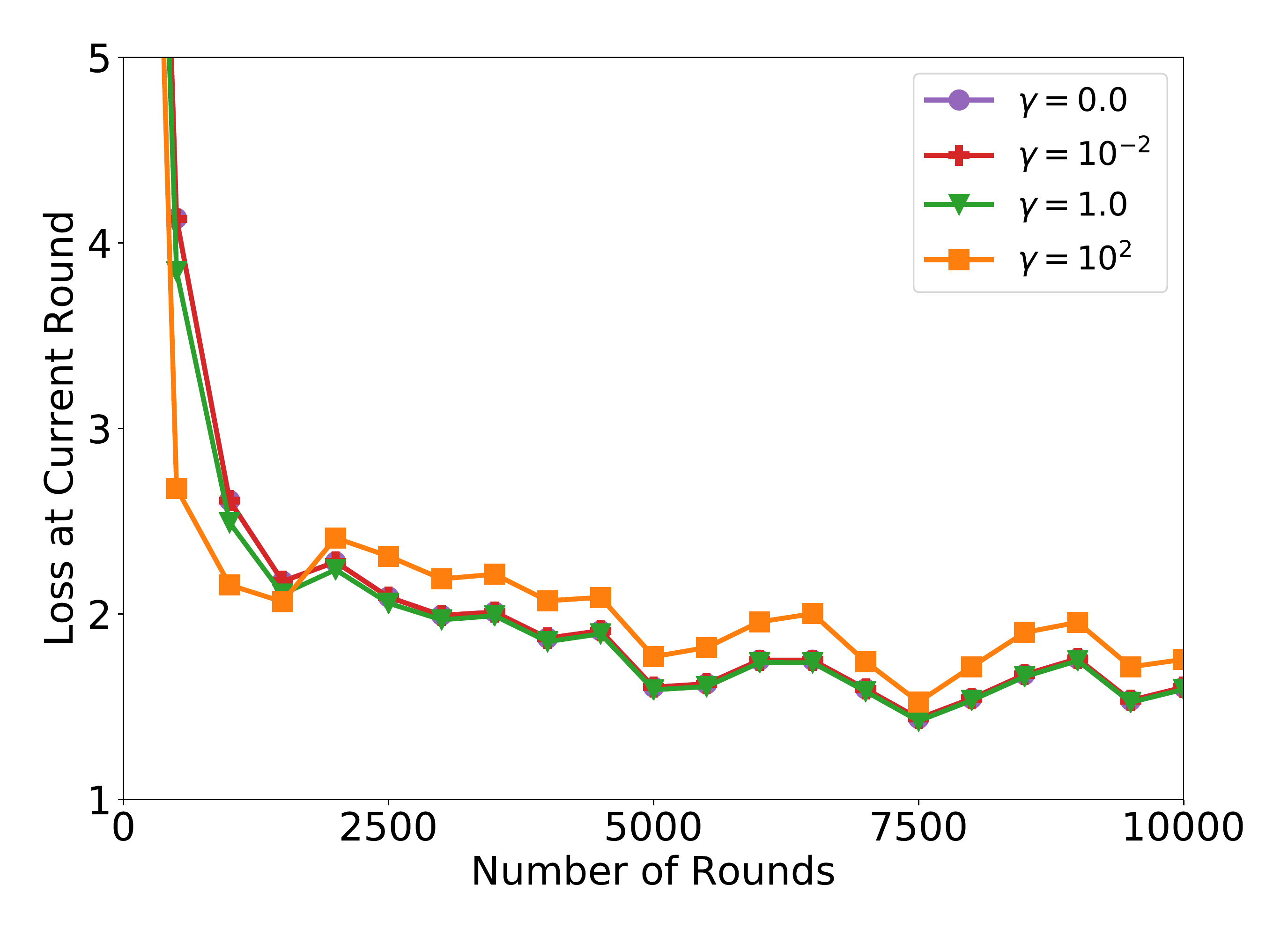}
    \caption{Stack Overflow - Tag Prediction}
    \end{subfigure}
    \begin{subfigure}{.45\linewidth}
    \centering
    \includegraphics[width=\linewidth]{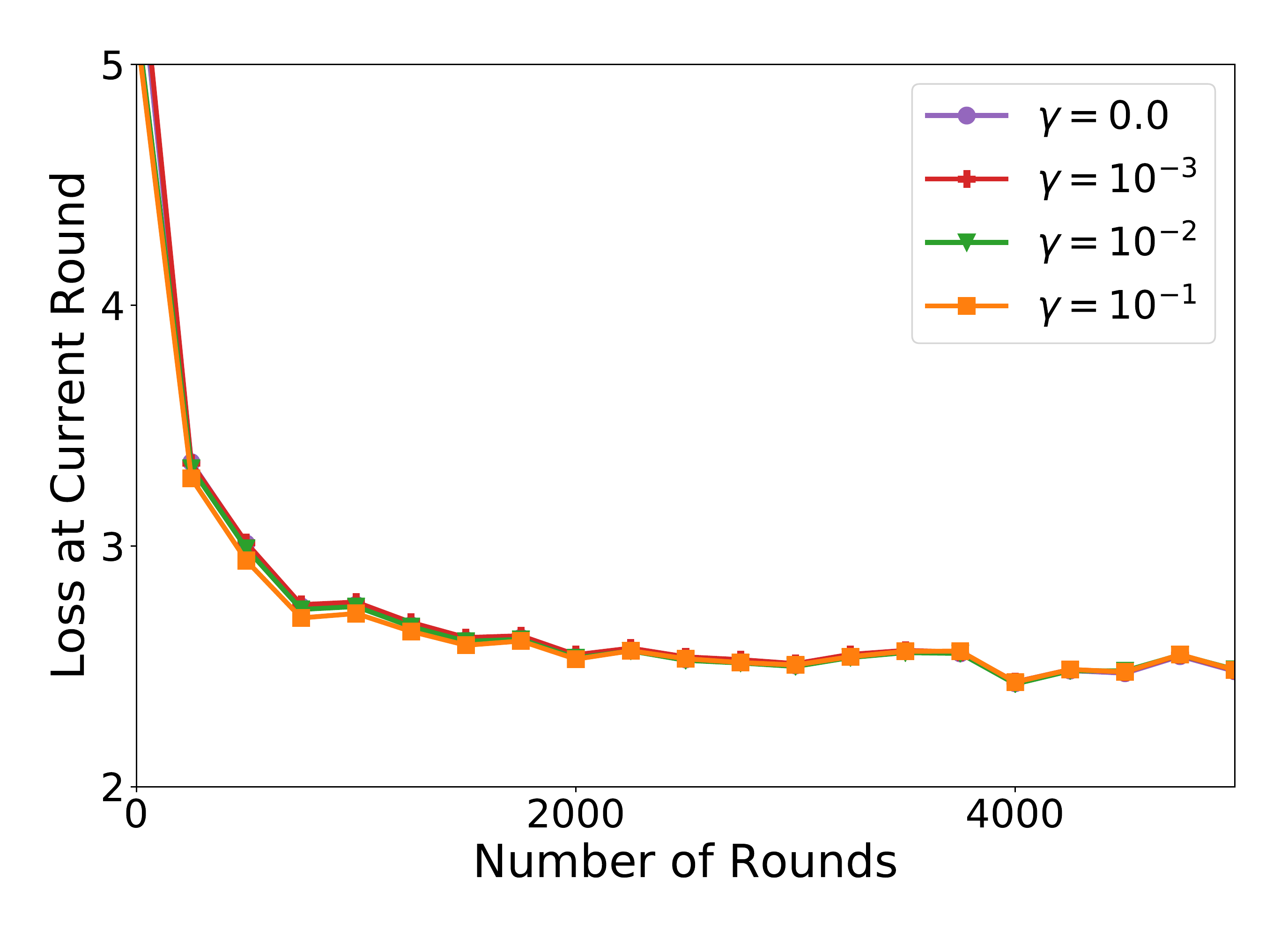}
    \caption{Stack Overflow - Next Word Prediction}
    \end{subfigure}
\caption{Cross-entropy training loss of \localupdate with server \yogi and $\Theta = \Theta_{1:10}$ on Stack Overflow with varying client learning rate $\gamma$ and fixed server learning rate $\eta = 0.01$.}
\label{fig:stack_overflow_constant}
\end{figure}

\begin{figure}[ht]
\centering
    \begin{subfigure}{.45\linewidth}
    \centering
    \includegraphics[width=\linewidth]{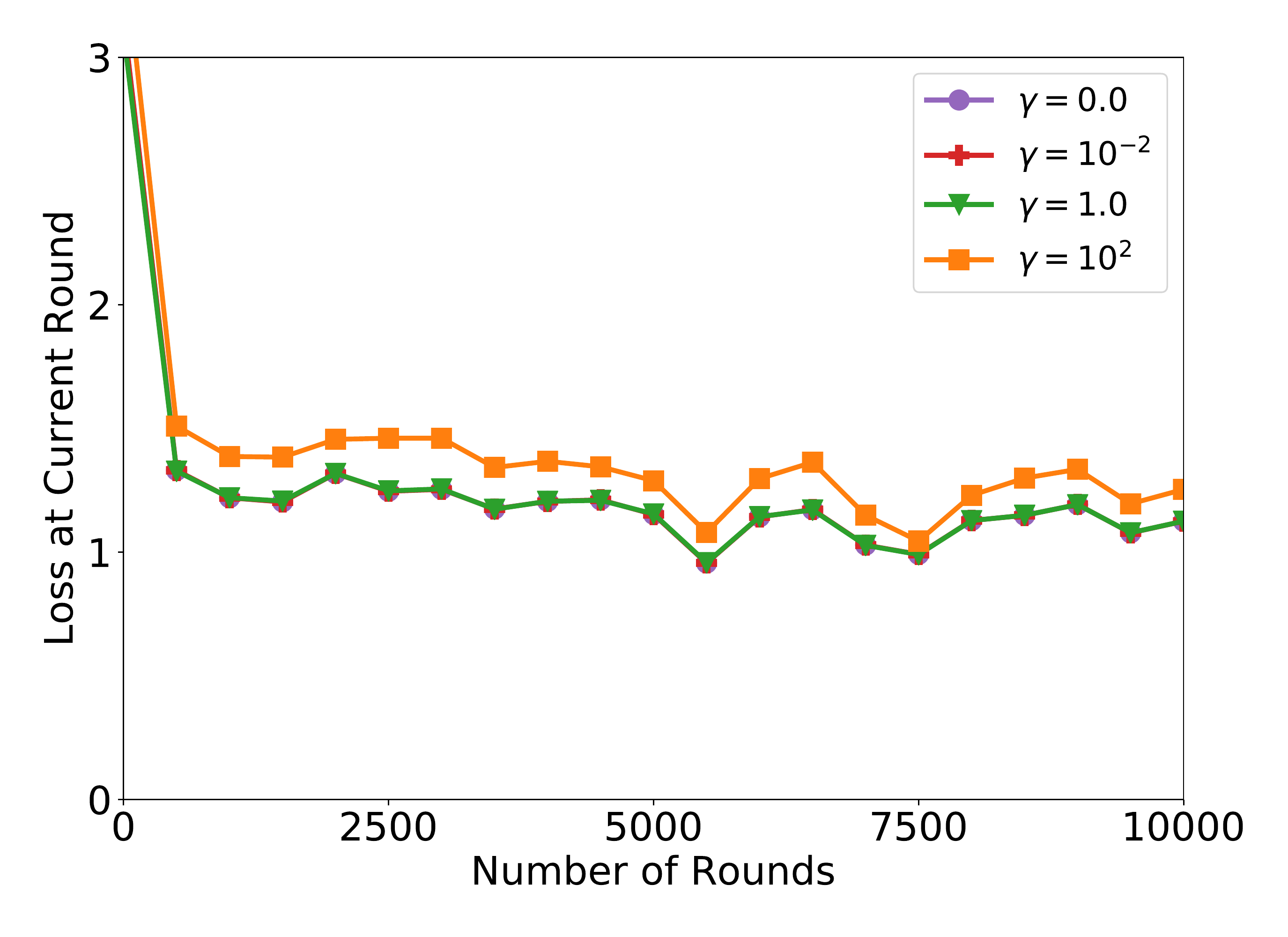}
    \caption{Stack Overflow - Tag Prediction}
    \end{subfigure}
    \begin{subfigure}{.45\linewidth}
    \centering
    \includegraphics[width=\linewidth]{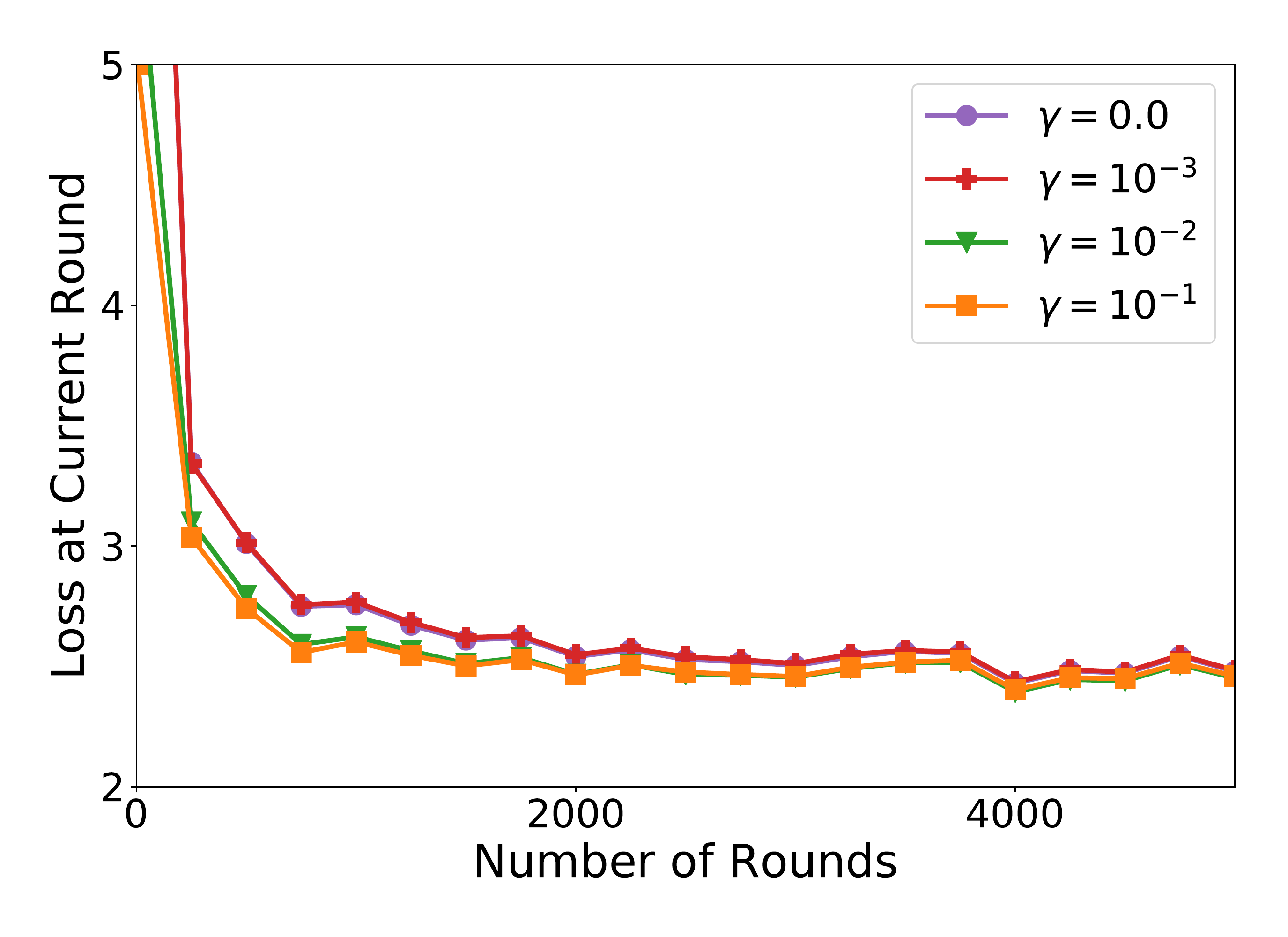}
    \caption{Stack Overflow - Next Word Prediction}
    \end{subfigure}
\caption{Cross-entropy training loss of \localupdate with server \yogi and $\Theta = \Theta_{1:10}$ on Stack Overflow with varying client learning rate $\gamma$ and tuned server learning rate $\eta$.}
\label{fig:stack_overflow_tuned}
\end{figure}

For fixed $\eta$, large $\gamma$ leads to an initially smaller loss that is eventually beaten or matched by smaller $\gamma$. However, we see similar behavior among all but the largest $\gamma$ in both tasks, potentially due to the adaptivity in \yogi. Looking at Figure \ref{fig:stack_overflow_tuned}, we see that tuning the server learning rate has slightly different effects on the two tasks: While Stack Overflow TP with tuned $\eta$ leads to the smaller $\gamma$ performing better throughout the training process, Stack Overflow NWP with tuned $\eta$ actually allows larger $\gamma$ to achieve lower loss throughout. Notably, we see that the gap between large $\gamma$ and small $\gamma$ in Stack Overflow NWP winnows as the number of rounds increases. Our results reinforce the notion that in communication-limited settings, server learning rate tuning is critical to ensure the best performance possible. This holds even when not using SGD on the server, but instead using an adaptive optimizer.

\section{Automatic learning rate decay}\label{sec:lr_decay}

In short, our results in the section above suggest that depending on the desired number of communication rounds, we may wish to use different client learning rates. Unfortunately, this requires a large degree of hyperparameter tuning. Both the client and server learning rate must be tuned. However, if the number of communication rounds in a federated learning system is limited, it may not be feasible to conduct extensive hyperparameter tuning, as the communication rounds required to do so may be better utilized by training your model for more rounds.

To help reduce the amount of learning rate tuning required, we propose a method for automatic learning rate decay that helps mitigate the need for \emph{client} learning rate tuning. Our method will utilize our theoretical and empirical observations above showing that large client learning rates should be used initially, while smaller learning rates eventually reach a lower training loss. By decaying the client learning rate automatically over time, we mitigate the need to tune it. Instead, it can be set to any moderate value that does not result in divergent behavior on clients. We will be particularly concerned with systems-level constraints (such as those encountered in federated learning) when describing our method.

One particularly important restriction in many local update settings is the ability to compute the training loss $f(x)$. Recall that we defined
\[
f(x) := \E_{i \sim \mP} [f_i(x)],~~~f_i(x) := \E_{z \sim \mD_i} [f(x; z)]
\]
where $\mP$ is a distribution over clients, and $\mD_i$ is the client's data distribution. In settings with limited communication, it may not be feasible to sample most or even a moderate fraction of the clients from $\mP$. Moreover, sampling a client $i \sim \mP$ purely for the purposes of estimating the loss may be a waste of resources, as that client could be used for training purposes.

Thus, we propose a version of \localupdate that simultaneously computes local updates and estimates the loss function. Specifically, at each round $t$, we sample a set $I_t$ of clients of size $M$. For each $i \in I_t$, we first compute an estimate $\ell_t^i$ of the local loss function $f_i(x_t)$, as well as the client's model update $q_t^i$ as defined in Algorithm \ref{alg:innerloop}. Note that it is important that $f_i(x_t)$ is estimated \emph{before} computing any local updates, as otherwise we risk overfitting to the client's dataset. The server then computes
\begin{equation}\label{eq:f_t_eq}
\ell_t := \frac{1}{M} \sum_{i \in I_t} \ell_t^i,~~~q_t := \frac{1}{M}\sum_{i \in I_t} q_t^i.
\end{equation}

As in \localupdate, we will use this $q_t$ to update the model. The loss estimate $\ell_t$ will be used to determine whether to decay the client or server learning rates. We will use the strategy of decaying learning rates on plateaus. If the loss $f_t$ is not sufficiently small, we will decay the client and server learning rates. We record the minimum loss seen up to round $t$, and then check if
\[
\ell_t < \min_{j < t} \ell_j - \Delta
\]
for some small $\Delta > 0$. Intuitively, if the loss $\ell_t$ has not decreased sufficiently, we may be at a suboptimal point (due to a large client learning rate $\gamma$) and therefore should decrease $\gamma$, which necessitates a decrease in $\eta$. We propose decaying $\gamma$ and $\eta$ by fixed factors $\alpha, \beta \in (0, 1)$, as this kind of ``staircase'' learning rate schedule has been repeatedly demonstrated to perform well across many tasks~\citep{krizhevsky2014one, goyal2017accurate}. A full version of this algorithm, which we refer to as \localupdatedecay, is given in Algorithms \ref{alg:outerloop_decay} and \ref{alg:innerloop_decay}.

\begin{figure}[ht]

\begin{minipage}[t]{.52\textwidth}
\begin{algorithm}[H]
\caption{\localupdatedecay: \\ Outer Loop}
\label{alg:outerloop_decay}
\begin{algorithmic}
\SUB{OuterLoop$(x, \eta, \gamma, \Theta, \Delta > 0, \alpha, \beta \in (0, 1))$:}
\STATE $x_1 = x$
\FOR{each round $t = 1, 2, \dots$, T}
    \STATE $I_t \leftarrow$ (random set of $M$ clients)
    \FOR{each client $i \in I_t$ \textbf{in parallel}}
        \STATE $\ell_t^i, q_t^i \leftarrow \text{InnerLoop}(i, x_t, \gamma, \Theta)$
    \ENDFOR
    \STATE $\ell_t \leftarrow (\nicefrac{1}{M})\sum_{i \in I_t} \ell_t^i$
    \STATE $q_t \leftarrow (\nicefrac{1}{M})\sum_{i \in I_t} q_{t}^i$
    \STATE $x_{t+1} = x_t-\eta_tq_t$
    \IF{$\ell_t > \min_{j < t} \ell_j - \Delta$}
        \STATE $\gamma \leftarrow \alpha\gamma,~~\eta \leftarrow \beta\eta$
    \ENDIF
\ENDFOR
\STATE return $x_{T+1}$
\end{algorithmic}
\end{algorithm}
\end{minipage}
\begin{minipage}[t]{.46\textwidth}
\begin{algorithm}[H]
\caption{\localupdatedecay: \\ Inner loop}
\label{alg:innerloop_decay}
\begin{algorithmic}
\renewcommand{\arraystretch}{1.6}
\SUB{$\InnerLoop(i, x, \gamma, \Theta$):}
\STATE $x_1 = x$
\STATE $\ell \leftarrow \E_{z\sim\mD_i} [f(x; z)]$
\FOR{$k = 1, 2, \dots, K(\Theta)$}
    \STATE sample a set $S_k$ of size $B$ from $\mD_i$
    \STATE $g_k = (\nicefrac{1}{B}) \sum_{z \in S_k} \nabla f(x_k ; z)$
    \STATE $x_{k+1} \leftarrow x_k - \gamma g_k$ 
\ENDFOR
\STATE return $\ell,~\sum_{k = 1}^{K(\Theta)} \theta_k g_k$
\end{algorithmic}
\end{algorithm}
\end{minipage}

\end{figure}

\subsection{Practical refinements to \localupdatedecay}\label{sec:practical_refinements}

In this section we describe three practical refinements to \localupdatedecay that can improve convergence behavior. First, in settings where the number of clients sampled per round $M$ is sufficiently small compared to $|\mI|$, we can instead estimate the loss by using a moving window average across rounds. That is, given some window size $W > 0$, we can compute an estimate of the loss $\hat{\ell}_t$ with reduced variance via
\[
\hat{\ell}_t = \dfrac{1}{W} \sum_{j = 0}^{W-1} \ell_{t-j}
\]
where $f_t$ is as in \eqref{eq:f_t_eq}. This is the average loss of all clients in the last $W$ rounds. Note that Algorithm \ref{alg:outerloop_decay} corresponds to $W = 1$. We would then decay the client and server learning rates if
\begin{equation}\label{eq:fhat_condition}
\hat{\ell}_t > \min_{j < t} \hat{\ell}_j - \Delta.
\end{equation}

Second, even using moving windows to estimate $f(x)$, the heterogeneity of clients can still cause problematic variance. Thus, one can instead decay the learning rates if there has been no progress (relative to $\Delta$) for $P$ \emph{consecutive} rounds. That is, we keep a counter for how many consecutive rounds the condition \eqref{eq:fhat_condition} holds. If this counter ever reaches $P$, we then decay $\gamma$ and $\eta$. If
\[
\hat{\ell}_t \leq \min_{j < t} \hat{\ell}_j - \Delta
\]
then we reset the counter to 0. Note that Algorithm \ref{alg:outerloop_decay} corresponds to $P = 1$.

Last, it is often useful to have a \emph{cooldown period} $C$, where after decaying the learning rates $\gamma$ and $\eta$, we do not decay the learning rate for the next $C$ rounds. This is beneficial both for recovering a new estimate of the loss function, and for ensuring that the learning rate does not decay too frequently. Additionally, we recommend using this cooldown period for the first $C$ rounds as well, as this allows one to develop a better estimate of $f(x)$ before decaying the learning rate.

While in practice, setting $W, P$, and $C$ may seem difficult, we found that setting them all to be the same value led to good behavior across datasets and tasks. Moreover, the value can be estimated using simple heuristics. For example, suppose that we think that $N$ randomly sampled clients are sufficient to give a good representation of $\mP$. Then, $P = W = C = \lceil N/M \rceil$ should serve as a default value for these parameters. In practice, we found that fixing all three values to some moderate constant (ex. $P = W = C = 100$) was sufficient, even across datasets with widely varying numbers of clients.

\subsection{Empirical evaluation}

The primary motivation for \localupdatedecay is removing the need for client learning rate tuning. Intuitively, we can use any moderately large value of $\gamma$ that results in non-divergent client behavior, and this will be gradually scaled back over time as we reach suboptimal critical points of the corresponding surrogate loss. To validate this, we compare \localupdatedecay to \localupdate on FEMNIST, CIFAR-100, Shakespeare, and Stack Overflow. For Stack Overflow, we again use a modified version where we apply \yogi on the server. In particular, we compare tuned but constant $\gamma$ and $\eta$, to \localupdatedecay.

\paragraph{Implementation and hyperparameters}

We implement \localupdatedecay in TensorFlow Federated as well. When computing the client's loss estimate $\ell$, we compute the average loss over the entire client dataset. Unlike in the previous section, we do not tune the client learning rate $\gamma$. We instead vary $\gamma$ over $\{10^2, 10, 1, 10^{-1}, 10^{-2}\}$ and select the largest $\gamma$ that results in a non-divergent training loss, as we intend to test the hypothesis that \localupdatedecay does not need $\gamma$ to be tuned in the same way that \localupdate does. In all settings, we still tune the server learning rate, as this is can be vital for getting accurate and fair comparisons (as discussed in Section \ref{sec:tuned_server}). We use the same server learning rate grid as with \localupdate, given in \eqref{eq:server_grid}, and select the value that reaches the lowest average training loss over the last 100 communication rounds. A table of learning rates used by \localupdatedecay for each task is given in Table \ref{table:decay_eta}.

\begin{table}[ht]
\setlength{\tabcolsep}{3.5pt}
\caption{Client learning rate $\gamma$ and server learning rate $\eta$ used in \localupdatedecay, in base-10 logarithm format.}
\label{table:decay_eta}
\begin{center}
\begin{sc}
\begin{tabular}[t]{l*{2}{c}}
\toprule
Task & $\gamma$ & $\eta$\\
\midrule
CIFAR-100 & 0 & -1 \\
FEMNIST & -1 & -1 \\
Shakespeare & 0 & \nicefrac{3}{2}\\
StackOverflow NWP & -1 & -\nicefrac{3}{2} \\
StackOverflow TP & 2 & 0 \\
\bottomrule
\end{tabular}
\end{sc}
\end{center}
\end{table}

We tune no other hyperparameters. In \localupdatedecay, specifically Algorithm \ref{alg:outerloop_decay}, we set $\Delta = 10^{-4}, \alpha = 0.1, \beta = 0.9$. We also use the practical refinements discussed in Section \ref{sec:practical_refinements}, setting $P = W = C = 100$. All other parameters are identical to that of standard \localupdate.

\begin{figure}[ht]
\centering
    \begin{subfigure}{.31\linewidth}
    \centering
    \includegraphics[width=\linewidth]{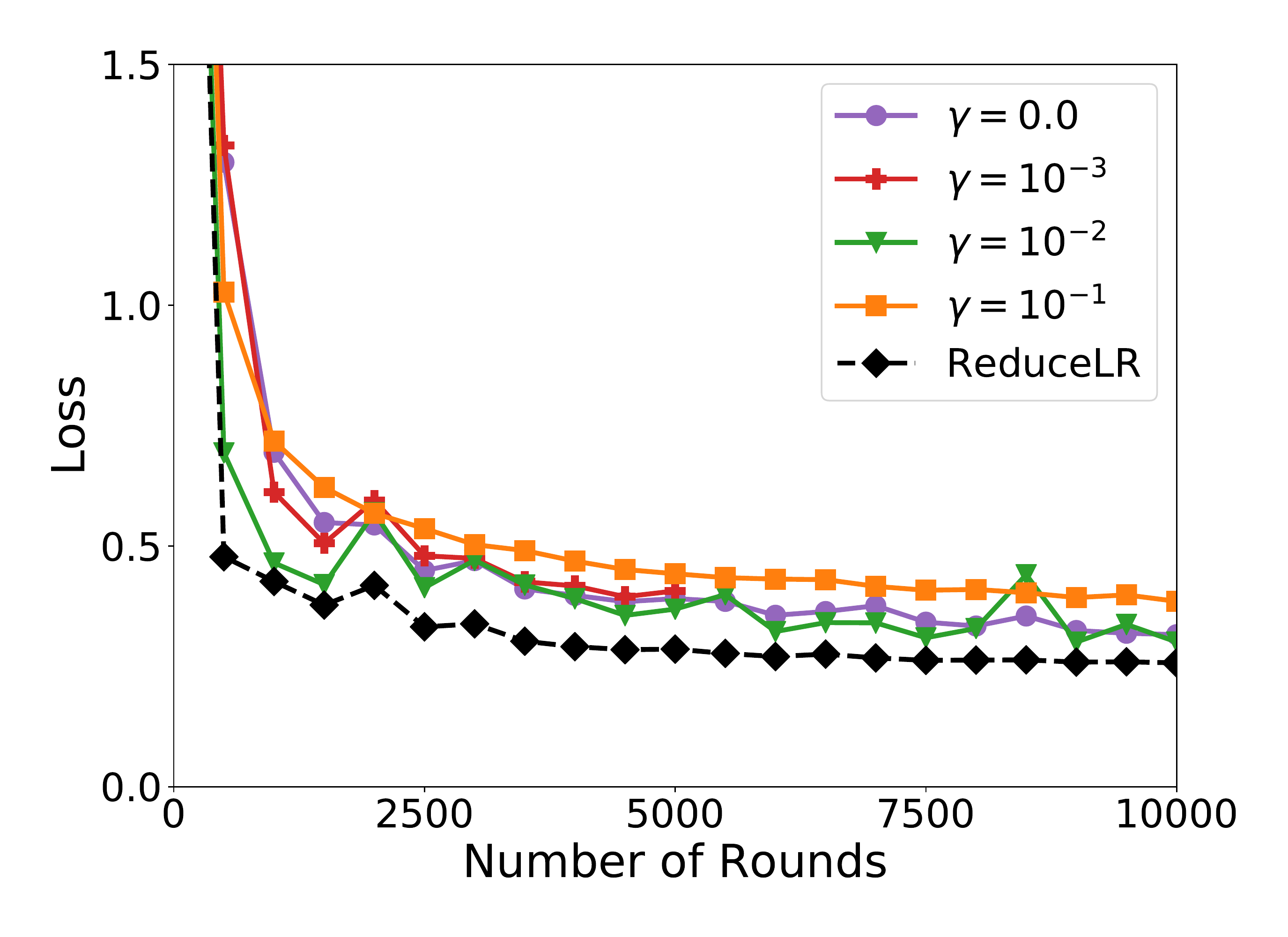}
    \caption{FEMNIST}
    \end{subfigure}
    \begin{subfigure}{.31\linewidth}
    \centering
    \includegraphics[width=\linewidth]{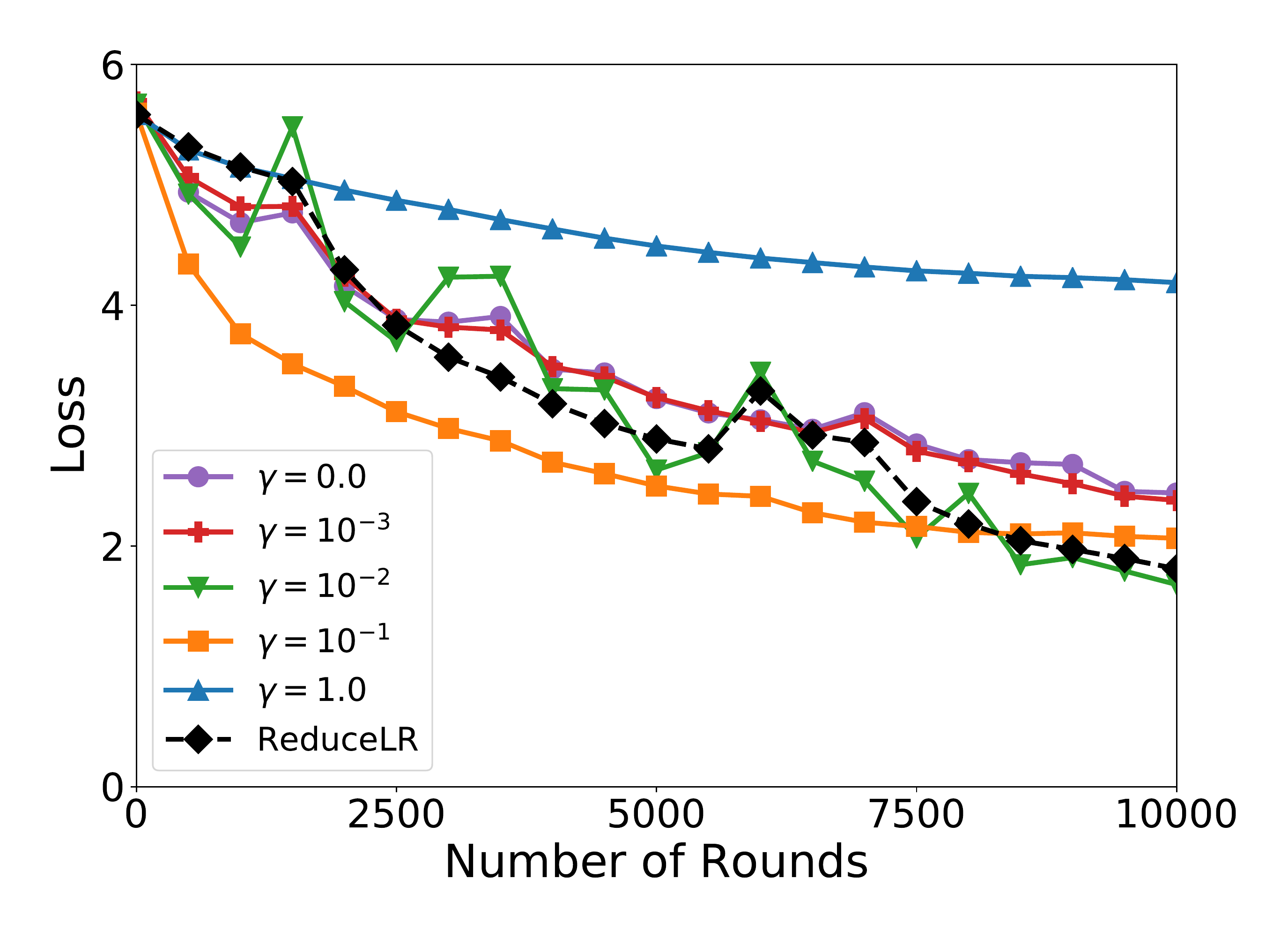}
    \caption{CIFAR-100}
    \end{subfigure}
    \begin{subfigure}{.31\linewidth}
    \centering
    \includegraphics[width=\linewidth]{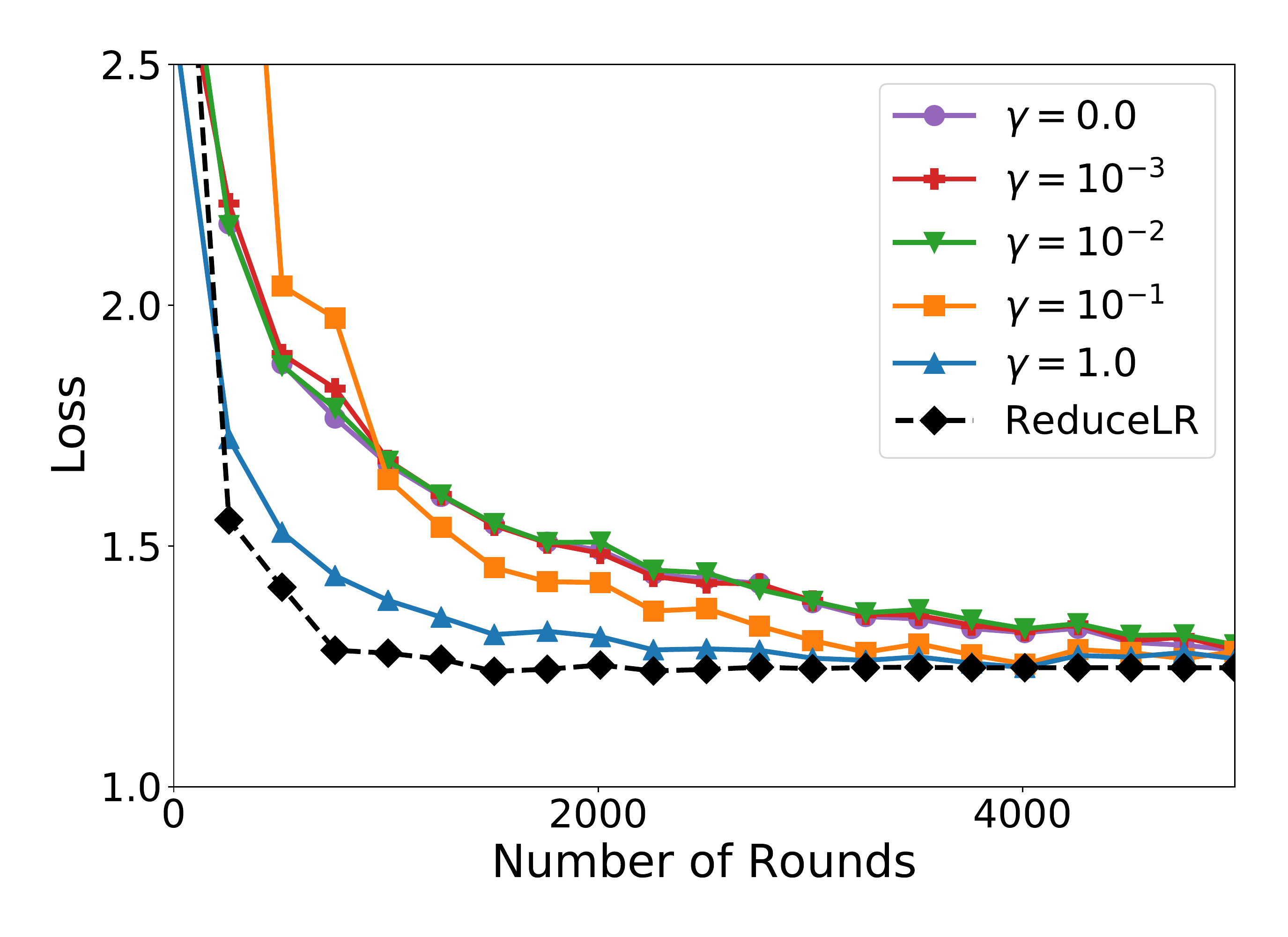}
    \caption{Shakespeare}
    \end{subfigure}
\caption{Cross-entropy training loss of \localupdate with $\Theta = \Theta_{1:10}$, varying client learning rate $\gamma$, and tuned server learning rate $\eta$. We also plot the cross-entropy loss of \localupdatedecay with tuned server learning rate (ReduceLR).}
\label{fig:compare_decay}
\end{figure}

\begin{figure}[ht]
\centering
    \begin{subfigure}{.31\linewidth}
    \centering
    \includegraphics[width=\linewidth]{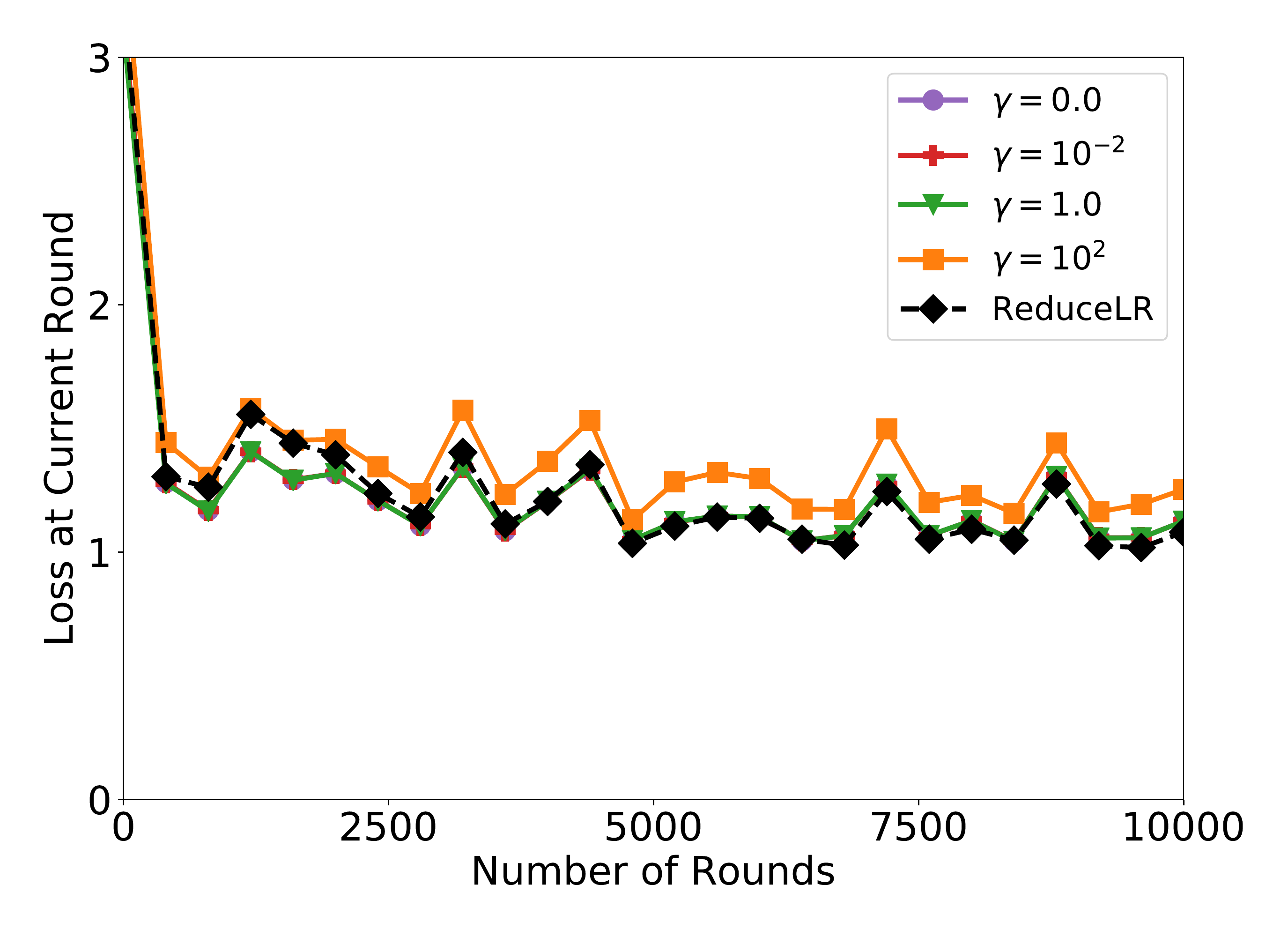}
    \caption{Stack Overflow - Tag Prediction}
    \end{subfigure}
    \begin{subfigure}{.31\linewidth}
    \centering
    \includegraphics[width=\linewidth]{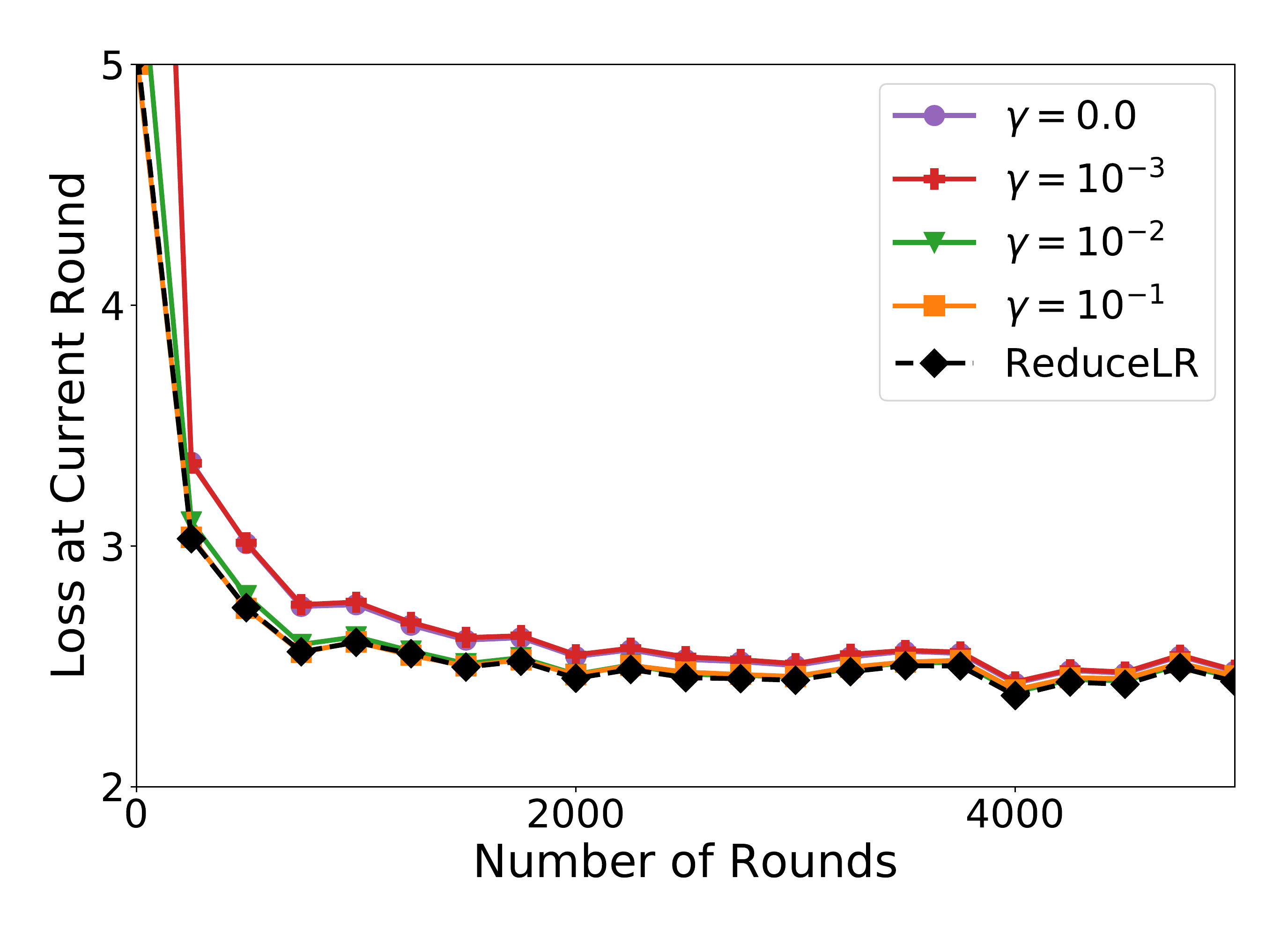}
    \caption{Stack Overflow - Next Word Prediction}
    \end{subfigure}
\caption{Cross-entropy training loss of \localupdate with server \yogi and $\Theta = \Theta_{1:10}$. We vary the client learning rate $\gamma$ and tune the server learning rate $\eta$. We also plot the cross-entropy loss of \localupdatedecay with server \yogi and tuned server learning rate (ReduceLR).}
\label{fig:stackoverflow_decay}
\end{figure}

\paragraph{Results}For FEMNIST, CIFAR-100, and Shakespeare, we plot the results in Figure \ref{fig:compare_decay}, where we also compare to the tuned results in Section \ref{sec:tuned_server}. We find that in all three tasks, \localupdatedecay eventually does as well as \localupdate with tuned $\gamma$, without the need for client learning rate tuning. For FEMNIST and Shakespeare, we find that \localupdatedecay almost immediately does better than \localupdate and continues to do at least as well throughout the course of training, often better. While this is not true for CIFAR-100, the results are still instructive. We see that while the client learning rate is initially set to a suboptimal value ($\gamma = 1.0$), the automatic learning rate decay enables us to move away from this suboptimal basin and towards something comparable to the best tuned client learning rate after enough rounds. We see that despite initializing with a bad $\gamma$, the non-divergence in earlier rounds is sufficient to allow eventually near-optimal performance in the later rounds.

For both Stack Overflow tasks we plot an analogous results in Figure \ref{fig:stackoverflow_decay}. As discussed in Section \ref{sec:adaptive_localupdate}, due to the size of the Stack Overflow dataset, we do not plot the loss $f(x)$ over all clients. Instead, we plot the average loss among all clients in each round before training. For Stack Overflow NWP, we see that \localupdatedecay performs comparably to the best tuned client learning rate. However, for Stack Overflow TP, we see that the decay actually helps significantly. While we initialize with a suboptimal client learning rate $\gamma = 100$, (which clearly achieves higher loss throughout), by decaying the client learning rate we are able to obtain comparable or lower loss than all other fixed $\gamma$.

\section{Open questions}

Our work above opens up a number of possible follow-up directions in the area of federated optimization. First, we expect the same kind of analysis obtained above to apply to methods similar to \localupdate. For example, while the FedProx algorithm \citep{li2019federated} does not fit the format of \localupdate as presented in this work, we believe that it can be analyzed in a similar way. As shown by \citet{pathak2020fedsplit}, even in the non-stochastic setting, FedProx is not optimizing the true loss function. Thus, a natural question is to understand exactly what loss function is being optimized, and how the structure of FedProx encourages convergence over FedAvg. Another natural algorithm for analysis is a more general version of \localupdate in which the client learning rate is decayed during a client's local computation. This may help combat adverse effects incurred by setting $K$ to be too large.

More generally, we would also like to understand the behavior of \localupdate on non-quadratic functions. Even generalizing the analysis above to the strongly convex case would be substantial progress towards understanding federated learning in heterogeneous settings. While there may be no surrogate loss that \localupdate is directly optimizing through SGD, we believe that the intuition behind our work can still be carried forward for more general loss functions. In particular, one might expect that the optimization dynamics of \localupdate can be parameterized in terms of $\gamma$ and $\Theta$, even in non-quadratic settings, and that the selection of these parameters governs a trade-off between the speed of convergence and the accuracy of the resulting critical point.

Another interesting open direction is determining how to best set $\Theta$ for a given problem. As seen above, the choice of $\Theta$ drastically alters optimization dynamics. While it is often chosen in an ad hoc manner (based in part on the cost of communication), one could imagine attempting to minimizing the number of rounds need to obtain a given accuracy level with respect to $\Theta$. Even for FedAvg/Reptile, it is not clear how to set the number of local steps $K$. Insights into this could greatly improve the performance of federated learning algorithms.

Finally, we note that our work is fundamentally concerned with the training dynamics of local update methods. In practice, we are often instead interested in the generalization ability of a model. We suspect that the choice of parameters in \localupdate can have large implications for generalization ability, the study of which we leave to future work.


\acks{We would like to thank Keith Rush and Sai Praneeth Karimireddy for the remarks that accidentally helped spark this work. We would also like to thank H. Brendan McMahan and Zachary Garrett for fruitful discussions about decoupling client and server learning rates in federated learning. Finally, we gratefully acknowledge Shanshan Wu for insights on personalization in federated learning.}

\appendix

\section{Relation between FedAvg, Local SGD, and \localupdate}\label{sec:special_cases}

In this section, we formalize the connection between FedAvg, Local SGD, and \localupdate. First, we note that in common descriptions of FedAvg and Local SGD algorithms (see \citep{mcmahan17fedavg} and \citep{stich2018local} for example), these two are effectively the same algorithm. The difference in nomenclature often reflects the distributed setting: Local SGD is often referred to in settings with homogeneous data, while FedAvg is often used in heterogeneous settings, especially for the purposes of federated learning.

We will use the following (simplified) version of the algorithms: At each iteration $t$ of FedAvg/Local SGD, we have some set of clients $I_t$ of size $M$. Each client $i \in I_t$ receives the server's model $x_t$, and applies $K$ steps of mini-batch SGD updates to its local model to create an updated local model $x_t^i$. The server then updates its model via
\[
x_{t+1} = \dfrac{1}{M}\sum_{i \in I_t} x_t^i.
\]
Fix $t$, and let $g_k^i$ denote the $k$-th mini-batch gradient of client $i$. Suppose we use a learning rate of $\gamma$ on each client when performing mini-batch SGD. Then we have
\begin{align*}
    x_{t+1} &= \dfrac{1}{M}\sum_{i \in I_t} x_t^i\\
    &= x_t - \dfrac{1}{M}\sum_{i \in I_t}\left(x_t - x_t^i\right)\\
    &= x_t - \dfrac{1}{M}\sum_{i \in I_t}x_t - \left(x_t - \gamma \sum_{k = 1}^K g_k^i\right)\\
    &= x_t - \gamma\dfrac{1}{M}\sum_{i \in I_t} \sum_{k=1}^K g_k^i.
\end{align*}
This is exactly \localupdate with $\Theta = \Theta_{1:K}$ and $\eta = \gamma$. However, by allowing $\eta$ to vary independently of $\gamma$ in \localupdate, we can obtain markedly different convergence behavior. We note that a form of this decoupling has previously been explored by \citet{karimireddy2019scaffold} and \citet{reddi2020adaptive}. However, these versions instead perform averaging on the so-called ``model delta'' (see \citep{reddi2020adaptive}), in which the server model is updated via
\begin{align*}
x_{t+1} &= x_t - \dfrac{\eta}{M}\sum_{i \in I_t} (x_t - x_t^i)\\
&= x_t - \dfrac{\eta}{M}\sum_{i \in I_t}\left(x_t - \left(x_t - \gamma \sum_{k=1}^K g_k^i\right)\right)\\
&= x_t - \dfrac{\eta\gamma}{M}\sum_{i \in I_t} \sum_{k=1}^K g_k^i.
\end{align*}
Thus, while this does decouple $\eta$ and $\gamma$ to some degree, it does not fully do so. In particular, if we set $\gamma = 0$, then $x_{t+1} = x_t$, in which case we can make no progress overall. This is particularly important because, as we show above, for many $\Theta$, the only way for the surrogate loss to have the same critical point as the true loss is by setting $\gamma = 0$. More generally, we see that the effective learning rate used in such an update is actually the product $\eta\gamma$, which can result in conflating the effect of changes in $\gamma$ with changes in $\eta$.

\newpage

\section{Proof of results}

\subsection{Results from Section \ref{sec:local_update_as_sgd}}

\subsubsection{Lemma \ref{lem:f_i_loss}}

\begin{proof}
Using the fact that $A_z$ is symmetric for all $z$, we have
\begin{align*}
    f_i(x) &= \E_{z \sim \mD_i} [f(x; z)]\\
    &= \E_{z \sim \mD_i}\left[ \dfrac{1}{2}\norm{A_z^{1/2}(x-c_z)}^2\right]\\
    &= \E_{z\sim \mD_i}\left[ \dfrac{1}{2}(x-c_z)^TA_z(x-c_z)\right]\\
    &= \E_{z \sim \mD_i}\left[ \dfrac{1}{2}x^TA_zx - x^TA_zc_z + \dfrac{1}{2}c_z^TA_zc_z\right]\\
    &= \dfrac{1}{2}x^T\E_{z \sim \mD_i}[A_z]x -x^T\E_{z \sim \mD_i}[A_zc_z] + \dfrac{1}{2}\E_{z \sim \mD_i}[c_z^TA_zc_z]\\
    &= \dfrac{1}{2}x^TA_ix - x^TA_ic_i + \dfrac{1}{2}\E_{z \sim \mD_i}[c_z^TA_zc_z]\\
    &= \dfrac{1}{2}\norm{A_i^{1/2}(x-c_i)}^2 - c_i^TA_ic_i + \dfrac{1}{2}\E_{z \sim \mD_i}[c_z^TA_zc_z].
\end{align*}
Setting $\tau_i = - c_i^TA_ic_i + \dfrac{1}{2}\E_{z \sim \mD_i}[c_z^TA_zc_z]$, we derive the result.
\end{proof}

\subsubsection{Lemma \ref{lem:grad_recurrence}}

\begin{proof}
For simplicity of notation, let $b_z = -A_zc_z$ and $b_i = -A_ic_i$. Let $S_{:k} = (S_1, S_2, \ldots, S_k)$. By \eqref{eq:inner_g_update}, $g_k$ is independent of $w_k$ given $S_{:k-1}$. Therefore, we have that for any $k$, 
\begin{align*}
    \E[g_{k}] &= \E_{S_{:k-1}}\left[\E_{S_k}\left[\frac{1}{B}\sum_{z \in S_k} \nabla f(x_k ; z) \middle|  S_{:k-1} \right] \right]\\
    &= \E_{S_{:k-1}}\left[\E_{S_k}\left[\frac{1}{B}\sum_{z \in S_k} A_zx_k + b_z \middle| S_{:k-1} \right] \right]\\
    &= \E_{S_{:k-1}}\left[A_ix_k + b_i\right]
\end{align*}

By linearity of expectation, we then find that
\begin{equation}\label{eq:expected_g_update}
    \E[g_k] = A_i\E[x_k] + b_i.
\end{equation}

By the assumption that $A_i$ is positive definite, this implies
\begin{equation}\label{eq:inverted_recurrence}
    \E[x_k] = A_i^{-1}(\E[g_{k}] - b_i).
\end{equation}

Plugging \eqref{eq:inverted_recurrence} into \eqref{eq:inner_x_update}, we have
\begin{equation}
    \E[x_{k+1}] = \E[x_k] - \gamma \E[g_k] = A_i^{-1}(\E[g_k] - b_i) - \gamma \E[g_k].
\end{equation}

Combining \eqref{eq:expected_g_update} and \eqref{eq:inverted_recurrence}, we have
\begin{align*}
    \E[g_{k+1}] &= A_i\E[x_{k+1}] + b_i\\
    &= A_i\left(A_i^{-1}(\E[g_{k}] - b_i)-\gamma\E[g_k]\right)+b_i\\
    &= (I-\gamma A_i)\E[g_k].
\end{align*}
\end{proof}

\subsubsection{Theorem \ref{thm:maml}}

\begin{proof}
For convenience of notation, we will fix $i$ and let $X_k$ denote $X_k^i(x)$, and let $X_0 = x$. Since we assume this is computed using gradient descent, we have
\[
X_{k+1} = X_k - \gamma \nabla f_i(X_k).
\]
Note that $\nabla^2 f_i(y) = A_i$ for all $y$. Therefore, for $0 \leq k \leq K-1$,
\begin{equation}\label{eq:maml_proof_1}
\nabla_{X_k}X_{k+1}= I- \gamma A_i.\end{equation}
Also recall that by Theorem \ref{thm:sgd_objective}, we have
\begin{equation}\label{eq:maml_proof_2}
\nabla_{X_K} f_i(X_K) = (I-\gamma A_i)^{K+1}.\end{equation}

For a function $q: \R^a \to \R^b$, let its Jacobian at a point $x \in \R^a$ be denoted by $J_x(g)$. Using the chain rule for Jacobians:
\begin{align*}
    \nabla m_K^i(x) &= J_{X_{K}}( f_i(X_{K}) )J_{x}(X_{K})\\
    &= J_{X_{K}}( f_i(X_K) )\prod_{k=0}^{K-1} J_{X_k}(X_{k+1})\\
    &= (I-\gamma A_i)^{K+1}\prod_{k=1}^{K}(I- \gamma A_i).
\end{align*}
This last step follows from \eqref{eq:maml_proof_1} and \eqref{eq:maml_proof_2}, and the fact that $A_z$ is symmetric for all $z$. The remainder of the result follows directly from Theorem \ref{thm:sgd_objective}.
\end{proof}

\subsection{Results from Section \ref{sec:properties}}

\subsubsection{Lemma \ref{lem:Q_special_case}}

\begin{proof}
By \eqref{eq:Q_matrix}, we have
\[
Q_i(0, \Theta) = \sum_{k= 1}^{K(\Theta)} \theta_k(I-0)^{k-1} = \left(\sum_{k= 1}^{K(\Theta)}\theta_k\right)I.
\]
Similarly, if $\Theta = (\theta_1)$ then
\[
Q_i(\gamma, \Theta) = \theta_1I.
\]
In either event, we see that $Q_i(\gamma, \Theta) = aI$ where $a = \sum_{k= 1}^{K(\Theta)} \theta_k$. Therefore,
\[
\tilde{f}_i(x, \gamma, \Theta) = \frac{1}{2}\norm{(aA_i)^{1/2}(x-c_i)}^2 = \frac{a}{2}\norm{A_i^{1/2}(x-c_i)}^2 = a\tilde{f}_i(x,\gamma,\Theta).
\]
\end{proof}

\subsubsection{Lemma \ref{lem:Q_structure}}

\begin{proof} We prove each property below.

    \emph{Proof of Property 1:} Since $\gamma < L_i^{-1}$ and $A_i$ is symmetric and positive definite we know that the matrix $I-\gamma A_i \preceq 0$ and therefore $(I-\gamma A_i)^k$ is symmetric and positive definite for all $k$. By \eqref{eq:Q_matrix} and Assumption \ref{assm0}, $Q_i(\gamma, \Theta)$ is a nonnegative linear combination of positive definite, symmetric matrices, with some coefficient $\theta_j > 0$. It is therefore positive definite and symmetric.
    
    \emph{Proof of Property 2:} Since $\qigt$ and $A_i$ are symmetric and positive definite, so too is their product. Therefore, the $i$-th surrogate loss defined in \eqref{eq:surrogate_i} is a positive definite quadratic function, and therefore has a unique minima. Solving for this minima explicitly by setting the gradient equal to 0, we get
    \begin{align*}
        \nabla_x \tilde{f}_i(x, \gamma, \Theta) = Q_i(\gamma, \Theta)A_i(x-c_i) = 0 \implies x = c_i.
    \end{align*}
    Here, we again used the fact that $\qigt A_i$ is symmetric and positive definite.
    
    \emph{Proof of Property 3:} Let $v$ be an eigenvector of $A_i$ with eigenvalue $\lambda$. Then note that $v$ is an eigenvector of $(I-\gamma A_i)^{k-1}$ with eigenvalue $(1-\gamma \lambda)^{k-1}$. Therefore, $v$ is an eigenvector of $\qigt$ with eigenvalue
    \[
    \sum_{k = 1}^{K(\Theta)} \theta_k(1-\gamma\lambda)^{k-1}.
    \]
    
    \emph{Proof of Property 4:} This follows from Property 3, noting that since $\gamma < L_i^{-1}$, we have $\gamma\lambda < 1$ for all eigenvalues of $A_i$. By Property 3, the eigenvalue is maximized when $\lambda = \lambda_{\min}(A_i) = \mu_i$ and minimized when $\lambda = \lambda_{\max}(A_i) = L_i$. 
\end{proof}

\subsubsection{Lemma \ref{lem:QA_structure}}

\begin{proof} We prove each property below.

    \emph{Proof of Property 1:} This follows by analogous reasoning to Property 3 in Lemma \ref{lem:Q_structure}. Note that every eigenvector $v$ of $A_i$ with associated eigenvalue $\lambda$ is also an eigenvector of $(I-\gamma A_i)^{k-1}$ with eigenvalue $(1-\gamma \lambda)^{k-1}$. By basic properties of eigenvectors, this implies the desired result.
    
    \emph{Proof of Property 2:} 
        By Lemma \ref{lem:Q_structure}, and because $\gamma < L_i^{-1} \leq \mu_i^{-1}$, for all $k \geq 1$,
        \[
        0 < \lambda_{\min}((I-\gamma A_i)^{k-1}) = (1-\gamma L_i)^{k-1} \leq (1-\gamma\mu_i)^{k-1} = \lambda_{\max}((I-\gamma A_i)^{k-1}).\]
        Since $(I-\gamma A_i)^{k-1}$ and $A_i$ are positive definite, we have
        \[
        \lambda_{\max}((I-\gamma A_i)^{k-1}A_i) \leq \lambda_{\max}((I-\gamma A_i)^{k-1})\lambda_{\max}(A_i) = (1-\gamma\mu_i)^{k-1}L_i.
        \]
        Repeatedly using the fact that for Hermitian matrices, $\lambda_{\max}(A+B) \leq \lambda_{\max}(A) + \lambda_{\max}(B)$, we derive the upper bound on $\lambda_{\max}(\qigt A_i)$. The lower bound on $\lambda_{\min}(\qigt A_i)$ follows from a similar argument, using the fact that for Hermitian matrices, $\lambda_{\min}(A+B) \geq \lambda_{\min}(A) + \lambda_{\min}(B)$.
\end{proof}

\subsubsection{Lemma \ref{lem:QA_structure_fedavg}}

\begin{proof}

    \emph{Proof of Property 1:} This follows immediately from Property 1 of Lemma \ref{lem:QA_structure} and \eqref{eq:phi_alt}.
    
    \emph{Proof of Property 2:} Fix $\gamma < L_i$ and define
    \[
    g(\lambda) := \phi_{K, \lambda}(\gamma).
    \]
    Note that we have
    \[g'(\lambda) = K(1-\gamma\lambda)^{K-1}.\]
    Since $\gamma < L_i^{-1}$, for $\lambda \in [\mu_i, L_i]$, we have $g'(\lambda) > 0$. Hence, for $\lambda \in [\mu_i, L_i]$ we know
    \[
    g(\mu_i) \leq g(\lambda) \leq g(L_i).\]
    Therefore, $g(\mu_i), g(L_i)$ are the minimum and maximum eigenvalues of $\qigtok A_i$.
    
    \emph{Proof of Property 3:} As in the proof of Property 2, for a fixed $\gamma < L^{-1}$ define
    \[g(\lambda) := \phi_{K, \lambda}(\gamma).\]
    Note that by assumption, we have $\gamma < L^{-1}$. Therefore, for $\lambda \in [\mu, L]$, $g'(\lambda) > 0$, so for all $\lambda$,
    \[
    g(\mu) \leq g(\lambda) \leq g(L).
    \]
    This implies the desired result.
\end{proof}

\subsubsection{Lemma \ref{lem:QA_fedavg}}

\begin{proof}
Because $0 < \gamma < L_i^{-1}$, we have
\[
0 \prec I-\gamma A_i \prec 1.
\]
Therefore,
\[
\qigtok = \sum_{k=1}^{K} (I-\gamma A_i)^{-1}
\]
is a partial geometric sum of matrices with eigenvalues $\lambda_i$ satisfying $0 < \lambda_i < 1$. This implies that, much like a geometric series of scalars,
\[
\qigtok = (I-(I-\gamma A_i)^{-1})(\gamma A_i)^{-1}.
\]
Note that here we used the fact that $A_i$ commutes with $(I-\gamma A_i)$ to interchange their order.
\end{proof}

\subsubsection{Lemma \ref{lem:QA_structure_maml}}

\begin{proof}

    \emph{Proof of Property 1:} This follows immediately from Property 1 of Lemma \ref{lem:QA_structure} and \eqref{eq:theta_maml}.
    
    \emph{Proof of Property 2:} Fix $\gamma < (KL_i)^{-1}$ and define
    \[
    g(\lambda) := (1-\gamma\lambda)^{K-1}\lambda.
    \]
    Note that we have
    \[
    g'(\lambda) = (1-\gamma\lambda)^{K-2}(1-K\gamma\lambda).
    \]
    Since $\gamma < (KL_i)^{-1}$, for $\lambda \in [\mu_i, L_i]$, we have
    \[
    1-K\gamma\lambda > 0.
    \]
    In particular, this implies $1-\gamma\lambda > 0$, so $g'(\lambda) > 0$. Hence, for $\lambda \in [\mu_i, L_i]$ we know
    \[
    g(\mu_i) \leq g(\lambda) \leq g(L_i).\]
    Therefore, $g(\mu_i), g(L_i)$ are the minimum and maximum eigenvalues of $Q_i(\gamma, \Theta_K) A_i$.
    
    \emph{Proof of Property 3:} As in the proof of Property 2, for a fixed $\gamma < L^{-1}$ define
    \[g(\lambda) := (1-\gamma\lambda)^{K-1}\lambda.\]
    By assumption, $\gamma < L^{-1}$. Therefore, for $\lambda \in [\mu, L]$, $g'(\lambda) > 0$, so for such $\lambda$,
    \[
    g(\mu) \leq g(\lambda) \leq g(L).
    \]
    This implies the desired result.
\end{proof}

\subsection{Results from Section \ref{sec:distance}}

In order to prove the results in this section, we will use the following straightforward lemma regarding the structure of $x^*(\gamma, \Theta)$.

\begin{lemma}\label{lem:opt_point}
If $\gamma < L^{-1}$, then
\[
    x^*(\gamma, \Theta) = \E\sbr{\tilde{A}_i}^{-1}\E\sbr{\tilde{A}_ic_i}.
\]
\end{lemma}

\begin{proof}
By definition,
\[
\tilde{f}(x, \gamma, \Theta) = \E\sbr*{\frac{1}{2}\norm{(Q_iA_i)^{1/2}(x-c_i)}^2 }.
\]
Therefore,
\begin{align*}
    \nabla \tilde{f}(x, \gamma, \Theta) = \E\sbr{Q_iA_i(x-c_i)} = \E\sbr{Q_iA_i}x - \E\sbr{Q_iA_ic_i}.
\end{align*}
Since $\tilde{f}(x, \gamma, \Theta)$ is strongly convex, it then follows that
\[
x^*(\gamma, \Theta) = \E_i\sbr{Q_iA_i}^{-1}\E_i\sbr{Q_iA_ic_i}.
\]
Scaling $Q_i$ by $(\sum_k \theta_k)^{-1}$ in both parts of the right-hand side, we derive the result.
\end{proof}

We will also need the following result on the relation between $\E[A^{-1}]$ and $\E[A]^{-1}$.

\begin{lemma}[\cite{groves1969note}]\label{lemma:expect_inverse}
Let $A$ be a random matrix such that $A$ is always real, symmetric, and positive definite. Then
\[
\E[A]^{-1} \preceq \E[A^{-1}] 
\]
as long as all expectations exist.\end{lemma}

In particular, this implies that for such a random matrix $A$, we have
\[
\norm{\E[A]^{-1}} \leq \norm{\E[A^{-1}]}.
\]

\subsubsection{Lemma \ref{lem:asymptotic}}

\begin{proof}
We will first prove analogous statements for any $i \in \mI$. Fix $i, \gamma$.
By Lemma \ref{lem:QA_structure_fedavg}, we have
\[
\tilde{f}_i(x, \gamma, \Theta_{1:K}) = \dfrac{1}{2}\norm{(\qigtok A_i)^{1/2}(x-c_i)}^2 \leq \dfrac{\phi_{K, L_i}(\gamma)}{2}\norm{x-c_i}^2.
\]
Since $0 < \gamma < L_i^{-1}$, by \eqref{eq:phi}, we have $\phi_{K, L_i}(\gamma) \leq \gamma^{-1}$, implying that
\begin{equation}\label{eq:asymp1}
\tilde{f}_i(x, \gamma, \Theta_{1:K}) \leq \dfrac{1}{2\gamma}\norm{x-c_i}^2.    
\end{equation}

By Lemma \ref{lem:QA_fedavg}, we have
\[
\qigtok A_i = \gamma^{-1}(I-(I-\gamma A_i)^K).
\]
Since $0 \prec I-\gamma A_i \prec I$, this implies
\begin{equation}\label{eq:asymp2}
\lim_{K \to \infty} \qigtok A_i = \gamma^{-1}I.
\end{equation}

Define
\[
\tilde{f}^\infty_{i}(x, \gamma) :=  \dfrac{1}{2\gamma}\norm{x-c_i}^2,
\]
\[
\tilde{f}^\infty(x, \gamma) := \E_{i \sim \mP}[ \tilde{f}^\infty_{i}(x, \gamma)].
\]
Note that by Assumption \ref{assm2}, we know that $f^\infty(x, \gamma)$ is well-defined and finite at all $x \in \R^d$, as
\[
\tilde{f}^\infty(x, \gamma) = \E_{i \sim\mP}\sbr*{\dfrac{1}{2}\norm{x-c_i}^2} \leq \E_{i \sim\mP}\sbr*{ \norm{x-c}^2 + \norm{c_i-c}^2 } \leq \norm{x-c}^2 + \sigma_c^2.
\]

By \eqref{eq:asymp1}, we have that for all $K$,
\[
\tilde{f}_i(x, \gamma, \Theta_{1:K}) \leq \tilde{f}_i^\infty(x, \gamma).
\]
By \eqref{eq:asymp2}, we have
\[
\lim_{K \to \infty}\tilde{f}_i(x,\gamma, \Theta_{1:K}) = \tilde{f}_i^\infty(x, \gamma).
\]
By the dominating convergence theorem, we have that for any $x$,
\begin{align*}
\lim_{K \to \infty} \tilde{f}(x, \gamma, \Theta_{1:K}) &= \lim_{K \to \infty}\E_{i \sim\mP} \tilde{f}_i(x, \gamma, \Theta_{1:K})\\
&= \E_{i \sim\mP} \lim_{K \to \infty}\tilde{f}_i(x, \gamma, \Theta_{1:K})\\
&= \tilde{f}^\infty(x,\gamma)\\
&= \dfrac{1}{2\gamma}\E_{i \sim\mP}[\norm{x-c_i}^2].
\end{align*}
The last part of the theorem follows from the fact that the expected value minimizes the expected $\ell_2$ distance to a random point.
\end{proof}

\subsubsection{Theorem \ref{thm:optima_distance}}

\begin{proof}
Throughout this proof, we will use the fact that for a positive semi-definite, symmetric matrix $P$, $\norm{P} = \lambda_{\max}(P)$ and that if $P$ is further positive definite, $\norm{P^{-1}} = \lambda_{\min}(P)^{-1}$.

By Lemma \ref{lem:opt_point},
\begin{align*}
    \norm*{x^*(\gamma, \Theta) - x^*} &= \norm{ \E\sbr{\tilde{A}_i}^{-1}\E\sbr{\tilde{A}_ic_i} - \E\sbr{A_i}^{-1}\E\sbr{A_ic_i}}\\
    &= \norm{ \tilde{A}^{-1}\E\sbr{\tilde{A}_ic_i} - \E\sbr{c_i} +\E\sbr{c_i} - A^{-1}\E\sbr{A_ic_i}}\\
    &= \norm{\tilde{A}^{-1}\rbr{\E\sbr{\tilde{A}_ic_i} - \E\sbr{\tilde{A}_i}\E\sbr{c_i}} - A^{-1}\rbr{\E\sbr{A_ic_i} - \E\sbr{A_i}\E\sbr{c_i}}  }\\
    &= \norm{\tilde{A}^{-1}\tilde{v} - A^{-1}v}.
\end{align*}
In the above, we defined
\[
    v := \E\sbr{A_ic_i} - \E\sbr{A_i}\E\sbr{c_i}
\]
\[
    \tilde{v} := \E\sbr{\tilde{A}_ic_i} - \E\sbr{\tilde{A}_i}\E\sbr{c_i}.
\]
The vectors $v$ and $\tilde{v}$ measure the correlation between $A_i$ (or $\tilde{A}_i$) and the $c_i$. When all $A_i$ are equal, or all $c_i$ are equal, one can easily show that $v = \tilde{v} = 0$. We then have
\begin{align*}
    \norm*{x^*(\gamma, \Theta) - x^*} &= \norm{\tilde{A}^{-1}\rbr{\tilde{v}-v} + \rbr{\tilde{A}^{-1} - A^{-1}}v} \\
    & \leq \underbrace{\norm{\tilde{A}^{-1}}}_{T_1}\underbrace{\norm{\tilde{v}-v}}_{T_2} + \underbrace{\norm{\tilde{A}^{-1} - A^{-1}}}_{T_3}\underbrace{\norm{v}}_{T_4}.
\end{align*}

We first bound $T_1$. Since $Q_i$ and $A_i$ are positive definite, so too is $\tilde{A}_i$ and therefore $\tilde{A}$. We therefore have,
\begin{align*}
    T_1 &= \norm{\tilde{A}^{-1}}\\
    &= \norm{\E\sbr{ \tilde{A}_i } ^{-1} }\\
    &\leq \norm{\E\sbr{ \tilde{A}_i^{-1} }}\\
    &\leq \E\norm{ \tilde{A}_i^{-1} }\\
    &= \E\sbr{\norm{A_i^{-1}(\tau Q_i)^{-1}}}\\
    &\leq \E\sbr{\norm{A_i^{-1}}\norm{(\tau Q_i)^{-1}}}\\
    &\leq \mu^{-1}\E\sbr{\lambda_{\min}(\tau Q_i)^{-1}}.
\end{align*}

For $T_2$, we have the following:
\begin{align*}
    T_2 &= \norm{\tilde{v} - v}\\
    &= \norm{\rbr{\E\sbr{\tau Q_i A_i c_i} - \E\sbr{\tau Q_i A_i}\E\sbr{c_i}} - \rbr{\E\sbr{A_ic_i}-\E\sbr{A_i}\E\sbr{c_i} }}\\
    &= \norm{\rbr{\E\sbr{\tau Q_i A_i c_i} - \E\sbr{\tau Q_i A_i}c} - \rbr{\E\sbr{A_ic_i}-\E\sbr{A_i}c }}\\
    &= \norm{\E\sbr{(\tau Q_i - I)(A_ic_i-A_ic)}}\\
    &\leq \sqrt{\E\sbr{\norm{\tau Q_i-I}^2}\E\sbr{\norm{A_ic_i-A_ic}^2}}\\
    &\leq \sqrt{\E\sbr{\norm{\tau Q_i-I}^2}\E\sbr{\norm{A_i}^2\norm{c_i-c}^2}}\\
    &\leq L\sigma_c\sqrt{\E\sbr{\norm{\tau Q_i - I}^2}}.
\end{align*} 

For $T_3$, we use the fact that for invertible matrices $B, C \in \R^{n\times n}$,
\begin{align*}
    \norm{B^{-1}-C^{-1}} = \norm{B^{-1}(B-C)C^{-1}} \leq \norm{B^{-1}}\norm{C^{-1}}\norm{B-C}.
\end{align*}

Therefore,
\begin{align*}
    T_3 &\leq \norm{A^{-1}}\norm{\tilde{A}^{-1}}\norm{\tilde{A}-A}.
\end{align*}

Note that $\norm{\tilde{A}^{-1}}$ is simply $T_1$ above, and that using a similar analysis, we have
\[
    \norm{A^{-1}} \leq \E\sbr{ \lambda_{\min}\rbr{A_i}^{-1}} \leq \mu^{-1}.
\]

By the Cauchy-Schwarz inequality,
\begin{align*}
    \norm{\tilde{A}-A} &= \norm{\E\sbr{\tau Q_iA_i - A_i}}\\
    &\leq \norm{\E\sbr{\rbr{\tau Q_i - I}A_i}}\\
    &\leq \sqrt{\E\sbr{\norm{\tau Q_i - I}^2} \E\sbr{\norm{A_i}^2}}\\
    &\leq L\sqrt{\E\sbr{\norm{\tau Q_i - I}^2}}.
\end{align*}

For $T_4$, we again use the Cauchy-Schwarz inequality, as
\begin{align*}
    T_4 &= \norm{\E\sbr{A_ic_i}- \E\sbr{A_i}\E\sbr{c_i}}\\
    &= \norm{\E\sbr{(A_i-A)(c_i-c)}}\\
    &\leq \sqrt{\E\sbr{\norm{A_i-A}^2}\E\sbr{\norm{c_i-c}^2}}\\
    &\leq \sigma_A\sigma_c.
\end{align*}

Putting this all together, we have
\begin{equation}\label{eq:distance_bound1}
    \norm*{x^*(\gamma, \Theta) - x^*} \leq \dfrac{L\sigma_c}{\mu}\rbr*{1+\dfrac{\sigma_A}{\mu}}\E\sbr{\lambda_{\min}(\tau Q_i)^{-1}}\sqrt{\E\sbr{\norm{\tau Q_i - I}^2}}.
\end{equation}

Note that by Assumption \ref{assm1}, we have that $1-\gamma \mu \leq 1, 1-\gamma L \leq 1$. By Lemma \ref{lem:Q_structure} and the definition of $\tau$, this implies that for all $i$,
\[
0 \prec \tau Q_i \preceq I.
\]
Therefore for all $i$,
\begin{equation}\label{eq:lambda_min2}
\norm{\tau Q_i - I} = \lambda_{\max}(I-\tau Q_i) \leq 1-\lambda_{\min}(\tau Q_i).
\end{equation}

Combining \eqref{eq:distance_bound1} and \eqref{eq:lambda_min2}, we have
\[
\norm*{x^*(\gamma, \Theta) - x^*} \leq \dfrac{L\sigma_c}{\mu}\rbr*{1+\dfrac{\sigma_A}{\mu}}\dfrac{1-\tau \min_i\{\lambda_{\min}(Q_i)\}}{\tau \min_i\{\lambda_{\min}(Q_i)\}}.
\]

By Lemma \ref{lem:Q_structure} and Assumption \ref{assm1}, we find that for all $i$,
\[
\tau \lambda_{\min}(\qigt) \geq \chi(\gamma, \Theta).
\]
Therefore,
\[
\dfrac{1-\tau \min_i\{\lambda_{\min}(Q_i)\}}{\tau \min_i\{\lambda_{\min}(Q_i)\}} \leq \dfrac{1-\chi(\gamma, \Theta)}{\chi(\gamma, \Theta)}.
\]

\end{proof}

\subsubsection{Theorem \ref{thm:optima_distance_fedavg}}

Note that for $\Theta=\Theta_{1:K}$, the quantity $\tau$ defined in $\eqref{eq:distance_quantities}$ is given by $\tau = K^{-1}$. We will require one further auxiliary lemma.
\begin{lemma}\label{lem:diff_qa_a}
    If $\gamma < L^{-1}$, then for all $i \in \mI$ we have
    \[
    \norm{\tau \qigtok A_i - I} \leq \dfrac{KL -\phi_{K, L}(\gamma)}{K}
    \]
    where $\phi_{K, L}(\gamma)$ is as defined in \eqref{eq:phi}.
\end{lemma}

\begin{proof}
    First note that since $\tau = K^{-1}$,
    \[
    \norm{\tau \qigtok A_i - I} = \dfrac{\norm{\qigtok A_i - KI}}{K}.\]

    It therefore suffices to show that
    \[
    \norm{\qigtok A_i - KI} \leq KL-\phi_{K, L}(\gamma).
    \]
    As discussed in the proof of Lemma \ref{lem:QA_structure} and Lemma \ref{lem:QA_structure_fedavg}, $\qigtok A_i$ has eigenvalues of the form $\phi_{K, \lambda}(\gamma)$ where $\lambda$ is an eigenvalue of $A_i$. Therefore, $KI - \qigtok A_i$ has eigenvalues of the form
    \[
    q(\lambda) = K\lambda - \phi_{K, \lambda}(\gamma)
    \]
    where $\lambda$ is an eigenvalue of $A_i$. As noted in \eqref{eq:phi_alt}, for $\lambda\gamma \leq 1$,
    \[
    \phi_{K, \lambda}(\gamma) = \sum_{k = 1}^K (1-\gamma\lambda)^{k-1}\lambda \leq K\lambda.
    \]
    Therefore, $KI-\qigtok A_i$ is symmetric and positive semi-definite. Moreover, a simple computation shows
    \[
    q'(\lambda) = K-K(1-\gamma\lambda)^{K-1}
    \]
    which is nonnegative for $0 \leq \gamma\lambda \leq 1$. In particular it is nonnegative for $\gamma \in [0, L^{-1}]$ (as in Assumption \ref{assm1}), implying that
    \[
    \lambda_{\max}(KI-\qigtok A_i) \leq q(L) = KL-\phi_{K, L}(\gamma).
    \]
\end{proof}

We can now prove the desired theorem.

\begin{proof}[Proof of Theorem \ref{thm:optima_distance_fedavg}]
    We begin our proof in a similar manner to the proof of Theorem \ref{thm:optima_distance}. Note that for FedAvg with $K$ local steps, $\tau = K$. By Lemma \ref{lem:opt_point},
    \begin{align*}
        \norm*{x^*(\gamma, \Theta_{1:K}) - x^*} &= \norm{ \E\sbr{\tilde{A}_i}^{-1}\E\sbr{\tilde{A}_ic_i} - \E\sbr{A_i}^{-1}\E\sbr{A_ic_i}}\\
        &= \norm{ \tilde{A}^{-1}\E\sbr{\tilde{A}_ic_i} - \E\sbr{c_i} +\E\sbr{c_i} - A^{-1}\E\sbr{A_ic_i}}\\
        &= \norm{\tilde{A}^{-1}\rbr{\E\sbr{\tilde{A}_ic_i} - \E\sbr{\tilde{A}_i}\E\sbr{c_i}} - A^{-1}\rbr{\E\sbr{A_ic_i} - \E\sbr{A_i}\E\sbr{c_i}}  }\\
        &= \norm{\tilde{A}^{-1}\tilde{v} - A^{-1}v}.
    \end{align*}
    In the above, we defined
    \[
        v := \E\sbr{A_ic_i} - \E\sbr{A_i}\E\sbr{c_i}
    \]
    \[
        \tilde{v} := \E\sbr{\tilde{A}_ic_i} - \E\sbr{\tilde{A}_i}\E\sbr{c_i}.
    \]
    We then have
    \begin{align*}
        \norm*{x^*(\gamma, \Theta_{1:K}) - x^*} &= \norm{\tilde{A}^{-1}\rbr{\tilde{v}-v} + \rbr{\tilde{A}^{-1} - A^{-1}}v} \\
        & \leq \underbrace{\norm{\tilde{A}^{-1}}}_{T_1}\underbrace{\norm{\tilde{v}-v}}_{T_2} + \underbrace{\norm{\tilde{A}^{-1} - A^{-1}}}_{T_3}\underbrace{\norm{v}}_{T_4}.
    \end{align*}
    
    We first bound $T_1$. Since $Q_i$ and $A_i$ are positive definite, so too is $\tilde{A}_i$ and therefore $\tilde{A}$. We therefore have,
    \begin{align*}
        T_1 &= \norm{\tilde{A}^{-1}}\\
        &= \norm{\E\sbr{ \tilde{A}_i } ^{-1} }\\
        &\leq \norm{\E\sbr{ \tilde{A}_i^{-1} }}\\
        &\leq \E\norm{ \tilde{A}_i^{-1} }.
    \end{align*}
    
    The penultimate step follows by Lemma \ref{lemma:expect_inverse}. By definition of $\tilde{A}_i^{-1}$, we have
    \begin{align*}
    T_1 &\leq \E\norm{ \tilde{A}_i^{-1} }\\
        &= K\E\sbr{\norm{(\qigtok A_i)^{-1}}}\\
        &\leq K\phi_{K, \mu}(\gamma)^{-1}.
    \end{align*}
    Note that the last step follows by \eqref{eq:phi} and Lemma \ref{lem:QA_structure_fedavg}.
    
    For $T_2$, we have the following:
    \begin{align*}
        T_2 &= \norm{\tilde{v} - v}\\
        &= \norm{\rbr{\E\sbr{\tau Q_i A_i c_i} - \E\sbr{\tau Q_i A_i}\E\sbr{c_i}} - \rbr{\E\sbr{A_ic_i}-\E\sbr{A_i}\E\sbr{c_i} }}\\
        &= \norm{\rbr{\E\sbr{\tau Q_i A_i c_i} - \E\sbr{\tau Q_i A_i}c} - \rbr{\E\sbr{A_ic_i}-\E\sbr{A_i}c }}\\
        &= \norm{\E\sbr{(\tau \qigtok A_i - A_i)(c_i-c)}}\\
        &\leq \sqrt{\E\sbr{\norm{\tau \qigtok A_i-A_i}^2}\E\sbr{\norm{c_i-c}^2}}\\
        &\leq \dfrac{\sigma_c(KL-\phi_{K, L}(\gamma))}{K}.
    \end{align*} 
    This last step holds by Lemma \ref{lem:diff_qa_a} and Assumption \ref{assm2}.
    
    For $T_3$, we use the fact that for invertible matrices $B, C \in \R^{n\times n}$,
    \begin{align*}
        \norm{B^{-1}-C^{-1}} = \norm{B^{-1}(B-C)C^{-1}} \leq \norm{B^{-1}}\norm{C^{-1}}\norm{B-C}.
    \end{align*}
    
    Therefore,
    \begin{align*}
        T_3 &\leq \norm{A^{-1}}\norm{\tilde{A}^{-1}}\norm{\tilde{A}-A}.
    \end{align*}
    
    Note that $\norm{\tilde{A}^{-1}}$ is simply $T_1$ above, and that using a similar analysis, we have
    \[
        \norm{A^{-1}} \leq \E\sbr{ \lambda_{\min}\rbr{A_i}^{-1}} \leq \mu^{-1}.
    \]
    
    Again using Lemma \ref{lem:diff_qa_a},
    \begin{align*}
        \norm{\tilde{A}-A} &= \norm{\E\sbr{\tau Q_iA_i - A_i}}\\
        &\leq \E\norm{\sbr{\rbr{\tau Q_i - I}A_i}}\\
        &\leq \dfrac{KL-\phi_{K, L}(\gamma)}{K}.
    \end{align*}
    
    For $T_4$, we again use the Cauchy-Schwarz inequality, as
    \begin{align*}
        T_4 &= \norm{\E\sbr{A_ic_i}- \E\sbr{A_i}\E\sbr{c_i}}\\
        &= \norm{\E\sbr{(A_i-A)(c_i-c)}}\\
        &\leq \sqrt{\E\sbr{\norm{A_i-A}^2}\E\sbr{\norm{c_i-c}^2}}\\
        &\leq \sigma_A\sigma_c.
    \end{align*}
    
    Putting this all together, we have derive the result.
\end{proof}

\subsubsection{Lemma \ref{lem:gamma_cond}}

\begin{proof}
    By properties of geometric sums, we have
    \[
    \phi_{K, \lambda}(\gamma) = \sum_{k = 1}^K (1-\gamma\lambda)^{k-1}\lambda.
    \]
    
    Note that this implies that the function is differentiable everywhere. Taking a derivative, we have
    \[
    \phi_{K, \lambda}'(\gamma) = \sum_{k=1}^K -(k-1)(1-\gamma\lambda)^{k-2}\lambda^2.
    \]
    Note that all terms in this sum are nonnegative when $\gamma \in [0, \lambda^{-1}]$. By assumption on $\epsilon$, we have
    \[
    \gamma \leq \dfrac{\ln(1/(1-\epsilon))}{K\lambda} \leq \lambda^{-1}.
    \]
    It therefore suffices sto show the desired lower bound on $\phi_{K, \lambda}(\gamma)$ when
    \[
    \gamma = \dfrac{\ln(1/(1-\epsilon))}{K\lambda}.
    \]
    By definition of $\phi_{K, \lambda}$, we have
    \begin{align*}
        \phi_{K, \lambda}(\gamma) = \dfrac{K\lambda}{\ln(1/(1-\epsilon))}\left(1-\left(1-\dfrac{\ln(1/(1-\epsilon))}{K}\right)^K\right).
    \end{align*}
    Let $x = \ln(1/(1-\epsilon))$. Note that by assumption on $\epsilon$, $0 \leq x \leq 1$. It then suffices to show that
    \[
    1-\left(1-\dfrac{x}{K}\right)^K \geq (1-\epsilon)x
    \]
    or equivalently,
    \begin{equation}\label{eq:aux_cond}
    \left(1-\dfrac{x}{K}\right)^K \leq 1-(1-\epsilon)x.
    \end{equation}
    By standard properties of exponentials,
    \[
    \left(1-\dfrac{x}{K}\right)^K \leq e^{-x} = e^{\ln(1-\epsilon)} = 1-\epsilon.
    \]
    Letting $y = 1-\epsilon$, we then use the fact that for $y \in (0, 1]$,
    \[
    y \leq 1 + y\ln(y) = 1 + (1-\epsilon)\ln(1-\epsilon) = 1-(1-\epsilon)\ln(1/(1-\epsilon)).
    \]
    This implies \eqref{eq:aux_cond}, proving the result.
\end{proof}

\subsubsection{Theorem \ref{thm:distance_optima_gamma}}

We will first require an auxiliary lemma.

\begin{lemma}\label{lem:diff_qa_structure}
    Suppose $\gamma_1 \leq \gamma_2 \leq L^{-1}$ and Assumption \ref{assm1} holds. Then for all $i$, the matrix
    \[
    C = \qigti A_i - \qigtii A_i\]
    is positive semidefinite and satisfies
    \[
    \lambda_{\max}(C) \leq \phi_{K, L}(\gamma_1) - \phi_{K, L}(\gamma_2).
    \]
\end{lemma}

\begin{proof}
    Recall that by Lemma \ref{lem:QA_structure_fedavg}, the eigenvalues of $\qigt A_i$ are of the form
    \[
    \gamma^{-1}(1-(1-\gamma\lambda)^K).
    \]
    We define a function
    \[
    h(\lambda) := \gamma_1^{-1}(1-(1-\gamma_1\lambda)^K) - \gamma_2^{-1}(1-(1-\gamma_2\lambda)^K).
    \]
    Since $\qigti A_i$ and $\qigtii A_i$ share the same eigenvectors as $A_i$, the eigenvalues of $C$ are of the form $h(\lambda)$ where $\lambda$ is an eigenvalue of $A_i$. Since $\gamma_1 \leq \gamma_2 < L_i^{-1}$, we clearly have $h(\lambda) \geq 0$ for $\lambda \in [\mu_i, L_i]$, implying that $C$ is positive semidefinite.
    
    For the maximum eigenvalue of $C$, we consider $h'(\lambda)$. A simple calculation shows
    \[
    h'(\lambda) = K(1-\gamma_1 \lambda)^{K-1} - K(1-\gamma_2\lambda)^{K-1}.
    \]
    Since $\gamma_1 \leq \gamma_2 < L^{-1}$, we find that $h'(\lambda) \geq 0$ for $\lambda \in [0, L]$. Therefore, the maximum eigenvalue of $C$ satisfies
    \[
    \lambda_{\max}(C) \leq h(L) = \phi_{K, L}(\gamma_1) - \phi_{K, L}(\gamma_2).
    \]
\end{proof}

We can now use this to prove the desired result, in a manner similar to the proof of Theorem \ref{thm:optima_distance_fedavg}.

\begin{proof}[Proof of Theorem \ref{thm:distance_optima_gamma}]
    For notational convenience, we define the following quantities (where $i \in \mI, j \in \{1, 2\})$
    \begin{align*}
        B_{i, j} &:= \tau Q_i(\gamma_j, \Theta_{1:K}) A_i\\
        B_j &:= \E_i\sbr*{B_{i,j}}\\
        v_j &:= \E[B_{i,1}c_i] - \E[B_{i, j}]\E[c_i]
    \end{align*}
    By Lemma \ref{lem:opt_point}, we know
    \begin{align*}
    \norm{x^*(\gamma_1, \Theta_{1:K}) - x^*(\gamma_2, \Theta_{1:K})} &= \norm{B_1^{-1}\E[B_{i,1}c_i] - B_2^{-1}\E[B_{i,2}c_i]}\\
    &= \norm{B_1^{-1}\E[B_{i,1}c_i] - \E[c_i] + \E[c_i] - B_2^{-1}\E[B_{i,2}c_i]}\\
    &= \norm{B_1^{-1}v_1 - B_2^{-1}v_2}.
    \end{align*}
    Splitting this up, we get
    \begin{align*}
    \norm{x^*(\gamma_1, \Theta_{1:K}) - x^*(\gamma_2, \Theta_{1:K})} &= \norm{B_1^{-1}(v_1-v_2) + (B_1^{-1}-B_2^{-1})v_2}\\
    &\leq \underbrace{\norm{B_1^{-1}}}_{T_1}\underbrace{\norm{v_1-v_2}}_{T_2} + \underbrace{\norm{B_1^{-1}-B_2^{-1}}}_{T_3}\underbrace{\norm{v_2}}_{T_4}.\end{align*}
    
    As in the proof of Theorem \ref{thm:optima_distance}, we can use the fact that $Q_i(\gamma_1, \Theta)$ and $A_i$ are symmetric and positive definite to bound $T_1$, as we have
    \begin{align*}
        T_1 &= \norm{B_1^{-1}}\\
        &= \norm{\E\sbr{B_{1, i}}^{-1}}\\
        &\leq \norm{\E[B_{1,i}^{-1}]}\\
        &\leq \E\norm{B_{1,i}^{-1}}.
    \end{align*}
    
    The penultimate step follows by Lemma \ref{lemma:expect_inverse}. By definition of $B_{1,i}^{-1}$,
    \begin{align*}
    T_1 &\leq \E\norm{B_{1,i}^{-1}}\\
        &= \E\norm{(K^{-1} \qigti A_i)^{-1}}\\
        &\leq K\lambda_{\min}(\qigti A_i)^{-1}\\
        &\leq K\phi_{K, \mu}(\gamma_1)^{-1}.
    \end{align*}
    This last step follows directly from Lemma \ref{lem:QA_structure_fedavg}.
    
    For $T_2$, we have the following:
    \begin{align*}
        T_2 &= \norm{v_1 - v_2}\\
        &= \norm{\rbr{\E\sbr{B_{i, 1} c_i} - \E\sbr{B_{1,i}}\E\sbr{c_i}} - \rbr{\E\sbr{B_{i, 2}c_i}-\E\sbr{B_{i, 2}}\E\sbr{c_i} }}\\
        &= \norm{\rbr{\E\sbr{B_{i, 1} c_i} - \E\sbr{B_{1,i}}c} - \rbr{\E\sbr{B_{i, 2}c_i}-\E\sbr{B_{i, 2}}c }}\\
        &= \norm{\E\sbr{(B_{i, 1}-B_{i, 2})c_i - (B_{i, 1}-B_{i, 2})c}}\\
        &= \norm{\E\sbr{(B_{i, 1}-B_{i, 2})(c_i-c)}}\\
        &\leq \sqrt{\E[\norm{B_{i,1} - B_{i,2}}^2]\E[\norm{c_i-c}^2]}\\
        &\leq \sigma_c\sqrt{\E[\norm{B_{i,1} - B_{i,2}}^2]}.
    \end{align*} 
    
    By definition of $B_{i, j}$, we have
    \begin{align*}
        \norm{B_{i, 1} - B_{i, 2}} &= \norm{K^{-1}\qigti A_i - K^{-1}\qigtii A_i}\\
        &\leq K^{-1} \norm{\qigti A_i - \qigtii A_i}\\
        &\leq K^{-1}(\phi_{K, L}(\gamma_1) - \phi_{K, L}(\gamma_2)).
    \end{align*}
    Here, this last step follows from Lemma \ref{lem:diff_qa_structure}.
    
    For $T_3$, we use the fact that for invertible matrices $A, B \in \R^{n\times n}$,
    \begin{align*}
        \norm{A^{-1}-B^{-1}} = \norm{A^{-1}(A-B)B^{-1}} \leq \norm{A^{-1}}\norm{B^{-1}}\norm{A-B}.
    \end{align*}
    
    Therefore,
    \begin{align*}
        T_3 &\leq \norm{B_1^{-1}}\norm{B_2^{-1}}\norm{B_1-B_2}.
    \end{align*}
    
    Note that $\norm{B_j^{-1}}$ can be bounded in the same manner as $T_1$ above. For the remaining term, we have
    \begin{align*}
        \norm{B_1 - B_2} &= K^{-1}\norm{\E[\qigti A_i - \qigtii A_i]}\\
        &\leq K^{-1} \E\norm{\qigti A_i - \qigtii A_i}\\
        &\leq K^{-1}(\phi_{K, L}(\gamma_1) - \phi_{K, L}(\gamma_2)).
    \end{align*}
    This last step follows from Lemma \ref{lem:diff_qa_structure}.
    
    Finally, for $T_4$ we have
    \begin{align*}
        T_4 &= \norm{\E[B_{i,2}c_i] - \E[B_{i, 2}]\E[c_i]}\\
        &= \norm{\E[(B_{i, 2}-B_2)(c_i-c)}\\
        &\leq \sqrt{\E[\norm{B_{i,2}-B_2}^2]\E[\norm{c_i-c}^2]}\\
        &\leq \sigma_c\sqrt{\E[\norm{B_{i, 2}-B_2}^2]}\\
        &\leq \sigma_c\sqrt{\E[\norm{B_{i, 2}}^2}.
    \end{align*}
    By Lemma \ref{lem:QA_structure_fedavg}, we have that for all $i$,
    \begin{align*}
        \norm{B_{i, 2}} = \norm{\tau\qigtii A_i} = \dfrac{\phi_{K, L}(\gamma_2)}{K}
    \end{align*}
    therefore implying that
    \[
    T_4 \leq \sigma_c\dfrac{\phi_{K, L}(\gamma_2)}{K}.
    \]
    
    Putting this all together, we find
    \begin{align*}
    \norm{x^*(\gamma_1, \Theta_{1:K}) - x^*(\gamma_2, \Theta_{1:K})} &\leq T_1T_2 + T_3T_4\\
    &\leq \sigma_c\left(1 + \dfrac{\phi_{K, L}(\gamma_2)}{\phi_{K, \mu}(\gamma_2)}\right)\left(\dfrac{\phi_{K, L}(\gamma_1) - \phi_{K, L}(\gamma_2)}{\phi_{K, \mu}(\gamma_1)}\right).
    \end{align*}
\end{proof}

\subsubsection{Corollary \ref{cor:distance_optima_gamma}}

To prove this, we will need to first bound the term $\phi_{K, L}(\gamma)/\phi_{K, \mu}(\gamma)$. To do so, we first require a simple lemma regarding ratios of sums.

\begin{lemma}\label{lem:frac_sum}
For $n \geq 1$, let $a_1, \dots, a_n, b_1, \dots b_n$ be positive real numbers such that for $1 \leq i \leq n$,
\[
\frac{a_i}{b_i} \leq c.
\]
Then
\[
\dfrac{a_1 + \dots + a_n}{b_1 + \dots + b_n} \leq c.
\]
\end{lemma}
\begin{proof}
We will prove this inductively. Note that when $n = 1$, the result immediately follows by assumption.

For $n > 1$, applying the inductive hypothesis to $a_1 + \dots a_{n-1}$, we have
\begin{align*}
    \dfrac{a_1 + \dots + a_n}{b_1 + \dots + b_n} &\leq \dfrac{c(b_1 + \dots b_{n-1}) + a_n}{b_1 + \dots + b_n}
\end{align*}
Let $x = c(b_1 + \dots + b_{n-1}), y = b_1 + \dots + b_{n-1}$. Note that $x/y \leq c$, so applying the inductive hypothesis we have
\begin{align*}
    \dfrac{a_1 + \dots + a_n}{b_1 + \dots + b_n} &\leq \dfrac{c(b_1 + \dots b_{n-1}) + a_n}{b_1 + \dots + b_n}\\
    &= \dfrac{x + a_n}{y + b_n}\\
    &\leq c.
\end{align*}
    
\end{proof}

We can now derive a bound on the condition number $\phi_{K, L}(\gamma)/\phi_{K, \mu}(\gamma)$.

\begin{lemma}\label{lem:cond_number}
    For $\mu \leq L$, $\gamma \in [0, L^{-1}]$ and $K \geq 1$,
    \begin{equation}\label{eq:cond_number}
    \dfrac{\phi_{K, L}(\gamma)}{\phi_{K, \mu}(\gamma)} \leq \dfrac{L}{\mu}.
    \end{equation}
\end{lemma}
\begin{proof}
    For $1 \leq j \leq K$, define
    \[
    a_j := (1-\gamma L)^{j-1}L,
    \]
    \[
    b_j := (1-\gamma \mu)^{j-1}\mu.
    \]
    Since $\mu \leq L$, $(1-\gamma L) \leq (1-\gamma \mu)$.  Therefore,
    \[
    \dfrac{a_i}{b_i} = \dfrac{L}{\mu}\dfrac{(1-\gamma L)^{i-1}}{(1-\gamma\mu)^{i-1}} \leq \frac{L}{\mu}.
    \]
    Applying Lemma \ref{lem:frac_sum}, we have
    \[
    \dfrac{a_1 + \dots a_K}{b_1 + \dots b_K} \leq \frac{L}{\mu}.
    \]
    The proof then follows by noting that
    \[
    \phi_{K, L}(\gamma) = a_1 + \dots a_K
    \]
    and
    \[
    \phi_{K, \mu}(\gamma) = b_1 + \dots b_K.
    \]
\end{proof}

With this in hand, we can prove Corollary \ref{cor:distance_optima_gamma}.

\begin{proof}[Proof of Corollary \ref{cor:distance_optima_gamma}]
    By definition of $\phi_{K, L}$, we have
    \begin{align*}
        \phi_{K, L}(\gamma_1) - \phi_{K, L}(\gamma_2) & = \sum_{k = 1}^K (1-\gamma_1 L)^{k-1}L - (1-\gamma_2 L)^{k-1}L\\
        & = \sum_{k = 2}^K (1-\gamma_1 L)^{k-1}L - (1-\gamma_2 L)^{k-1}L\\
        &\leq \sum_{k=2}^K (1-\gamma_1 L)^{k-1}L - (1-\gamma_2L)(1-\gamma_1L)^{k-2}L\\
        &= \sum_{k=2}^K (1-\gamma_1L)^{k-2}L\left((1-\gamma_1L)-(1-\gamma_2L)\right)\\
        &= \left(\sum_{k=2}^K(1-\gamma_1L)^{k-2}L\right)L(\gamma_2-\gamma_1).
    \end{align*}
    Here the second line follows from the fact that $(1-\gamma_1L)^0 = (1-\gamma_2L)^0 = 1$, and the third line follows from the fact that $\gamma_1 \leq \gamma_2$. Therefore,
    \begin{align*}
        \phi_{K, L}(\gamma_1) - \phi_{K, L}(\gamma_2) &\leq \left(\sum_{k=2}^K(1-\gamma_1L)^{k-2}\right)L(\gamma_2-\gamma_1)\\
        &\leq \phi_{K, L}(\gamma_1)L(\gamma_2-\gamma_1).
    \end{align*}
    By Theorem \ref{thm:distance_optima_gamma}, this implies
    \[
    \norm{x^*(\gamma_1, \Theta_{1:K}) - x^*(\gamma_2, \Theta_{1:K})} \leq \sigma_c\left(1 + \dfrac{\phi_{K, L}(\gamma_2)}{\phi_{K, \mu}(\gamma_2)}\right)\dfrac{\phi_{K, L}(\gamma_1)}{\phi_{K, \mu}(\gamma_1)}L(\gamma_2-\gamma_1).
    \]
    The result follows by applying Lemma \ref{lem:cond_number} and noting that $L \geq \mu$.
\end{proof}

\subsection{Results in Section \ref{sec:convergence}}

\subsubsection{Lemma \ref{lem:var_bound}}

\begin{proof}
    Recall that in Algorithm \ref{alg:outerloop},
    \[
    q_\gamma(x) = \frac{1}{M}\sum_{i \in \mI} q^i
    \]
    where $q^i = \InnerLoop(i, x, \gamma, \Theta)$ (Algorithm \ref{alg:innerloop}) and $\mI$ is a set of size $M$ sampled independently and uniformly at random from $\mP$. Since the $q^i$ are independent, it suffices to show that for any $i$, the vector $q^i := \InnerLoop(i, x, \gamma, \Theta)$ satisfies
    \[
    \E\norm{q^i - \E[q^i]}^2 \leq \dfrac{KG^2}{B}.
    \]
    By Algorithm \ref{alg:innerloop}, we have
    \[
    q^i = \sum_{k = 1}^K g_k
    \]
    where $g_K$ is a mini-batch stochastic gradient of batch size $B$ taken at $x_k$, and the $x_k$ are updated via
    \[
    x_{k+1} = x_k - \gamma g_k.
    \]
    Let $s_k = g_k - \E[g_k]$, and let $s_{1:k} = s_1 + \dots + s_k$. Note that $\{s_{1:k}\}_{k=1}^K$ form a Martingale sequence. We therefore have
    \begin{align*}
        \E\norm{q^i - \E[q^i]}^2 = \E\norm{s_{1:k}}^2 = \sum_{k=1}^K \E\norm{s_{1:k}}^2.
    \end{align*}
    
    By Assumption \ref{assm3}, we have that for any $k$ and $z \sim \mD_i$,
    \[
    \E\norm{\nabla f(x_k; z) - \E[\nabla f(x_k; z)]} \leq G^2.
    \]
    Since
    \[
    s_{k} = \dfrac{1}{B}\sum_{z \in S_k}\nabla f(x_k; z) - \E[\nabla f(x_k; z)]
    \]    
    where $|S_k| = B$ and each $z \in S_k$ is identically and independently distributed, we have
    \[
    \E\norm{s_k}^2 \leq \dfrac{G^2}{B}
    \]
    implying from our reasoning above that
    $\E\norm{s_{1:k}}^2 \leq KG^2/B$, therefore implying the desired result on the variance of $q_t$. 
\end{proof}

\subsubsection{Lemma \ref{lem:descent_step}}

\begin{proof}
    We will proceed using a similar analysis to Lemma 1 in \cite{rakhlin2011making}. 
    
    Since $\tilde{f}_\gamma$ is $\mu_\gamma$ strongly convex (by Lemma \ref{lem:ft_params}), we have
    \begin{equation}\label{eq:sc1}
    \langle \nabla \tilde{f}_\gamma(x_t), x_t-x_\gamma^*\rangle \geq \tilde{f}_\gamma(x_t)-\tilde{f}_\gamma^* + \frac{\mu_\gamma}{2}\norm{x_t-x_\gamma^*}^2,\end{equation}
    
    \begin{equation}\label{eq:sc2}
    \tilde{f}_\gamma(x_t) - \tilde{f}_\gamma^* \geq \dfrac{\mu_\gamma}{2}\norm{x_t-x_\gamma^*}^2.
    \end{equation}
    
    Since $\tilde{f}_\gamma$ is $L_\gamma$ smooth, we have
    \begin{equation}\label{eq:smooth}
    \norm{\nabla \tilde{f}_\gamma(x)} \leq L_\gamma\norm{x - x_\gamma^*}.
    \end{equation}
    
    Let $q_t := q_\gamma(x_t)$. Recall that we have
    \[
    x_{t+1} = x_t-\eta_tq_t
    \]
    and
    \[
    \E[q_t] = \nabla \tilde{f}_\gamma(x_t).
    \]
    Using the equations \eqref{eq:sc1} and \eqref{eq:sc2}, we have
    \begin{align*}
        &\E[\norm{x_{t+1}-x_\gamma^*}^2]\\ &= \E[\norm{x_t - \eta_tq_t - x_\gamma^*}^2]\\
        & = \E[\norm{x_t-x_\gamma^*}^2] -2\eta_t\E[\langle q_t, x_t-x_\gamma^*\rangle] + \eta_t^2\E[\norm{q_t}^2]\\
        &\leq \E[\norm{x_t-x_\gamma^*}^2] - 2\eta_t\E[\langle q_t, x_t-x_\gamma^*\rangle] + \eta_t^2\E[\norm{q_t}^2]\\
        &= \E[\norm{x_t-x_\gamma^*}^2] - 2\eta_t\E[\langle \nabla \tilde{f}_\gamma(x_t), x_t-x_\gamma^*\rangle] + \eta_t^2\E[\norm{q_t}^2]\\
        &\leq \E[\norm{x_t-x_\gamma^*}^2] - 2\eta_t \E\sbr*{\tilde{f}_\gamma(x_t)-\tilde{f}_\gamma^* + \frac{\mu_\gamma}{2}\norm{x_t-x_\gamma^*}^2 } + \eta_t^2\E[\norm{q_t}^2]\\
        &\leq \E[\norm{x_t-x_\gamma^*}^2] -2\eta_t\E\sbr*{\frac{\mu_\gamma}{2}\norm{x_t-x_\gamma^*}^2 + \frac{\mu_\gamma}{2}\norm{x_t-x_\gamma^*}^2 } + \eta_t^2\E[\norm{q_t}^2]\\
        &= (1-2\eta_t\mu_\gamma)\E[\norm{x_t-x_\gamma^*}^2] + \eta_t^2\E[\norm{q_t}^2].
    \end{align*}
    
    Applying Lemma \ref{lem:var_bound}, \eqref{eq:smooth} and \eqref{eq:eta_t_bound}, we have
    \begin{align*}
     &\E[\norm{x_{t+1}-x_\gamma^*}^2]\\
     &\leq (1-2\eta_t\mu_\gamma)\E\norm{x_t-x_\gamma^*}^2 + \eta_t^2\E[\norm{q_t}^2]\\
     &\leq (1-2\eta_t\mu_t)\E[\norm{x_t-x_\gamma^*}^2] +\eta_t^2\E[\norm{\nabla \tilde{f}_\gamma(x_t)}^2] + \eta_t^2\frac{KG^2}{MB}\\
     &\leq (1-2\eta_t\mu_\gamma + \eta_t^2L_\gamma^2)\E\norm{x_t-x_\gamma^*}^2 + \eta_t^2\frac{KG^2}{MB}\\
     &\leq (1-\eta_t\mu_\gamma)\E\norm{x_t-x_\gamma^*}^2 + \eta_t^2\frac{KG^2}{MB}.
    \end{align*}
\end{proof}

\subsubsection{Theorem \ref{thm:fix_client_decay_server}}

\begin{proof}
We will proceed using similar techniques to those in Theorem 4.7 of \cite{bottou2018optimization}. Note that by construction of $b_\gamma, a_\gamma$, we have that for all $t \geq 1$,
\begin{align*}
    \eta_t \leq \dfrac{a_\gamma}{b_\gamma} = \dfrac{\mu_\gamma}{L_\gamma^2}.
\end{align*}

Therefore, by Lemma \ref{lem:descent_step},
\begin{equation}\label{eq:descent}
    \E[\norm{x_{t+1}-x_\gamma^*}^2] \leq (1-\eta_t\mu_\gamma)\E[\norm{x_t-x_\gamma^*}^2] + \eta_t^2\dfrac{KG^2}{MB}.
\end{equation}
We then proceed by induction. For $t = 1$, we have
\begin{align*}
    \norm{x_1-x_\gamma^*}^2 = \dfrac{(b_\gamma+1)\norm{x_1-x_\gamma^*}^2}{b_\gamma+1} \leq \dfrac{v_\gamma}{b_\gamma + 1}.
\end{align*}

For $t > 1$, let $\hat{t} := b_\gamma + t$. Therefore, $\eta_t = a_\gamma/\hat{t}$. Using \eqref{eq:descent} and the inductive hypothesis,
\begin{align*}
    \E[\norm{x_{t+1}-x_\gamma^*}^2] &\leq \left(1-\dfrac{a_\gamma\mu_\gamma}{\hat{t}}\right)\dfrac{\nu_\gamma}{\hat{t}} + \dfrac{a_\gamma^2KG^2}{\hat{t}^2MB}\\
    &= \left(\dfrac{\hat{t}-2}{\hat{t}^2}\right)\nu_\gamma + \dfrac{4KG^2}{\mu_\gamma^2\hat{t}^2MB}\\
    &= \left(\dfrac{\hat{t}-1}{\hat{t}^2}\right)\nu_\gamma - \dfrac{\nu_\gamma}{\hat{t}^2} + \dfrac{4KG^2}{\mu_\gamma^2\hat{t}^2MB}\\
    &= \left(\dfrac{\hat{t}-1}{\hat{t}^2}\right)\nu_\gamma + \dfrac{1}{\hat{t}^2}\underbrace{\left(-\nu_\gamma + \dfrac{4KG^2}{\mu_\gamma^2MB}\right)}_{\Xi}
\end{align*}

Note that $\Xi \leq 0$ by assumption on $\nu_\gamma$, and that simple analysis shows
\[
\hat{t}^2 \geq (\hat{t}-1)(\hat{t}+1).
\]
Putting this together, we have
\[
\E[\norm{x_{t+1}-x_\gamma^*}^2] \leq \dfrac{\nu_\gamma}{\hat{t}+1} = \dfrac{\nu_\gamma}{b_\gamma + t + 1}.\]
\end{proof}

\subsubsection{Corollary \ref{cor:fix_client_decay_server}}

\begin{proof}
We have
\begin{align*}
    \E[\norm{x_t-x^*}^2] \leq 2\E[\norm{x_t-x_\gamma^*}^2] + 2\norm{x_\gamma^*-x^*}^2.
\end{align*}
Using the bounds in Theorems \ref{thm:fix_client_decay_server} and Corollary \ref{cor:optima_distance_fedavg}, we have

\begin{align*}
    \E[\norm{x_t-x^*}^2] \leq \dfrac{2\nu_\gamma}{b_\gamma + t} + 2\sigma_c^2\left(1 + \dfrac{\sigma_a}{\mu}\right)^2\dfrac{L^2}{\mu^2}\dfrac{\epsilon^2}{(1-\epsilon)^2}
\end{align*}
where
\[
\nu_\gamma := \max\left\{\dfrac{4KG^2}{\mu_\gamma^2MB}, (b_\gamma + 1)\norm{x_1-x_\gamma^*}^2 \right\}.
\]
By Lemma \ref{lem:gamma_cond} and assumption on $\gamma, \epsilon$, we have
\[
\dfrac{1}{\mu_\gamma} \leq \dfrac{1}{(1-\epsilon)K\mu}.
\]
Using this to derive an upper bound on the first term in $\nu_\gamma$, we conclude the proof.
\end{proof}

\subsubsection{Theorem \ref{thm:convergence_overall}}

For convenience of notation (and in a slight abuse of previous notation), we define
\begin{align*}
    \mu_t &:= \phi_{K, \mu}(\gamma_t)\\
    L_t &:= \phi_{K, L}(\gamma_t)\\
    \tilde{f}_t(x) &:= \tilde{f}(x, \gamma_t, \Theta_{1:K})\\
    x_t^* &:= \arg\min_{x} \tilde{f}_t(x)\\
    \kappa &:= L/\mu.
\end{align*}

\begin{proof}
We have
\begin{align*}
    \E[\norm{x_t-x^*}^2] &= \E[\norm{x_t-x_t^* + x_t^*-x^*}^2]\\
    &\leq 2\underbrace{\E[\norm{x_t-x_t^*}^2]}_{\alpha_t} + 2\underbrace{\norm{x_t^*-x^*}^2}_{\beta_t}.
\end{align*}

Next, note that for any $t \geq 1$,
\[
\eta_t = \dfrac{a_t}{b+t} = \dfrac{2/\mu_t}{2\kappa^2 + t} < \dfrac{1}{\mu_t\kappa^2}.
\]
Note that by Lemma \ref{lem:cond_number} we have
\[
\dfrac{L_t}{\mu_t} \leq \dfrac{L}{\mu} = \kappa.
\]
Therefore,
\[
\eta_t \leq \dfrac{1}{\mu_t\kappa^2} \leq \dfrac{\mu_t^2}{\mu_tL_t^2} = \dfrac{\mu_t}{L_t^2}.
\]
Therefore, we can apply Lemma \ref{lem:descent_step} (with $\mu_\gamma = \mu_t, L_\gamma = L_t$, $x_\gamma^* = x_t^*$) to find
\begin{equation}\label{eq:descent_t}
    \E[\norm{x_{t+1}-x_t^*}^2] \leq (1-\eta_t\mu_t)\E[\norm{x_t-x_t^*}^2] + \eta_t^2\dfrac{KG^2}{MB}.
\end{equation}

For any $\omega_t > 0$, we therefore have
\begin{equation}
\E[\norm{x_{t+1}-x_t^*}^2] \leq (1+\omega_t)\left(\E[\norm{x_{t+1}-x_t^*}^2]\right) + (1+\omega_t^{-1})\norm{x_t^*-x_{t+1}^*}^2.
\end{equation}

Using \eqref{eq:descent_t}, we derive the following recursion on the $\alpha_t$.
\begin{equation}\label{eq:alpha_recurrence}
\alpha_{t+1} \leq (1+\omega_t)\left((1-\eta_t\mu_t)\alpha_t + \eta_t^2\dfrac{KG^2}{MB}\right) + (1+\omega_t^{-1})\norm{x_t^*-x_{t+1}^*}^2.
\end{equation}

Let $C = 4\sigma_c^2\kappa^2$ and let $\hat{t} = b+t$. . We will use \eqref{eq:alpha_recurrence} inductively to show that
\begin{equation}\label{eq:alpha_t_bound}
\alpha_t \leq \dfrac{\nu + C}{\hat{t}}.
\end{equation}
For $t = 1$, we have
\begin{equation}\label{eq:alpha_1}
\alpha_1 = \norm{x_1-x_1^*}^2 = \dfrac{(b+1)\norm{x_1-x_1^*}^2}{b+1} \leq \dfrac{\nu_1}{b + 1}.
\end{equation}

Similar analysis can be done in the case that $t = 2$. When $t \geq 3$, using the inductive hypothesis, we have
\begin{align*}
    (1-\eta_t\mu_t)\alpha_t + \eta_t^2\dfrac{KG^2}{MB} &\leq (1-\eta_t\mu_t)\dfrac{\nu+C}{\hat{t}} + \dfrac{a_t^2KG^2}{\hat{t}^2MB}\\
    &\leq \left(\dfrac{\hat{t}-3}{\hat{t}^2}\right)(\nu+C) +\dfrac{9KG^2}{\mu_t^2\hat{t}^2MB}\\
    &= \left(\dfrac{\hat{t}-2}{\hat{t}^2}\right)(\nu+C) - \dfrac{C}{\hat{t}^2} +\dfrac{1}{\hat{t}^2}\underbrace{\left(-\nu + \dfrac{9KG^2}{\mu_t^2MB}\right)}_{\xi_t}.
\end{align*}

By definition of $\nu$, we have
\begin{align*}
    \xi_t \leq \dfrac{9KG^2}{\mu_t^2MB} - \dfrac{18G^2}{\mu^2 KMB}.
\end{align*}
However, since $\gamma_t \leq \ln(2)/K\mu$, by Lemma \ref{lem:gamma_cond}, we know
\[
\mu_t \geq \dfrac{K\mu}{2}.
\]
Therefore,
\[
\xi_t \leq \dfrac{18G^2}{\mu^2KMB} - \dfrac{18G^2}{\mu^2 KMB} = 0.
\]
Hence,
\begin{equation}\label{eq:t_3_bound}
    (1-\eta_t\mu_t)\alpha_t + \eta_t^2\dfrac{KG^2}{MB} \leq \left(\dfrac{\hat{t}-2}{\hat{t}^2}\right)(\nu+C) - \dfrac{C}{\hat{t}^2}.
\end{equation}

Let
\[
\omega_t = \dfrac{\hat{t}+2}{(\hat{t}-2)(\hat{t}+1)}.
\]
Note that this implies that
\[
1+\omega_t = \dfrac{\hat{t}^2}{(\hat{t}-2)(\hat{t}+1)}
\]
and
\[
1+\omega_t^{-1} = \dfrac{\hat{t}^2}{\hat{t}+2}.
\]
Therefore, multiplying \eqref{eq:t_3_bound} by $1+\omega_t$, we have
\begin{equation}\label{eq:alpha_lhs}
    (1+\omega_t)\left( (1-\eta_t\mu_t)\alpha_t + \eta_t^2\dfrac{KG^2}{MB}\right) \leq \dfrac{\nu+C}{\hat{t} +1 } - \dfrac{C}{(\hat{t}-2)(\hat{t}+1)}.
\end{equation}

This bounds the first part of \eqref{eq:alpha_recurrence}. For the second part, we will use Corollary \ref{cor:distance_optima_gamma}. In particular, since
\[
\gamma_t \leq \dfrac{1}{L\hat{t}}
\]
we find that
\begin{align*}
\norm{x_t^*-x_{t+1}^*}^2 &\leq 4\sigma_c^2\kappa^4\left(\dfrac{1}{\hat{t}} - \dfrac{1}{\hat{t}+1}\right)^2\\
&= \dfrac{C\kappa^2}{\hat{t}^2(\hat{t}+1)^2}\\
&\leq \dfrac{C}{\hat{t}^2(\hat{t}+1)}.
\end{align*}
Here we used the fact that $\hat{t}+1 \geq b \geq \kappa^2$. Multiplying by $1+\omega_t^{-1}$,
\begin{equation}\label{eq:alpha_rhs}
    (1+\omega_t^{-1})\norm{x_t^*-x_{t+1}^*}^2 \leq \dfrac{C}{(\hat{t}+1)(\hat{t}+2)}.
\end{equation}

Combining \eqref{eq:alpha_recurrence}, \eqref{eq:alpha_lhs} and \eqref{eq:alpha_rhs}, we have
\begin{align*}
    \alpha_{t+1} \leq \dfrac{\nu+C}{\hat{t}+1} - \dfrac{C}{(\hat{t}-2)(\hat{t}+1)} + \dfrac{C}{(\hat{t}+2)(\hat{t}+1)} \leq \dfrac{\nu+C}{\hat{t}+1}.
\end{align*}
This proves \eqref{eq:alpha_t_bound}. To get a bound on the distance to the minima $x^*$, we then have
\begin{align*}
    \E[\norm{x_t - x^*}^2] &\leq 2\E[\norm{x_t-x_t^*}^2] + 2\norm{x_t^*-x^*}^2\\
    &\leq \dfrac{2(\nu+C)}{\hat{t}} + 2\norm{x_t^*-x^*}^2.
\end{align*}
We can again use Corollary \ref{cor:distance_optima_gamma}, letting $\gamma_1 = 0, \gamma_2 = \gamma_t$ in the statement of that result. We then get
\begin{align*}
    \norm{x_t^* - x^*} \leq \dfrac{ C\kappa^2}{\hat{t}^2}\leq \dfrac{C}{\hat{t}}.
\end{align*}
This again uses the fact that $\kappa^2 \leq b \leq \hat{t}$. Combining, this implies
\begin{align*}
    \E[\norm{x_t - x^*}^2] &\leq \dfrac{2\nu}{\hat{t}} + \dfrac{4C}{\hat{t}}.
\end{align*}
Substituting in $C$, this proves the desired result.
\end{proof}

\section{Datasets and Models}\label{sec:models}

Below, we provide detailed description of the datasets and models used in the paper. We use federated versions of vision datasets FEMNIST~\citep{caldas2018leaf} and CIFAR-100~\citep{krizhevsky2009learning}, and language modeling datasets Shakespeare~\citep{mcmahan17fedavg} and StackOverflow~\citep{stackoverflow}. We give descriptions of the datasets, models, and tasks below.

\paragraph{CIFAR-100} The CIFAR-100 dataset is a popular computer vision dataset consisting of $32 \times 32 \times 3$ images with 100 possible labels. While this dataset is not a federated dataset, a federated version was created by \citet{reddi2020adaptive}, using hierarchical latent Dirichlet allocation to enforce moderate amounts of heterogeneity among clients. The resulting dataset has 500 clients, each with 100 unique examples. We train a ResNet-18 on this dataset, where we replace all batch normalization layers with group normalization layers~\citep{wu2018group}. The use of group norm over batch norm in federated learning was first advocated by \citet{hsieh2019non}.

We perform small amounts of data augmentation and preprocessing, as is standard with CIFAR-100. We first perform a random crop to shape $(24, 24, 3)$, followed by a random horizontal flip. We then normalize the pixel values according to their mean and standard deviation. Thus, given an image $x$, we compute $(x - \mu)/\sigma$ where $\mu$ is the average of the pixel values in $x$, and $\sigma$ is the standard deviation.

\paragraph{FEMNIST} FEMNIST consists of $28\times 28$ gray-scale images of both numbers and upper- and lower-case English characters, with 62 possible labels in total. The digits are partitioned according to their author, resulting in a naturally heterogeneous federated dataset. We do not use any preprocessing on the images. We train a moderately-sized CNN, with identical architecture to the CNN used by \citet{mcmahan17fedavg}. The CNN contains two convolutional layer, each with $5\times 5$ kernels. The convolutional layers have 32 and 64 filters, respectively, and are each followed by a $2\times 2$ max pooling layer. Finally, the model has a dense layer with 512 units and ReLU activation, followed by a softmax activation.

\paragraph{Shakespeare} The Shakespeare dataset is derived from the benchmark designed by \citet{caldas2018leaf}. The dataset corpus is the collected works of William Shakespeare, and the clients correspond to roles in Shakespeare's plays with at least two lines of dialogue. To eliminate confusion, \emph{character} here will refer to alphanumeric and other such symbols, while we will use \emph{client} to denote the various roles in plays. We split each client's lines into sequences of 80 characters, padding if necessary. We use a vocabulary size of 90: 86 characters contained in Shakespeare's work, beginning and end of line tokens, padding tokens, and out-of-vocabulary tokens. We perform next-character prediction on the clients' dialogue using an RNN. The RNN takes as input a sequence of 80 characters, embeds it into a learned 8-dimensional space, and passes the embedding through 2 LSTM layers, each with 256 units. Finally, we use a softmax output layer with 80 units, where we try to predict a sequence of 80 characters formed by shifting the input sequence over by one. Therefore, our output dimension is $80\times 90$. We compute loss using cross-entropy loss.

\paragraph{Stack Overflow} Stack Overflow is a text datasets consisting of questions and answers posted to the Stack Overflow website. Each user is a client, and their datasets consist of questions and answers posted by this user. Each post has associated meta-data, including a list of associated tags (e.g. a post could have the tag \emph{javascript} if it concerns the javascript language). We perform two tasks on this dataset: tag prediction, and next word prediction. In both cases, we restrict to the 10,000 most frequently used words in the total dataset, as well as the 500 most frequently used tags for the tag prediction task.

For Stack Overflow tag prediction, we use a multi-class logistic regression classifier with 500 output units (one for each of the 500 most frequently used tags), and adopt a one-versus-rest classification strategy. Note that the corresponding multi-class logistic loss is convex. The inputs to our model are 10,000-dimensional vectors forming bag-of-words vectors for each post. Each vector is normalized to have sum 1.

For Stack Overflow next word prediction, we restrict each client to the first 128 posts in their history (for computational efficiency reasons, as some clients have tens of thousands of posts). We perform truncation and padding so that each post has 21 words (including word tokens for beginning of sentence, end of sentence, padding, and out-of-vocabulary words). The sequence is split into input and output length-20 sequences, corresponding to the first and the last 20 characters (ie. one is the other sequence, shifted by one). The first of these sequences is embedded into a learned 96-dimensional space, and then fed into an LSTM with 670 units. Finally, the output is fed into a densely connected softmax layer with 10,004 units (corresponding to the 10,000 in-vocabulary words, and the extra tokens mentioned above). We attempt to predict the shifted-by-one sequence, and compute the loss via cross-entropy.

\section{Tuned Server Learning Rates}\label{sec:tuned_lrs}

In this section, we detail the best server learning rate $\eta$ found for each corresponding client learning rate $\gamma$ and task.

\begin{table}[ht]
\setlength{\tabcolsep}{3.5pt}
\caption{Best server learning rate $\eta$ for each client learning rate $\gamma$ in the CIFAR-100 task.}
\label{table:best_eta_cifar100}
\begin{center}
\begin{sc}
\begin{tabular}[t]{cc}
\toprule
$\gamma$ & $\eta$\\
\midrule
$1$ & $10^{-1}$ \\
$10^{-1}$ & $10^{-1}$\\
$10^{-2}$ & $10^{-3/2}$\\
$10^{-3}$ & $10^{-3/2}$\\
$0$ & $10^{-3/2}$\\
\bottomrule
\end{tabular}
\end{sc}
\end{center}
\end{table}

\begin{table}[ht]
\setlength{\tabcolsep}{3.5pt}
\caption{Best server learning rate $\eta$ for each client learning rate $\gamma$ in the FEMNIST task.}
\label{table:best_eta_femnist}
\begin{center}
\begin{sc}
\begin{tabular}[t]{cc}
\toprule
$\gamma$ & $\eta$\\
\midrule
$10^{-1}$ & $10^{-2}$\\
$10^{-2}$ & $10^{-3/2}$\\
$10^{-3}$ & $10^{-3/2}$\\
$0$ & $10^{-3/2}$\\
\bottomrule
\end{tabular}
\end{sc}
\end{center}
\end{table}

\begin{table}[ht]
\setlength{\tabcolsep}{3.5pt}
\caption{Best server learning rate $\eta$ for each client learning rate $\gamma$ in the Shakespeare task.}
\label{table:best_eta_shakespeare}
\begin{center}
\begin{sc}
\begin{tabular}[t]{cc}
\toprule
$\gamma$ & $\eta$\\
\midrule
$1$ & $10^{-1/2}$\\
$10^{-1}$ & $1$\\
$10^{-2}$ & $10^{-1}$\\
$10^{-3}$ & $10^{-1}$\\
$0$ & $10^{-1}$\\
\bottomrule
\end{tabular}
\end{sc}
\end{center}
\end{table}

\begin{table}[ht]
\setlength{\tabcolsep}{3.5pt}
\caption{Best server learning rate $\eta$ for each client learning rate $\gamma$ in the Stack Overflow next word prediction task.}
\label{table:best_eta_so_nwp}
\begin{center}
\begin{sc}
\begin{tabular}[t]{cc}
\toprule
$\gamma$ & $\eta$\\
\midrule
$10^{-1}$ & $10^{-3/2}$\\
$10^{-2}$ & $10^{-3/2}$\\
$10^{-3}$ & $10^{-2}$\\
$0$ & $10^{-2}$\\
\bottomrule
\end{tabular}
\end{sc}
\end{center}
\end{table}

\begin{table}[ht]
\setlength{\tabcolsep}{3.5pt}
\caption{Best server learning rate $\eta$ for each client learning rate $\gamma$ in the Stack Overflow tag prediction task.}
\label{table:best_eta_so_tp}
\begin{center}
\begin{sc}
\begin{tabular}[t]{cc}
\toprule
$\gamma$ & $\eta$\\
\midrule
$100 $ & $10^{-1/2}$\\
$10$ & $10^{-1/2}$\\
$1$ & $10^{-1/2}$\\
$10^{-1}$ & $10^{-1/2}$\\
$10^{-2}$ & $10^{-1/2}$\\
$0$ & $10^{-1/2}$\\
\bottomrule
\end{tabular}
\end{sc}
\end{center}
\end{table}

\clearpage

\vskip 0.2in
\bibliography{main}

\begin{thebibliography}{65}
\providecommand{\natexlab}[1]{#1}
\providecommand{\url}[1]{\texttt{#1}}
\expandafter\ifx\csname urlstyle\endcsname\relax
  \providecommand{\doi}[1]{doi: #1}\else
  \providecommand{\doi}{doi: \begingroup \urlstyle{rm}\Url}\fi

\bibitem[Antoniou et~al.(2019)Antoniou, Edwards, and Storkey]{antoniou2018how}
Antreas Antoniou, Harrison Edwards, and Amos Storkey.
\newblock How to train your {MAML}.
\newblock In \emph{International Conference on Learning Representations}, 2019.
\newblock URL \url{https://openreview.net/forum?id=HJGven05Y7}.

\bibitem[Augenstein et~al.(2020)Augenstein, McMahan, Ramage, Ramaswamy,
  Kairouz, Chen, Mathews, and y~Arcas]{augenstein2020generative}
Sean Augenstein, H.~Brendan McMahan, Daniel Ramage, Swaroop Ramaswamy, Peter
  Kairouz, Mingqing Chen, Rajiv Mathews, and Blaise~Aguera y~Arcas.
\newblock Generative models for effective {ML} on private, decentralized
  datasets.
\newblock In \emph{International Conference on Learning Representations}, 2020.
\newblock URL \url{https://openreview.net/forum?id=SJgaRA4FPH}.

\bibitem[Authors(2019)]{stackoverflow}
The TensorFlow~Federated Authors.
\newblock Tensor{F}low {F}ederated {Stack Overflow} dataset, 2019.
\newblock URL
  \url{https://www.tensorflow.org/federated/api_docs/python/tff/simulation/datasets/stackoverflow/load_data}.

\bibitem[Bagdasaryan et~al.(2018)Bagdasaryan, Veit, Hua, Estrin, and
  Shmatikov]{bagdasaryan2018backdoor}
Eugene Bagdasaryan, Andreas Veit, Yiqing Hua, Deborah Estrin, and Vitaly
  Shmatikov.
\newblock How to backdoor federated learning.
\newblock \emph{arXiv preprint arXiv:1807.00459}, 2018.

\bibitem[Basu et~al.(2019)Basu, Data, Karakus, and Diggavi]{basu2019qsparse}
Debraj Basu, Deepesh Data, Can Karakus, and Suhas Diggavi.
\newblock Qsparse-local-{SGD}: Distributed {SGD} with quantization,
  sparsification and local computations.
\newblock In \emph{Advances in Neural Information Processing Systems}, pages
  14668--14679, 2019.

\bibitem[Bonawitz et~al.(2019)Bonawitz, Eichner, Grieskamp, Huba, Ingerman,
  Ivanov, Kiddon, Kone\v{c}n\'{y}, Mazzocchi, McMahan, Van~Overveldt, Petrou,
  Ramage, and Roselander]{bonawitz2019towards}
Keith Bonawitz, Hubert Eichner, Wolfgang Grieskamp, Dzmitry Huba, Alex
  Ingerman, Vladimir Ivanov, Chlo\'{e} Kiddon, Jakub Kone\v{c}n\'{y}, Stefano
  Mazzocchi, Brendan McMahan, Timon Van~Overveldt, David Petrou, Daniel Ramage,
  and Jason Roselander.
\newblock Towards federated learning at scale: System design.
\newblock In \emph{Proceedings of Machine Learning and Systems 2019}, pages
  374--388. 2019.

\bibitem[Bottou et~al.(2018)Bottou, Curtis, and
  Nocedal]{bottou2018optimization}
L{\'e}on Bottou, Frank~E Curtis, and Jorge Nocedal.
\newblock Optimization methods for large-scale machine learning.
\newblock \emph{Siam Review}, 60\penalty0 (2):\penalty0 223--311, 2018.

\bibitem[Brisimi et~al.(2018)Brisimi, Chen, Mela, Olshevsky, Paschalidis, and
  Shi]{brisimi2018federated}
Theodora~S Brisimi, Ruidi Chen, Theofanie Mela, Alex Olshevsky, Ioannis~Ch
  Paschalidis, and Wei Shi.
\newblock Federated learning of predictive models from federated electronic
  health records.
\newblock \emph{International journal of medical informatics}, 112:\penalty0
  59--67, 2018.

\bibitem[Bubeck(2017)]{bubeck2014convex}
S{\'e}bastien Bubeck.
\newblock Convex optimization: Algorithms and complexity.
\newblock \emph{Foundations and Trends in Machine Learning}, 2017.

\bibitem[Caldas et~al.(2018)Caldas, Wu, Li, Kone{\v{c}}n{\'y}, McMahan, Smith,
  and Talwalkar]{caldas2018leaf}
Sebastian Caldas, Peter Wu, Tian Li, Jakub Kone{\v{c}}n{\'y}, H~Brendan
  McMahan, Virginia Smith, and Ameet Talwalkar.
\newblock {LEAF}: A benchmark for federated settings.
\newblock \emph{arXiv preprint arXiv:1812.01097}, 2018.

\bibitem[Chen et~al.(2019)Chen, Suresh, Mathews, Wong, Allauzen, Beaufays, and
  Riley]{chen2019federated}
Mingqing Chen, Ananda~Theertha Suresh, Rajiv Mathews, Adeline Wong, Cyril
  Allauzen, Fran{\c{c}}oise Beaufays, and Michael Riley.
\newblock Federated learning of n-gram language models.
\newblock \emph{arXiv preprint arXiv:1910.03432}, 2019.

\bibitem[Fallah et~al.(2019)Fallah, Mokhtari, and
  Ozdaglar]{fallah2019convergence}
Alireza Fallah, Aryan Mokhtari, and Asuman Ozdaglar.
\newblock On the convergence theory of gradient-based model-agnostic
  meta-learning algorithms.
\newblock \emph{arXiv preprint arXiv:1908.10400}, 2019.

\bibitem[Fallah et~al.(2020)Fallah, Mokhtari, and
  Ozdaglar]{fallah2020personalized}
Alireza Fallah, Aryan Mokhtari, and Asuman Ozdaglar.
\newblock Personalized federated learning: A meta-learning approach.
\newblock \emph{arXiv preprint arXiv:2002.07948}, 2020.

\bibitem[Finn et~al.(2017)Finn, Abbeel, and Levine]{finn2017model}
Chelsea Finn, Pieter Abbeel, and Sergey Levine.
\newblock Model-agnostic meta-learning for fast adaptation of deep networks.
\newblock In \emph{Proceedings of the 34th International Conference on Machine
  Learning-Volume 70}, pages 1126--1135. JMLR, 2017.

\bibitem[Ghosh et~al.(2019)Ghosh, Hong, Yin, and Ramchandran]{ghosh2019robust}
Avishek Ghosh, Justin Hong, Dong Yin, and Kannan Ramchandran.
\newblock Robust federated learning in a heterogeneous environment.
\newblock \emph{arXiv preprint arXiv:1906.06629}, 2019.

\bibitem[Goyal et~al.(2017)Goyal, Doll{\'a}r, Girshick, Noordhuis, Wesolowski,
  Kyrola, Tulloch, Jia, and He]{goyal2017accurate}
Priya Goyal, Piotr Doll{\'a}r, Ross Girshick, Pieter Noordhuis, Lukasz
  Wesolowski, Aapo Kyrola, Andrew Tulloch, Yangqing Jia, and Kaiming He.
\newblock Accurate, large minibatch {SGD}: Training {I}mage{N}et in 1 hour.
\newblock \emph{arXiv preprint arXiv:1706.02677}, 2017.

\bibitem[Grant et~al.(2018)Grant, Finn, Levine, Darrell, and
  Griffiths]{grant2018recasting}
Erin Grant, Chelsea Finn, Sergey Levine, Trevor Darrell, and Thomas Griffiths.
\newblock Recasting gradient-based meta-learning as hierarchical {Bayes}.
\newblock In \emph{International Conference on Learning Representations}, 2018.
\newblock URL \url{https://openreview.net/forum?id=BJ_UL-k0b}.

\bibitem[Groves and Rothenberg(1969)]{groves1969note}
Theodore Groves and Thomas Rothenberg.
\newblock A note on the expected value of an inverse matrix.
\newblock \emph{Biometrika}, 56\penalty0 (3):\penalty0 690--691, 1969.

\bibitem[Hard et~al.(2018)Hard, Rao, Mathews, Ramaswamy, Beaufays, Augenstein,
  Eichner, Kiddon, and Ramage]{hard2018federated}
Andrew Hard, Kanishka Rao, Rajiv Mathews, Swaroop Ramaswamy, Fran{\c{c}}oise
  Beaufays, Sean Augenstein, Hubert Eichner, Chlo{\'e} Kiddon, and Daniel
  Ramage.
\newblock Federated learning for mobile keyboard prediction.
\newblock \emph{arXiv preprint arXiv:1811.03604}, 2018.

\bibitem[Hard et~al.(2020)Hard, Partridge, Nguyen, Subrahmanya, Shah, Zhu,
  Moreno, and Mathews]{hard2020training}
Andrew Hard, Kurt Partridge, Cameron Nguyen, Niranjan Subrahmanya, Aishanee
  Shah, Pai Zhu, Ignacio~Lopez Moreno, and Rajiv Mathews.
\newblock Training keyword spotting models on non-{IID} data with federated
  learning.
\newblock \emph{arXiv preprint arXiv:2005.10406}, 2020.

\bibitem[Hsieh et~al.(2019)Hsieh, Phanishayee, Mutlu, and
  Gibbons]{hsieh2019non}
Kevin Hsieh, Amar Phanishayee, Onur Mutlu, and Phillip~B Gibbons.
\newblock The non-{IID} data quagmire of decentralized machine learning.
\newblock \emph{arXiv preprint arXiv:1910.00189}, 2019.

\bibitem[Hsu et~al.(2019)Hsu, Qi, and Brown]{hsu2019measuring}
Tzu-Ming~Harry Hsu, Hang Qi, and Matthew Brown.
\newblock Measuring the effects of non-identical data distribution for
  federated visual classification.
\newblock \emph{arXiv preprint arXiv:1909.06335}, 2019.

\bibitem[Ingerman and Ostrowski(2019)]{ingerman2019tff}
Alex Ingerman and Krzys Ostrowski.
\newblock Introducing tensorflow federated, 2019.
\newblock URL
  \url{https://medium.com/tensorflow/introducing-tensorflow-federated-a4147aa20041}.

\bibitem[Jiang et~al.(2019)Jiang, Kone{\v{c}}n{\`y}, Rush, and
  Kannan]{jiang2019improving}
Yihan Jiang, Jakub Kone{\v{c}}n{\`y}, Keith Rush, and Sreeram Kannan.
\newblock Improving federated learning personalization via model agnostic meta
  learning.
\newblock \emph{arXiv preprint arXiv:1909.12488}, 2019.

\bibitem[Kairouz et~al.(2019)Kairouz, McMahan, Avent, Bellet, Bennis, Bhagoji,
  Bonawitz, Charles, Cormode, Cummings, et~al.]{kairouz2019advances}
Peter Kairouz, H~Brendan McMahan, Brendan Avent, Aur{\'e}lien Bellet, Mehdi
  Bennis, Arjun~Nitin Bhagoji, Keith Bonawitz, Zachary Charles, Graham Cormode,
  Rachel Cummings, et~al.
\newblock Advances and open problems in federated learning.
\newblock \emph{arXiv preprint arXiv:1912.04977}, 2019.

\bibitem[Karimireddy et~al.(2019)Karimireddy, Kale, Mohri, Reddi, Stich, and
  Suresh]{karimireddy2019scaffold}
Sai~Praneeth Karimireddy, Satyen Kale, Mehryar Mohri, Sashank~J Reddi,
  Sebastian~U Stich, and Ananda~Theertha Suresh.
\newblock {SCAFFOLD}: Stochastic controlled averaging for on-device federated
  learning.
\newblock \emph{arXiv preprint arXiv:1910.06378}, 2019.

\bibitem[Khaled et~al.(2020)Khaled, Mishchenko, and
  Richt{\'a}rik]{khaled2020tighter}
A~Khaled, K~Mishchenko, and P~Richt{\'a}rik.
\newblock Tighter theory for local {SGD} on identical and heterogeneous data.
\newblock In \emph{The 23rd International Conference on Artificial Intelligence
  and Statistics (AISTATS 2020)}, 2020.

\bibitem[Khodak et~al.(2019)Khodak, Balcan, and Talwalkar]{khodak2019adaptive}
Mikhail Khodak, Maria-Florina~F Balcan, and Ameet~S Talwalkar.
\newblock Adaptive gradient-based meta-learning methods.
\newblock In \emph{Advances in Neural Information Processing Systems}, pages
  5915--5926, 2019.

\bibitem[Kingma and Ba(2014)]{kingma2014adam}
Diederik~P Kingma and Jimmy Ba.
\newblock Adam: A method for stochastic optimization.
\newblock \emph{arXiv preprint arXiv:1412.6980}, 2014.

\bibitem[Kone{\v{c}}n{\'y} et~al.(2016)Kone{\v{c}}n{\'y}, McMahan, Yu,
  Richt{\'a}rik, Suresh, and Bacon]{konevcny2016federated}
Jakub Kone{\v{c}}n{\'y}, H~Brendan McMahan, Felix~X Yu, Peter Richt{\'a}rik,
  Ananda~Theertha Suresh, and Dave Bacon.
\newblock Federated learning: Strategies for improving communication
  efficiency.
\newblock \emph{arXiv preprint arXiv:1610.05492}, 2016.

\bibitem[Krizhevsky(2014)]{krizhevsky2014one}
Alex Krizhevsky.
\newblock One weird trick for parallelizing convolutional neural networks.
\newblock \emph{arXiv preprint arXiv:1404.5997}, 2014.

\bibitem[Krizhevsky and Hinton(2009)]{krizhevsky2009learning}
Alex Krizhevsky and Geoffrey Hinton.
\newblock Learning multiple layers of features from tiny images.
\newblock Technical report, Citeseer, 2009.

\bibitem[Li et~al.(2019)Li, Sahu, Talwalkar, and Smith]{li2019federated}
Tian Li, Anit~Kumar Sahu, Ameet Talwalkar, and Virginia Smith.
\newblock Federated learning: Challenges, methods, and future directions.
\newblock \emph{arXiv preprint arXiv:1908.07873}, 2019.

\bibitem[Li et~al.(2020{\natexlab{a}})Li, Sahu, Zaheer, Sanjabi, Talwalkar, and
  Smith]{li2018federated}
Tian Li, Anit~Kumar Sahu, Manzil Zaheer, Maziar Sanjabi, Ameet Talwalkar, and
  Virginia Smith.
\newblock Federated optimization in heterogeneous networks.
\newblock In \emph{Proceedings of Machine Learning and Systems 2020}, pages
  429--450. 2020{\natexlab{a}}.

\bibitem[Li et~al.(2020{\natexlab{b}})Li, Sanjabi, Beirami, and
  Smith]{li2019fair}
Tian Li, Maziar Sanjabi, Ahmad Beirami, and Virginia Smith.
\newblock Fair resource allocation in federated learning.
\newblock In \emph{International Conference on Learning Representations},
  2020{\natexlab{b}}.
\newblock URL \url{https://openreview.net/forum?id=ByexElSYDr}.

\bibitem[Li et~al.(2020{\natexlab{c}})Li, Huang, Yang, Wang, and
  Zhang]{li2019convergence}
Xiang Li, Kaixuan Huang, Wenhao Yang, Shusen Wang, and Zhihua Zhang.
\newblock On the convergence of fedavg on non-{IID} data.
\newblock In \emph{International Conference on Learning Representations},
  2020{\natexlab{c}}.
\newblock URL \url{https://openreview.net/forum?id=HJxNAnVtDS}.

\bibitem[Malinovsky et~al.(2020)Malinovsky, Kovalev, Gasanov, Condat, and
  Richtarik]{malinovsky2020local}
Grigory Malinovsky, Dmitry Kovalev, Elnur Gasanov, Laurent Condat, and Peter
  Richtarik.
\newblock From local {SGD} to local fixed point methods for federated learning.
\newblock \emph{arXiv preprint arXiv:2004.01442}, 2020.

\bibitem[McMahan et~al.(2017)McMahan, Moore, Ramage, Hampson, and
  y~Arcas]{mcmahan17fedavg}
Brendan McMahan, Eider Moore, Daniel Ramage, Seth Hampson, and
  Blaise~Ag{\"{u}}era y~Arcas.
\newblock Communication-efficient learning of deep networks from decentralized
  data.
\newblock In \emph{Proceedings of the 20th International Conference on
  Artificial Intelligence and Statistics, {AISTATS} 2017}, pages 1273--1282,
  2017.

\bibitem[McMahan et~al.(2018)McMahan, Ramage, Talwar, and
  Zhang]{mcmahan2017learning}
H.~Brendan McMahan, Daniel Ramage, Kunal Talwar, and Li~Zhang.
\newblock Learning differentially private recurrent language models.
\newblock In \emph{International Conference on Learning Representations}, 2018.
\newblock URL \url{https://openreview.net/forum?id=BJ0hF1Z0b}.

\bibitem[Mohri et~al.(2019)Mohri, Sivek, and Suresh]{mohri2019agnostic}
Mehryar Mohri, Gary Sivek, and Ananda~Theertha Suresh.
\newblock Agnostic federated learning.
\newblock \emph{arXiv preprint arXiv:1902.00146}, 2019.

\bibitem[Nichol et~al.(2018)Nichol, Achiam, and Schulman]{nichol2018first}
Alex Nichol, Joshua Achiam, and John Schulman.
\newblock On first-order meta-learning algorithms.
\newblock \emph{arXiv preprint arXiv:1803.02999}, 2018.

\bibitem[Pathak and Wainwright(2020)]{pathak2020fedsplit}
Reese Pathak and Martin~J Wainwright.
\newblock Fed{S}plit: An algorithmic framework for fast federated optimization.
\newblock \emph{arXiv preprint arXiv:2005.05238}, 2020.

\bibitem[Raghu et~al.(2020)Raghu, Raghu, Bengio, and Vinyals]{raghu2019rapid}
Aniruddh Raghu, Maithra Raghu, Samy Bengio, and Oriol Vinyals.
\newblock Rapid learning or feature reuse? towards understanding the
  effectiveness of {MAML}.
\newblock In \emph{International Conference on Learning Representations}, 2020.
\newblock URL \url{https://openreview.net/forum?id=rkgMkCEtPB}.

\bibitem[Rajeswaran et~al.(2019)Rajeswaran, Finn, Kakade, and
  Levine]{rajeswaran2019meta}
Aravind Rajeswaran, Chelsea Finn, Sham~M Kakade, and Sergey Levine.
\newblock Meta-learning with implicit gradients.
\newblock In \emph{Advances in Neural Information Processing Systems}, pages
  113--124, 2019.

\bibitem[Rakhlin et~al.(2012)Rakhlin, Shamir, and Sridharan]{rakhlin2011making}
Alexander Rakhlin, Ohad Shamir, and Karthik Sridharan.
\newblock Making gradient descent optimal for strongly convex stochastic
  optimization.
\newblock In \emph{Proceedings of the 29th International Conference on Machine
  Learning (ICML-12)}, pages 449--456, 2012.

\bibitem[Reddi et~al.(2020)Reddi, Charles, Zaheer, Garrett, Rush,
  Kone{\v{c}}n{\`y}, Kumar, and McMahan]{reddi2020adaptive}
Sashank Reddi, Zachary Charles, Manzil Zaheer, Zachary Garrett, Keith Rush,
  Jakub Kone{\v{c}}n{\`y}, Sanjiv Kumar, and H~Brendan McMahan.
\newblock Adaptive federated optimization.
\newblock \emph{arXiv preprint arXiv:2003.00295}, 2020.

\bibitem[Reisizadeh et~al.(2020)Reisizadeh, Mokhtari, Hassani, Jadbabaie, and
  Pedarsani]{reisizadeh2019fedpaq}
Amirhossein Reisizadeh, Aryan Mokhtari, Hamed Hassani, Ali Jadbabaie, and
  Ramtin Pedarsani.
\newblock Fed{PAQ}: A communication-efficient federated learning method with
  periodic averaging and quantization.
\newblock In \emph{The 23rd International Conference on Artificial Intelligence
  and Statistics (AISTATS 2020)}, 2020.

\bibitem[Rusu et~al.(2019)Rusu, Rao, Sygnowski, Vinyals, Pascanu, Osindero, and
  Hadsell]{rusu2018metalearning}
Andrei~A. Rusu, Dushyant Rao, Jakub Sygnowski, Oriol Vinyals, Razvan Pascanu,
  Simon Osindero, and Raia Hadsell.
\newblock Meta-learning with latent embedding optimization.
\newblock In \emph{International Conference on Learning Representations}, 2019.
\newblock URL \url{https://openreview.net/forum?id=BJgklhAcK7}.

\bibitem[Samarakoon et~al.(2018)Samarakoon, Bennis, Saad, and
  Debbah]{samarakoon2018federated}
Sumudu Samarakoon, Mehdi Bennis, Walid Saad, and Merouane Debbah.
\newblock Federated learning for ultra-reliable low-latency {V2V}
  communications.
\newblock In \emph{2018 IEEE Global Communications Conference (GLOBECOM)},
  pages 1--7. IEEE, 2018.

\bibitem[Sattler et~al.(2019)Sattler, Wiedemann, M{\"u}ller, and
  Samek]{sattler2019robust}
Felix Sattler, Simon Wiedemann, Klaus-Robert M{\"u}ller, and Wojciech Samek.
\newblock Robust and communication-efficient federated learning from non-{IID}
  data.
\newblock \emph{IEEE transactions on neural networks and learning systems},
  2019.

\bibitem[Stich(2019)]{stich2018local}
Sebastian~U. Stich.
\newblock Local {SGD} converges fast and communicates little.
\newblock In \emph{International Conference on Learning Representations}, 2019.
\newblock URL \url{https://openreview.net/forum?id=S1g2JnRcFX}.

\bibitem[Stich and Karimireddy(2019)]{stich2019error}
Sebastian~U Stich and Sai~Praneeth Karimireddy.
\newblock The error-feedback framework: Better rates for {SGD} with delayed
  gradients and compressed communication.
\newblock \emph{arXiv preprint arXiv:1909.05350}, 2019.

\bibitem[Sun et~al.(2019)Sun, Kairouz, Suresh, and McMahan]{sun2019can}
Ziteng Sun, Peter Kairouz, Ananda~Theertha Suresh, and H~Brendan McMahan.
\newblock Can you really backdoor federated learning?
\newblock \emph{arXiv preprint arXiv:1911.07963}, 2019.

\bibitem[Vanschoren(2019)]{vanschoren2019meta}
Joaquin Vanschoren.
\newblock Meta-learning.
\newblock In \emph{Automated Machine Learning}, pages 35--61. Springer, 2019.

\bibitem[Wang and Joshi(2018)]{wang2018cooperative}
Jianyu Wang and Gauri Joshi.
\newblock Cooperative {SGD}: A unified framework for the design and analysis of
  communication-efficient {SGD} algorithms.
\newblock \emph{arXiv preprint arXiv:1808.07576}, 2018.

\bibitem[Wen et~al.(2019)Wen, Li, Roth, and Dogra]{nvidia_clara}
Yuhong Wen, Wenqi Li, Holger Roth, and Prerna Dogra.
\newblock Federated learning powered by {NVIDIA} {Clara}, December 2019.
\newblock URL \url{https://devblogs.nvidia.com/federated-learning-clara/}.
\newblock NVIDIA Developer Blog.

\bibitem[Woodworth et~al.(2020)Woodworth, Patel, Stich, Dai, Bullins, McMahan,
  Shamir, and Srebro]{woodworth2020local}
Blake Woodworth, Kumar~Kshitij Patel, Sebastian~U Stich, Zhen Dai, Brian
  Bullins, H~Brendan McMahan, Ohad Shamir, and Nathan Srebro.
\newblock Is local {SGD} better than minibatch {SGD}?
\newblock \emph{arXiv preprint arXiv:2002.07839}, 2020.

\bibitem[Wu and He(2018)]{wu2018group}
Yuxin Wu and Kaiming He.
\newblock Group normalization.
\newblock In \emph{Proceedings of the European Conference on Computer Vision
  (ECCV)}, pages 3--19, 2018.

\bibitem[Xie et~al.(2019)Xie, Koyejo, Gupta, and Lin]{xie2019local}
Cong Xie, Oluwasanmi Koyejo, Indranil Gupta, and Haibin Lin.
\newblock Local {A}da{A}lter: Communication-efficient stochastic gradient
  descent with adaptive learning rates.
\newblock \emph{arXiv preprint arXiv:1911.09030}, 2019.

\bibitem[Yang et~al.(2019)Yang, Liu, Chen, and Tong]{yang2019federated}
Qiang Yang, Yang Liu, Tianjian Chen, and Yongxin Tong.
\newblock Federated machine learning: Concept and applications.
\newblock \emph{ACM Transactions on Intelligent Systems and Technology (TIST)},
  10\penalty0 (2):\penalty0 1--19, 2019.

\bibitem[Yang et~al.(2018)Yang, Andrew, Eichner, Sun, Li, Kong, Ramage, and
  Beaufays]{yang2018applied}
Timothy Yang, Galen Andrew, Hubert Eichner, Haicheng Sun, Wei Li, Nicholas
  Kong, Daniel Ramage, and Fran{\c{c}}oise Beaufays.
\newblock Applied federated learning: Improving google keyboard query
  suggestions.
\newblock \emph{arXiv preprint arXiv:1812.02903}, 2018.

\bibitem[Yu et~al.(2019)Yu, Yang, and Zhu]{yu2019parallel}
Hao Yu, Sen Yang, and Shenghuo Zhu.
\newblock Parallel restarted {SGD} with faster convergence and less
  communication: Demystifying why model averaging works for deep learning.
\newblock In \emph{Proceedings of the AAAI Conference on Artificial
  Intelligence}, volume~33, pages 5693--5700, 2019.

\bibitem[Zaheer et~al.(2018)Zaheer, Reddi, Sachan, Kale, and
  Kumar]{zaheer2018adaptive}
Manzil Zaheer, Sashank Reddi, Devendra Sachan, Satyen Kale, and Sanjiv Kumar.
\newblock Adaptive methods for nonconvex optimization.
\newblock In \emph{Advances in Neural Information Processing Systems}, pages
  9815--9825, 2018.

\bibitem[Zhang et~al.(2019)Zhang, Lucas, Ba, and Hinton]{zhang2019lookahead}
Michael Zhang, James Lucas, Jimmy Ba, and Geoffrey~E Hinton.
\newblock Lookahead optimizer: k steps forward, 1 step back.
\newblock In \emph{Advances in Neural Information Processing Systems}, pages
  9593--9604, 2019.

\bibitem[Zinkevich et~al.(2010)Zinkevich, Weimer, Li, and
  Smola]{zinkevich2010parallelized}
Martin Zinkevich, Markus Weimer, Lihong Li, and Alex~J Smola.
\newblock Parallelized stochastic gradient descent.
\newblock In \emph{Advances in neural information processing systems}, pages
  2595--2603, 2010.

\end{thebibliography}

\end{document}